\documentclass{article}

\usepackage{microtype}
\usepackage{graphicx}
\usepackage{subcaption}
\usepackage{booktabs}

\usepackage{hyperref}

\usepackage{placeins}

\usepackage{float}
\floatstyle{plain}
\restylefloat{algorithm}

\usepackage[preprint]{style_file}

\usepackage{amsmath}
\usepackage{amssymb}
\usepackage{mathtools}
\usepackage{amsthm}

\usepackage{booktabs,tabularx}
\usepackage{multirow}

\usepackage[capitalize,noabbrev]{cleveref}

\usepackage{algorithm}
\usepackage{bbm}
\usepackage[table]{xcolor}

\usepackage[dvipsnames]{xcolor}

\usepackage{arydshln}

\AtBeginDocument{\setlength{\textfloatsep}{12pt}
}

\definecolor{mplmagenta}{HTML}{FF00FF}
\definecolor{rowblue}{RGB}{230,245,255}
\definecolor{rowyellow}{RGB}{255,249,210}

\newcommand{\E}{\mathbb{E}}
\newcommand{\R}{\mathbb{R}}
\newcommand{\KL}{\mathrm{KL}}
\newcommand{\tr}{\mathrm{tr}}

\newcommand{\onelinebox}[1]{\scalebox{0.885}{\fcolorbox{black}{rowblue}{$\displaystyle #1$}}}

\newcommand{\ema}{\text{ema}}

\newcommand{\apxpage}[1]{\dotfill\hyperref[#1]{\pageref*{#1}}}
\newcommand{\refpage}[1]{\hfill\hyperref[#1]{\pageref*{#1}}}

\theoremstyle{plain}
\newtheorem{theorem}{Theorem}[section]
\newtheorem{proposition}[theorem]{Proposition}
\newtheorem{lemma}[theorem]{Lemma}

\theoremstyle{definition}

\theoremstyle{remark}

\usepackage[textsize=tiny]{todonotes}

\icmltitlerunning{EMA Policy Gradient}

\begin{document}

\twocolumn[
  \icmltitle{EMA Policy Gradient:\\
    Taming Reinforcement Learning for LLMs with EMA Anchor and Top-k KL}

  \icmlsetsymbol{equal}{*}

  \begin{icmlauthorlist}
    \icmlauthor{Lunjun Zhang}{uoft}
    \icmlauthor{Jimmy Ba}{uoft}
  \end{icmlauthorlist}

  \icmlaffiliation{uoft}{Department of Computer Science, University of Toronto}

  \icmlcorrespondingauthor{Lunjun Zhang}{lunjun@cs.toronto.edu}

  \icmlkeywords{Large Language Models, Reinforcement Learning}

  \vskip 0.3in
]

\printAffiliationsAndNotice{}

\begin{abstract}

Reinforcement Learning (RL) has enabled Large Language Models (LLMs) to acquire increasingly complex reasoning and agentic behaviors. In this work, we propose two simple techniques to improve policy gradient algorithms for LLMs. First, we replace the fixed anchor policy during RL with an Exponential Moving Average (EMA), similar to a target network in deep Q-learning. Second, we introduce Top-$k$ KL estimator, which allows for flexible interpolation between exact KL and sampled KL.
We derive the stability conditions for using EMA anchor; moreover, we show that our Top-$k$ KL estimator yields both unbiased KL values and unbiased gradients at any $k$, while bringing the benefits of exact KL.
When combined with GRPO, the two techniques (\textbf{EMA-PG}) lead to a significant performance boost. On math reasoning, it allows R1-distilled Qwen-1.5B to reach 53.9\% on OlympiadBench compared to 50.8\% by GRPO. On agentic RL domains, with Qwen-3B base, EMA-PG improves GRPO by an average of 35.6\% across 7 datasets of Q\&A with search engines, including 29.7\% $\rightarrow$ 45.8\% on HotpotQA, 27.4\% $\rightarrow$ 40.5\% on 2WikiMultiHopQA. Overall, we show that EMA-PG is a simple, principled, and powerful approach to scaling RL for LLMs. Code: \href{https://github.com/LunjunZhang/ema-pg}{github.com/LunjunZhang/ema-pg}.

\vspace{-20pt}
\end{abstract}

\begin{figure}[t]
  \hspace{20pt}
  \includegraphics[width=0.8\linewidth]{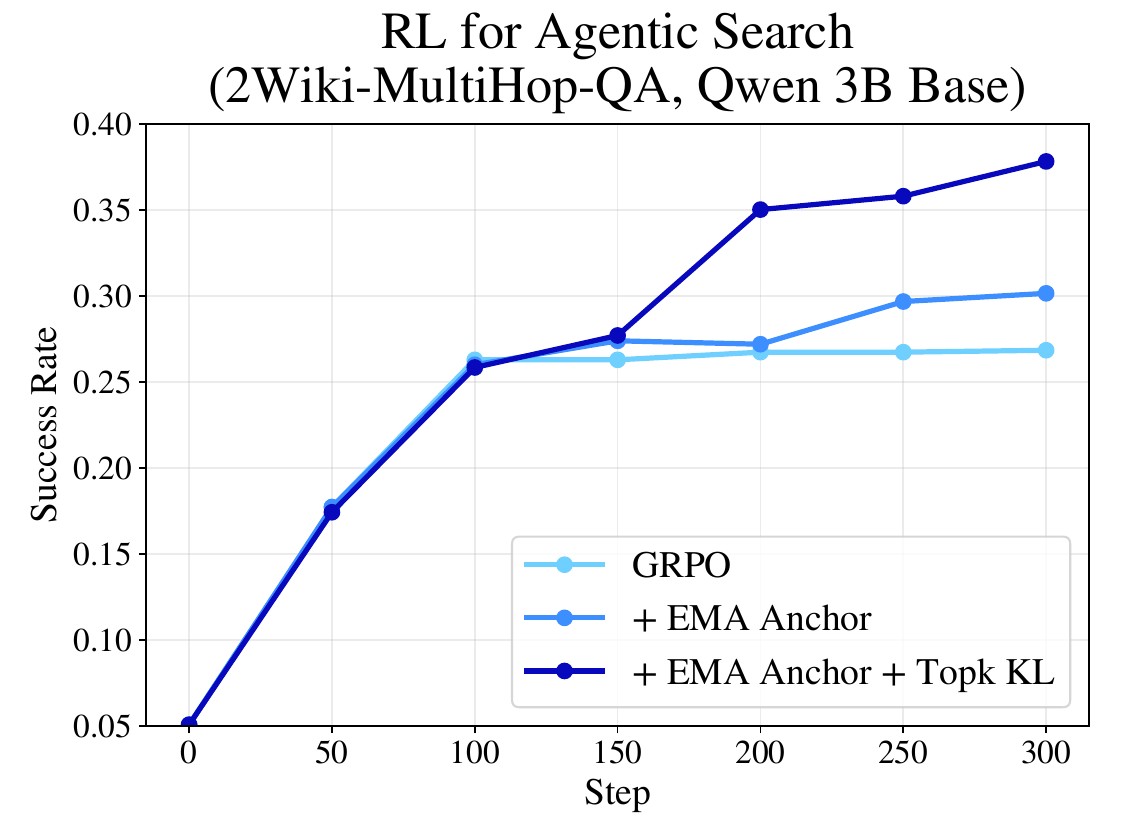} 
  \vspace{-5pt}
  \caption{\footnotesize Using EMA anchor policy and Top-$k$ KL estimator significantly improves the performance of RL algorithms like GRPO.}
  \label{fig:2-wiki-plot}
  \vspace{-2pt}
\end{figure}

\begin{table*}[t]
\setlength{\tabcolsep}{1.5pt}
\centering
    \caption{\centering\textbf{Token-level KL Estimators}. For brevity, we omit the conditioning of $\pi_{\theta}$ and $\pi_{\text{ref}}$ on past tokens. Notations: $\mathrm{sg}$ is stop-gradient,
    ${p\sim \pi_{\theta}(\cdot)}$ is the sampled token, ${r=\pi_{\theta}(p)/\mathrm{sg}(\pi_{\theta}(p))}$, ${w=\pi_{\text{ref}}(p)/\pi_{\theta}(p)}$; $\mathcal{V}$ is the vocabulary, ${\widehat{w}(j)=\pi_{\text{ref}}(j)/\pi_{\theta}(j)}, \forall j\in \mathcal{V}$.\\
    {\color{RoyalBlue}The estimators being highlighted in blue} are the KL estimators that provide both unbiased KL values and unbiased KL gradients.
    }
    \label{tab:kl-estimators}
    \vspace{-5pt}
    \begin{tabular}{lccccccc}
    \toprule
    $\E_{\pi_{\theta}}[\hat{D}_{f}(\pi_{\theta}, \pi_{\text{ref}})]$ & Expression & Value & Unbiased? & Gradient & Unbiased? & Memory & Overall \\
    \midrule
    \textbf{Reverse KL}: exact & $\sum_{j\in \mathcal{V}} \pi_{\theta}(j) [ -\log \widehat{w}(j) ]$ & $\mathrm{KL}(\pi_{\theta}, \pi_{\text{ref}})$ & ${\checkmark}$ & $\nabla_{\theta}\mathrm{KL}(\pi_{\theta}, \pi_{\text{ref}})$ & ${\checkmark}$ & {\color{lightgray}$O(|\mathcal{V}|)$} & {\color{lightgray}$\boldsymbol{\mathbin{\times}}$} \\
    K1 & $-\log w$ & $\mathrm{KL}(\pi_{\theta}, \pi_{\text{ref}})$ & ${\checkmark}$ & {\color{lightgray}$\nabla_{\theta}0$} & {\color{lightgray}$\boldsymbol{\mathbin{\times}}$} & $O(1)$ & {\color{lightgray}$\boldsymbol{\mathbin{\times}}$} \\
    K2 & $(\log w)^{2}/2$ & {\color{lightgray}$\E_{\pi_{\theta}}[(\log w)^{2}/2]$} & {\color{lightgray}$\boldsymbol{\mathbin{\times}}$} & $\nabla_{\theta}\mathrm{KL}(\pi_{\theta}, \pi_{\text{ref}})$ & ${\checkmark}$ & $O(1)$ & {\color{lightgray}$\boldsymbol{\mathbin{\times}}$} \\
    K3 & $-\log w + w -1$ & $\mathrm{KL}(\pi_{\theta}, \pi_{\text{ref}})$ & ${\checkmark}$ & {\color{lightgray}$\nabla_{\theta}\mathrm{KL}(\pi_{\text{ref}}, \pi_{\theta})$} & {\color{lightgray}$\boldsymbol{\mathbin{\times}}$} & $O(1)$ & {\color{lightgray}$\boldsymbol{\mathbin{\times}}$} \\
    \arrayrulecolor{gray!80}\hline
    ${\color{RoyalBlue}\textbf{K3}^{\boldsymbol{++}}}$ & $r \cdot (-\log w + w -1)$ & $\mathrm{KL}(\pi_{\theta}, \pi_{\text{ref}})$ & ${\checkmark}$ & $\nabla_{\theta}\mathrm{KL}(\pi_{\theta}, \pi_{\text{ref}})$ & ${\checkmark}$ & $O(1)$ & {\color{RoyalBlue}$\checkmark$} \\
    {\color{RoyalBlue}$\textbf{K4}$} & $r\cdot \mathrm{sg}(-\log w)$ & $\mathrm{KL}(\pi_{\theta}, \pi_{\text{ref}})$ & ${\checkmark}$ & $\nabla_{\theta}\mathrm{KL}(\pi_{\theta}, \pi_{\text{ref}})$ & ${\checkmark}$ & $O(1)$ & {\color{RoyalBlue}$\checkmark$} \\
    {\color{RoyalBlue}$\mathrm{Topk}$\textbf{Reverse}$\mathrm{KL}$} & Alg \ref{alg:topk-kl-reverse} & $\mathrm{KL}(\pi_{\theta}, \pi_{\text{ref}})$ & ${\checkmark}$ & $\nabla_{\theta}\mathrm{KL}(\pi_{\theta}, \pi_{\text{ref}})$ & ${\checkmark}$ & $O(k)$ & {\color{RoyalBlue}$\checkmark$} \\
    \midrule
    \textbf{Forward KL}: exact & $\sum_{j \in \mathcal{V}} \pi_{\text{ref}}(j) [ \log \widehat{w}(j) ]$ & $\mathrm{KL}(\pi_{\text{ref}}, \pi_{\theta})$ & ${\checkmark}$ & $\nabla_{\theta}\mathrm{KL}(\pi_{\text{ref}}, \pi_{\theta})$ & ${\checkmark}$ & {\color{lightgray}$O(|\mathcal{V}|)$} & {\color{lightgray}$\boldsymbol{\mathbin{\times}}$} \\
    \arrayrulecolor{gray!80}\hline
    {\color{RoyalBlue}\textbf{K5}} & $\mathrm{sg}(w)\log w + \log r$ & $\mathrm{KL}(\pi_{\text{ref}}, \pi_{\theta})$ & ${\checkmark}$ & $\nabla_{\theta}\mathrm{KL}(\pi_{\text{ref}}, \pi_{\theta})$ & ${\checkmark}$ & $O(1)$ & {\color{RoyalBlue}$\checkmark$} \\
    {\color{RoyalBlue}$\mathrm{Topk}$\textbf{Forward}$\mathrm{KL}$} & Alg \ref{alg:topk-kl-forward} & $\mathrm{KL}(\pi_{\text{ref}}, \pi_{\theta})$ & ${\checkmark}$ & $\nabla_{\theta}\mathrm{KL}(\pi_{\text{ref}}, \pi_{\theta})$ & ${\checkmark}$ & $O(k)$ & {\color{RoyalBlue}$\checkmark$} \\
    \bottomrule
  \end{tabular}
\vspace{-13pt}
\end{table*}

\section{Introduction}

As the pre-training paradigm began to show diminishing returns for large-language models (LLMs), reinforcement learning (RL) has re-emerged as a powerful mechanism for LLMs to continually self-improve based on self-generated experience. While RL was initially only used for alignment within the RL from Human Feedback (RLHF) framework \citep{rlhf, instructgpt, anthropic-hh, constitutional-ai}, it has since then been applied to reasoning models such as o1 \citep{openai-o1} and R1 \cite{deepseekr1}. Perhaps surprisingly, the predominant RL algorithm for LLMs today is still REINFORCE \cite{reinforce, sutton1999policy} and its modern off-policy variants like PPO \cite{ppo}, GRPO \cite{grpo}, and CISPO \cite{cispo}. This motivates us to study the foundational building blocks of policy gradient algorithms, so that such advancements can be broadly applied to LLMs.

The use of a target network has been the common practice since the dawn of deep Q-learning \cite{dqn}. A target network uses an Exponential Moving Average (EMA) to produce stable supervision targets. Using EMA has seen success in self-supervised learning as well \cite{byol}. In RLHF, the idea of an EMA anchor policy has been explored in TR-DPO \cite{tr-dpo} and WARP \cite{warp}. A crucial difference between our work and theirs is how the EMA anchor is used: prior work optimizes sequence-level KL, but we optimize token-level KL, which we show matters in reasoning and agentic settings.

KL regularization is a key component of policy optimization. Initially, it was a way to enforce trust-region \cite{trpo} and ensure training stability. In LLM alignment, it mitigates reward hacking \cite{rm-over-optimization}. In LLM reasoning, it helps preserve base model capabilities \cite{deepseek-v3}. However, what KL estimator to use for LLM has long been a curious question. InstructGPT \cite{instructgpt} adds a sequence-level sampled KL to the reward; GRPO \cite{grpo} uses token-level sampled KL; the two objectives are different \cite{kl-pitfalls}. Moreover, knowledge distillation \cite{kd} shows that leveraging logits in KL computation can significantly facilitate learning. We analyze the tradeoffs between sequence-level KL vs token-level KL, sampled KL vs exact KL, to identify the best practices in Top-k KL.

\begin{algorithm}[t]

\caption{${\color{RoyalBlue}\mathrm{Topk}\textbf{Reverse}\mathrm{KL}}(\pi_{\theta}, \pi_{\text{ref}}, \pi_{\text{sampling}}, q)$}
\label{alg:topk-kl-reverse}
\begin{algorithmic}[1]
\REQUIRE Policy $\pi_\theta$, Anchor $\pi_{\text{ref}}$ for $\mathrm{KL}(\pi_{\theta}, \pi_{\text{ref}})$
\REQUIRE Top-$k$ indices ${q}$ of {\color{RoyalBlue}sampling policy} $\pi_{\text{sampling}}$
\STATE $\mathrm{KL}_{\text{trun}} = \sum_{j\in q} \pi_{\theta}(j) [ \log \pi_{\theta}(j) - \log \pi_{\text{ref}}(j) ]$
\STATE {\color{RoyalBlue}$\Delta = \sum_{j\in q} [\pi_{\theta}(j) - \mathrm{sg}(\pi_{\theta}(j))]$}
\STATE Draw tail token $z \sim \mu$, \; $\mu(j) \propto \pi_{\text{sampling}}(j)\,\mathbbm{1}(j \not\in q)$
\STATE $r_{\theta}=\pi_{\theta}(z) / \mathrm{sg}(\pi_{\theta}(z))$, \quad {$w_{\theta}=\pi_{\text{ref}}(z) / \pi_{\theta}(z)$}
\STATE $\widehat{\mathrm{KL}}_{\text{sampled}} = {\color{RoyalBlue}r_{\theta}\cdot \mathrm{sg}(-\log w_{\theta})}$
\STATE $s = \mathrm{sg}(\pi_{\theta}(z) / \mu(z))$ \hfill {\color{Gray}\{importance weight\}}
\STATE \textbf{return} $\mathrm{KL}_{\text{trun}} {\color{RoyalBlue}\,-\,\Delta} + s \, \widehat{\mathrm{KL}}_{\text{sampled}}$
\end{algorithmic}
\end{algorithm}

We propose \textbf{two very simple changes} to policy gradients. Given an RL update for LLM $\pi_{\theta}$ on prompt $x$:
\begin{align*}
    \theta \leftarrow \theta + \alpha( \underbrace{{A}(x,y)\nabla_{\theta} \log \pi_{\theta}(y|x)}_{\text{Any policy gradient algorithm}} - \beta \underbrace{\nabla_{\theta} \widehat{\mathrm{KL}}}_{\text{What we change}} )
\end{align*}
where $y$ is the LLM response of length $L$, and $A(x,y)$ is the advantage function. We only change {\color{magenta}\textbf{\textit{what}}} the KL is regularizing against and {\color{RoyalBlue}\textbf{\textit{how}}} the KL is computed:
\begin{align*}
    \boxed{\widehat{\mathrm{KL}} = {\textstyle\sum}_{n=0}^{L-1} \! {\color{RoyalBlue}\mathrm{TopkKL}}( \pi_{\theta}(\cdot | x{,} y_{0:n-1}), \pi_{\color{magenta}\theta_{\text{ema}}}(\cdot | x{,} y_{0:n-1}) )}
\end{align*}
where ${\color{RoyalBlue}\mathrm{TopkKL}}$ is an unbiased token-level KL estimator that computes exact KL on the top-$k$ indices of logits and a sampled KL on the rest; ${\color{magenta}\theta_{\text{ema}}}$ is an EMA target network:
\begin{align*}
    \boxed{{\color{magenta}\theta_{\text{ema}}} \leftarrow  \eta \cdot  {\color{magenta}\theta_{\text{ema}}} + (1-\eta) \cdot \theta}
\end{align*}
The two changes make a significant difference on RL training of reasoning and agentic models. On math, it allows R1-distilled Qwen-1.5B to reach 53.9\% on OlympiadBench (text-only) \cite{olympiadbench} compared to 50.8\% by GRPO. In agentic settings, with Qwen-3B base, EMA-PG improves GRPO by 35.6\% across 7 datasets of Q\&A with search engine, including 29.7\% $\rightarrow$ 45.8\% on HotpotQA \cite{hotpotqa}, 27.4\% $\rightarrow$ 40.5\% on 2WikiMultiHopQA \cite{2wiki}, 12.8\% $\rightarrow$ 34.7\% on Bamboogle \cite{bamboogle}.

\begin{algorithm}[t]

\caption{${\color{RoyalBlue}\mathrm{Topk}\textbf{Forward}\mathrm{KL}}(\pi_{\theta}, \pi_{\text{ref}}, q)$}
\label{alg:topk-kl-forward}
\begin{algorithmic}[1]
\REQUIRE Policy $\pi_{\theta}$, Anchor $\pi_{\text{ref}}$ for $\mathrm{KL}(\pi_{\text{ref}}, \pi_{\theta})$
\REQUIRE Top-$k$ indices ${q}$ of {\color{RoyalBlue}anchor policy} $\pi_{\text{ref}}$
\STATE $\mathrm{FKL}_{\text{trun}} = \sum_{j\in q} \pi_{\text{ref}}(j) [ \log \pi_{\text{ref}}(j) - \log \pi_{\theta}(j)]$
\STATE {\color{RoyalBlue}$\Delta = \sum_{j\in q} [\pi_{\theta}(j) - \mathrm{sg}(\pi_{\theta}(j))]$}
\STATE Draw tail token $z \sim \mu$, \; $\mu(j) \propto \pi_{\text{ref}}(j)\,\mathbbm{1}(j \not\in q)$
\STATE $r_{\theta}=\pi_{\theta}(z) / \mathrm{sg}(\pi_{\theta}(z))$, \quad  {$w_{\theta}=\pi_{\text{ref}}(z) / \pi_{\theta}(z)$}
\STATE $\widehat{\mathrm{FKL}}_{\text{sampled}} = {\color{RoyalBlue} \mathrm{sg}(w_{\theta}) \log w_{\theta} + \log r_{\theta} }$
\STATE $s = \mathrm{sg}(\pi_{\theta}(z) / \mu(z))$ \hfill {\color{Gray}\{importance weight\}}
\STATE \textbf{return} $\mathrm{FKL}_{\text{trun}} {\color{RoyalBlue}\,+\,\Delta} + s \, \widehat{\mathrm{FKL}}_{\text{sampled}}$
\end{algorithmic}
\end{algorithm}

The rest of the paper is organized as follows. \textbf{(1)} We derive the stability condition for using an EMA anchor (\S\ref{sec:dynamics-ema-pg}). \textbf{(2)} We explain why token-level KL is more suited for reasoning models than sequence-level KL (\S\ref{sec:seq-kl-vs-token-kl}). \textbf{(3)} For token-level KL, common estimators such as K1, K2, K3 cannot yield unbiased KL estimates and unbiased gradients at the same time, so we construct estimators $\text{K3}^{++}$, K4, and K5 that can (\S\ref{sec:sampled-kl}). \textbf{(4)} We propose Top-k KL estimators, which extract dense supervision from logits with low memory footprint (\S\ref{sec:top-k-kl-estimator}). \textbf{(5)} We discuss how our analysis of KL extends to the policy gradient itself (\S\ref{sec:pg-as-sampled-kl}). Overall, EMA anchor and Top-k KL result in notable sample efficiency gain for RL.

\begin{figure*}[t]
  \centering
  \includegraphics[width=0.98\textwidth]{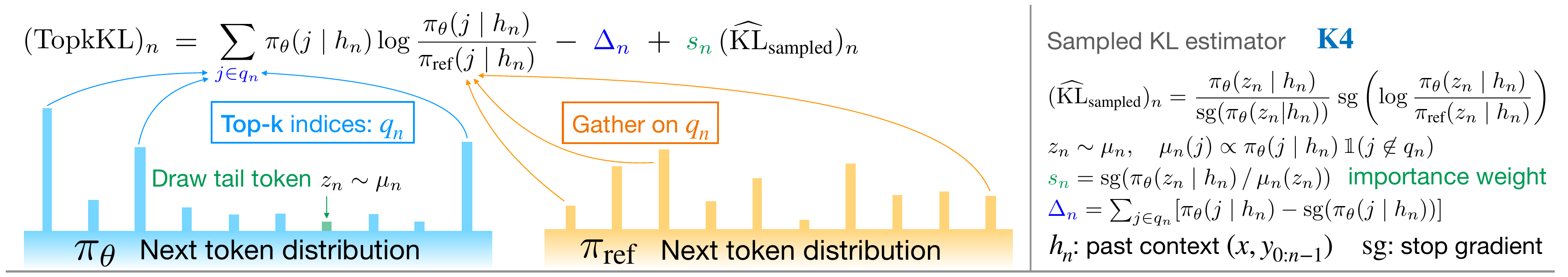}
  \vspace{-3pt}
  \caption{Top-k reverse KL (Alg \ref{alg:topk-kl-reverse}) computes exact KL on the top-k indices, and adds a sampled KL on a tail token, with a zero-valued correction $\Delta_n$, to keep the gradient estimator \emph{unbiased}.}
  \label{fig:ema-illustration}
  \vspace{-12pt}
\end{figure*}

\vspace{-5pt}
\section{Related Work}
\vspace{-2pt}
\paragraph{Improving RL for LLMs} Leveraging verifiable domains such as math and coding, lots of effort has gone into scaling RL for reasoning \cite{openai-o1, deepseekr1, cispo} and agentic tasks \cite{kimi-k2, gemini-2.5}. The predominant algorithm for RL in reasoning, GRPO \cite{grpo}, has gained widespread adoption due to its effectiveness. Many works have tried to further improve it, which include removing bias in policy gradient \cite{dr-grpo}, better importance weight clipping \cite{dapo}, using sequence-level objective \cite{gspo}, asynchronous data generation \cite{pipeline-rl}, better scaling practices \cite{art-of-scaling-rl}. Since we exclusively focus on the KL component in RL, our work is largely orthogonal to those techniques above.

\vspace{-12pt}
\paragraph{KL Estimators for LLMs} KL estimation for LLMs is challenging due to autoregressive action space and large vocabulary size. Early work on RLHF adds a sequence-level KL penalty to the sequence-level reward \cite{instructgpt, anthropic-hh}. As reasoning models begin to scale the sequence length of thinking traces, token-level KL regularization becomes popular \cite{deepseekr1}. Token-level sampled KL estimators are typically K1, K2, and K3 \citep{approximating-kl}; K3 was used in GRPO. Since then, many works have pointed out that K3 gradients are forward KL gradients \cite{kl-pitfalls, RPG, comedy-of-kl}. Sampled KL ignores the fact that we have white-box access to logits and next-token distributions; knowledge distillation has shown that such information is very useful to learning \cite{kd, sequence-kd, on-policy-distillation, distillspec}. When the vocabulary is large, storing full logits is too memory-intensive; our insight is to reduce the memory usage of logits via Top-k truncation.

\vspace{-12pt}
\paragraph{EMA in Deep Learning} Deep Q-learning almost universally uses an EMA network to produce stable learning targets \cite{dqn, ddpg, sac}. Self-supervised learning has also used EMA for stable bootstrapping \cite{momentum-contrast, byol, how-to-scale-ema}. EMA has been shown to improve generalization in deep learning \cite{reptile-algorithm, lookahead-optimizer}. In the diffusion model literature, EMA is often used to produce the final model \cite{ddpm}. For RL on LLMs, prior work has also suggested using an EMA anchor for alignment \cite{tr-dpo, warp}; our work differs in how the KL is computed. Elastic reset \cite{elastic-reset} also keeps an EMA copy during RL, but does not use it for reference policy. The effectiveness of EMA can be partially explained by Linear Mode Connectivity \cite{linear-mode-connectivity} and Monotonic Linear Interpolation \cite{qualitatively-characterizing-nn, analyzing-monotonic-linear-interpolation} between neural networks that share the same initialization.

\vspace{-5pt}
\section{Background}

\paragraph{Reinforcement Learning (RL)} Given a pre-trained LLM, RL trains the model to maximize a reward function while remaining close to the original model. Concretely, let $\pi_\theta$ denote a policy (language model) parameterized by $\theta \in \R^d$, and let $R(x, y)$ be a reward function scoring response $y$ to prompt $x$. The standard objective is to maximize:
\begin{align}
\label{eq:rlhf_objective}
\E_{x} \{\E_{\pi_\theta(y|x)}\bigl[R(x, y)\bigr] - \beta \KL(\pi_\theta (\cdot \mid x) \| \pi_{\text{ref}} (\cdot \mid x))\}
\end{align}
where $x\sim\mathcal{D}$ is a sampled prompt, $\pi_{\text{ref}}$ is the reference policy (also called the \textbf{anchor} policy), $\beta$ controls how much $\pi_{\theta}$ is allowed to deviate from the anchor. The term $\KL(\pi_\theta \| \pi_{\text{ref}}) = \E_{x}\E_{y \sim \pi_{\theta}(\cdot | x)}[\log (\pi_\theta(y|x) / \pi_{\text{ref}}(y|x))]$ is the reverse KL divergence for the policy $\pi_{\theta}$.

\vspace{-5pt}
\paragraph{Exponential Moving Average (EMA)} EMA is a technique for maintaining a smoothed version of an iteratively updated quantity.
Given a sequence $\{\theta_t\}_{t=0}^{\infty}$ and coefficient $\eta \in (0, 1)$, the EMA sequence $\{\theta_{\text{ema, }t}\}$ is:
\begin{align}
\theta_{\text{ema, } 0} = \theta_{0} \qquad \theta_{\text{ema, }{t+1}} &= \eta\,\theta_{\text{ema, } t} + (1-\eta)\,\theta_{t}
\end{align}
which is widely used in deep Q-learning as \textit{target network}.

\vspace{-5pt}
\paragraph{Local Quadratic Approximation of KL}
The Fisher information matrix of a policy $\pi_\theta$ is defined as:
\begin{equation*}
\scalebox{0.96}{$\displaystyle F_\theta \triangleq \E_{x \sim \mathcal{D}} \E_{y \sim \pi_\theta(\cdot | x)}\bigl[(\nabla_\theta \log \pi_\theta(y|x))(\nabla_\theta \log \pi_\theta(y|x))^\top\bigr]$}
\end{equation*}
which captures the local sensitivity of a policy’s output distribution to parameter changes. The natural gradient $F_\theta^{-1}g$ is the steepest-ascent direction under a local KL trust-region constraint. For $\theta'=\theta+\delta$ with $\|\delta\|\to 0$,
\begin{equation}
\label{eq:kl_quadratic}
\KL(\pi_\theta \| \pi_{\theta+\delta})
= \delta^\top F_{\theta}\delta / 2 + O(\|\delta\|^3),
\end{equation}
and the gradient w.r.t $\theta$ is:
$
\nabla_{\theta}\KL(\pi_\theta \| \pi_{\theta'}) \approx F_{\theta}(\theta-\theta')
$.

\begin{figure*}[t]
  \centering
  \includegraphics[width=0.92\textwidth]{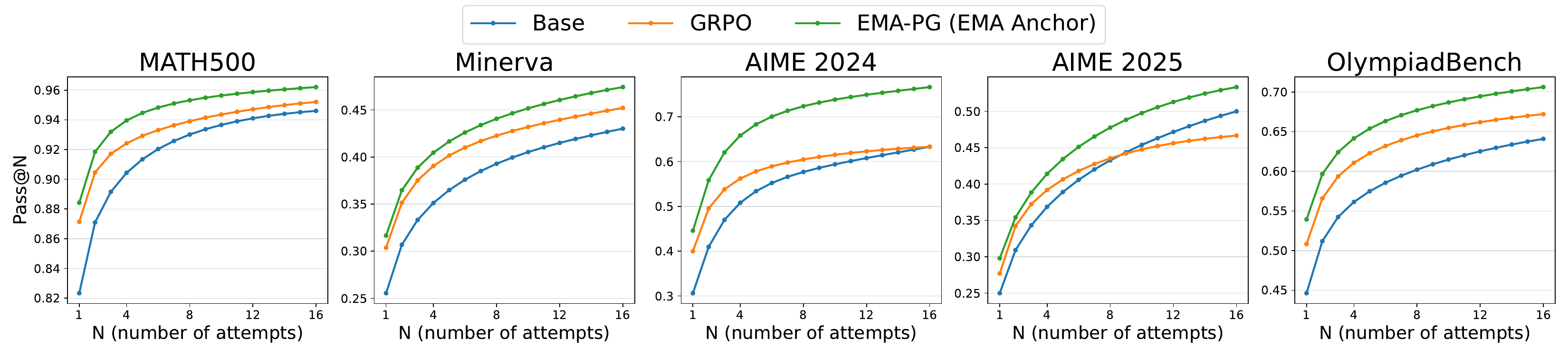}
  \caption{\textbf{On math reasoning datasets, using EMA Anchor improves not only Pass@1 but Pass@N as well compared to GRPO}. Results obtained on a 1.5B model {DeepSeek-R1-Distill-Qwen-1.5B} \cite{deepseekr1}.}
  \label{fig:math-results}
  \vspace{-13pt}
\end{figure*}

\vspace{-5pt}
\section{Training Dynamics of the EMA Anchor}
\label{sec:dynamics-ema-pg}

Consider the policy gradient $g_{t}= (R\,\nabla_{\theta} \log \pi_{\theta})|_{\theta=\theta_{t}}$ at $t$:
\begin{align*}
\theta_{t+1} &= \theta_t + \alpha [g_{t} - \beta\, \nabla_{\theta} \KL(\pi_{\theta} \| \pi_{\theta_{\text{ema, } t}})|_{\theta=\theta_{t}}] \\
\theta_{\text{ema, }t+1} &= \eta \,\theta_{\text{ema, }t} + (1-\eta) \, \theta_{t}
\end{align*}
We define the \textit{lag} between current and EMA weights as:
\begin{align}
\delta_t &\coloneqq \theta_t - \theta_{\text{ema, } t}\label{eq:def_delta}
\end{align}
Using the local approximation~\eqref{eq:kl_quadratic}, with $F_{t}=F_{\theta_t}$
\begin{align}
\label{eq:policy_update}
\theta_{t+1} &= \theta_t + \alpha g_{t} - \alpha\beta F_{t}\delta_{t}
\end{align}
The joint dynamics of $\theta_{t}$ and $\delta_{t}$ becomes:
\begin{equation}
\label{eq:tilde_matrix_update_singleline}
\begin{bmatrix}
\theta_{t+1}\\
\delta_{t+1}
\end{bmatrix}
=
\begin{bmatrix}
I & -\alpha\beta F_t \\
0 & \eta I - \alpha\beta F_t \\
\end{bmatrix}
\begin{bmatrix}
\theta_t\\
\delta_t
\end{bmatrix} + \begin{bmatrix}
    \alpha I \\
    \alpha I
\end{bmatrix} g_{t}
\end{equation}

\begin{lemma}{\underline{Quasi-steady state}:} If we assume that the gradient and the Fisher matrix change extremely slowly over $k$ steps, meaning that for $\Delta =0 \cdots k-1$:
\begin{align}
    g_{t+\Delta} \approx g, \qquad F_{t+\Delta} \approx F 
    \label{eq:ema-pg-assumption}
\end{align}
then, under the following definitions of matrices $D$, $S_k$, $M_{k}$:
\begin{align*}
    D &= \eta I-\alpha\beta F \quad S_{k} = {\sum}_{\Delta=0}^{k-1} D^{\Delta} \quad M_{k} = {\sum}_{\Delta=0}^{k-1}S_{\Delta}
\end{align*}
we have: firstly, $I-D$ is invertible and $S_{k} = (I-D^{k})(I-D)^{-1}$, $M_{k} = (kI - S_{k})(I-D)^{-1}$; secondly, we have a closed-form dynamics of EMA Policy Gradient:
\begin{equation*}
\begin{bmatrix}
\theta_{t+k}\\
\delta_{t+k}\\
g
\end{bmatrix}
=
\begin{bmatrix}
I & -\alpha\beta F S_{k} & \alpha kI - \alpha^{2}\beta F M_{k}\\
0 & D^{k} & \alpha S_{k}\\
0 & 0 & I
\end{bmatrix}
\begin{bmatrix}
\theta_t\\
\delta_t\\
g
\end{bmatrix}
\end{equation*}
\end{lemma}

\begin{lemma}{\underline{Stability Condition}:} Let $\{v_{i}\}_{i=0}^{d-1}$ be the eigenbasis of the Fisher matrix $F$, and $\{\lambda_{i}\}_{i=0}^{d-1}$ be its corresponding eigenvalues (with $\lambda_{\max}$ being the largest eigenvalue). The dynamics of the lag $\delta_{t}$ in the eigenbasis of the Fisher, where $\delta_{t}^{(i)}\coloneqq {\delta_{t}}^{\top} {v_{i}}$ and $g^{(i)}\coloneqq {g}^{\top} {v_{i}}$, can be written as:
\begin{align}
    \delta_{t+k}^{(i)} = \underbrace{(\chi_{i})^{k} \delta_{t}^{(i)}}_{\text{transient}} + \underbrace{\alpha {\textstyle\frac{ 1- (\chi_{i})^{k} }{1 - \chi_{i}}}g^{(i)}}_{\text{particular solution}} \quad \chi_i = \eta - \alpha\beta \lambda_i
\end{align}
Thus, the stability condition of $\delta$ is:
\begin{align}
    \boxed{\alpha\beta \lambda_{\max} < 1 + \eta} \label{eq:ema-pg-stability-condition}
\end{align}
\end{lemma}
\textbf{Summary of Stability / Oscillation Conditions:}
\begin{enumerate}
    \item $0\leq \alpha\beta \lambda_{\max} \leq \eta$: stable, non-oscillatory.
    \item $\eta < \alpha\beta \lambda_{\max} < 1+\eta$: stable, oscillatory.
    \item $\alpha\beta\lambda_{\max} \geq 1+\eta$: unstable.
\end{enumerate}
This characterizes the exact conditions for stability in EMA-PG's training dynamics under mild assumptions \eqref{eq:ema-pg-assumption}.

\begin{figure*}[t]
\vspace{-4pt}
  \centering
  \includegraphics[width=0.95\textwidth]{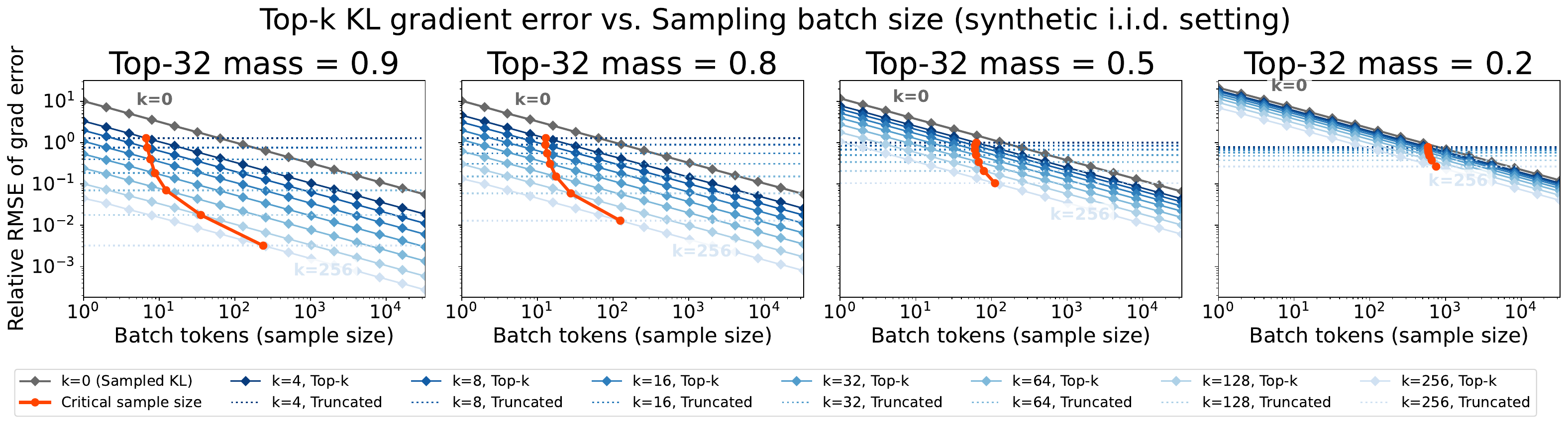}
  \vspace{-4pt}
  \caption{\textbf{Bias-variance tradeoff in Top-k KL estimator:} Top-$k$ KL (unbiased) has lower gradient error than sampled KL (unbiased) in all regimes, but only outperforms Truncated KL (biased) beyond a certain {\textit{critical sample size}}; we illustrate this bias-variance tradeoff in a synthetic setting (\S\ref{app:topk_kl_sim_exp}). {Rule of thumb: when using a small $k$, add the tail term \eqref{eq:tail-correction} (default); for a large $k$, consider Truncated KL.}}
  \label{fig:topk_kl_sim}
  \vspace{-12pt}
\end{figure*}

\section{KL Estimators: Which and Why}

\subsection{Sequence-level KL or Token-level KL?}
\label{sec:seq-kl-vs-token-kl}

Because LLMs have an autoregressive action space, it is worth discussing whether we should apply sequence-level KL or token-level KL. For a language model, the KL component in \eqref{eq:rlhf_objective} can be rewritten as the loss:
\begin{align}
    \widehat{\mathrm{KL}}_{\text{seq}} &\triangleq \log (\pi_{\theta}(y|x)/\pi_{\text{ref}}(y|x)) \\
    \nabla_{\theta}L_{\mathrm{KL}}^{\text{seq}} &\triangleq \E_{\pi_{\theta}(y|x)} [\mathrm{sg}( \widehat{\mathrm{KL}}_{\text{seq}} ) \nabla_{\theta}\log \pi_{\theta}(y|x)]
\end{align}
where $\mathrm{sg}(\cdot)$ denotes the stop gradient operator. Note that this is an unbiased \textbf{sequence-level KL} estimator:
\begin{align}
    \nabla_{\theta} \mathrm{KL}(\pi_{\theta}(\cdot\mid x) \| \pi_{\text{ref}}(\cdot\mid x)) = \nabla_{\theta}L_{\mathrm{KL}}^{\text{seq}}(\theta)
    \label{eq:unbiased-seq-level-kl}
\end{align}
This KL regularizer was used in both InstructGPT \citep{instructgpt} and WARP \citep{warp}. It is simple to implement: just adding a scalar sequence logprob difference to the scalar sequence-level reward would suffice. DPO \cite{dpo} also uses sequence-level KL.

In contrast, \textbf{token-level KL} regularization, as popularized by GRPO \citep{grpo, deepseekr1, deepseek-v3}, implements the following:
\begin{align}
    &\widehat{\mathrm{KL}}_{n} \approx {\mathrm{KL}}( \pi_{\theta}(\cdot \mid x, y_{0:n-1}) \| \pi_{\text{ref}}(\cdot \mid x, y_{0:n-1}) ) \\
    &\nabla_{\theta}L_{\mathrm{KL}}^{\text{token}} \triangleq \E_{\pi_{\theta}(y|x)} \Big[ {\sum}_{n=0}^{L-1} \nabla_{\theta} \widehat{\mathrm{KL}}_{n} \Big]
\end{align}
where the KL estimator is a design choice (See Tab \ref{tab:kl-estimators}). We assume the use of exact KL rather than sampled KL in our analysis of token-level KL. Token-level KL was traditionally seen as a surrogate objective for sequence-level KL;  here we analyze their differences. Rewrite sequence-level KL as:
\begin{align*}
    \rho_{n} &\triangleq \log (\pi_{\theta}(y_{n}|x, y_{0:n-1})/ \pi_{\text{ref}}(y_{n}|x, y_{0:n-1})) \\
    \omega_{n} &\triangleq \nabla_{\theta} \log \pi_{\theta}(y_{n}|x, y_{0:n-1})
\end{align*}
It can be shown that (\S\ref{app:seq-vs-token-kl}):
\begin{align}
    \nabla_{\theta} L_{\mathrm{KL}}^{\text{seq}} &= \E_{\pi_{\theta}(y|x)} \Big[\sum_{n}\Big( \sum_{{\color{red}m\geq n}} {\color{red}\rho_{m}} \Big)\, \omega_{n} \Big] \\
    \nabla_{\theta} L_{\mathrm{KL}}^{\text{token}} &= \E_{\pi_{\theta}(y|x)} \Big[\sum_{n}\big( {\color{red}\rho_{n}} \big) \, \omega_{n} \Big]
\end{align}
The difference is clear: {\textbf{sequence-level KL punishes each current token for the divergence of all its future tokens until the end of sequence}}, while token-level KL does not.

We find that in reasoning and agentic settings, since the sequence lengths of thinking traces are long and the rewards are reliable, {sequence-level KL does not help; \textbf{token-level KL should be used instead}}. In other words, during RL training for reasoning, KL should be token-wise in the loss, not added to the outcome reward, even though sequence-level KL may seem more “correct” \eqref{eq:unbiased-seq-level-kl} for objective \eqref{eq:rlhf_objective}.

\vspace{-5pt}
\subsection{Exact KL or Sampled KL?}
\label{sec:sampled-kl}

For token-level KL estimation, one can technically implement \textbf{exact KL}, which is unbiased and low-variance, but at the cost of memory. We use the following notations:
\begin{align}
    h_n \triangleq (x, y_{0:n-1}),\quad
    \E_{n} [\cdot] \triangleq \mathbb{E}_{y_n \sim \pi_{\theta}(\cdot \mid h_n)}[\cdot]
    \label{eq:token-level-definition}
\end{align}
In the case of reverse KL, for each token $y_n \sim \pi_{\theta}(\cdot \mid h_n)$:
\begin{align}
    \mathrm{KL}_{n} \triangleq {\sum}_{j\in \mathcal{V}} \pi_{\theta}(j \mid h_{n}) \log \dfrac{\pi_{\theta}(j \mid h_{n})}{\pi_{\text{ref}}(j \mid h_{n})} 
\end{align}
Since this is the ground-truth definition of reverse KL, it naturally yields exact values and exact gradients. The problem with exact KL is that it requires storing full logits for every position in each sequence both for current policy $\pi_{\theta}$ and reference policy $\pi_{\text{ref}}$, costing a memory of $O(|\mathcal{V}|)$ per token. The alternative, which has negligible memory footprint $O(1)$, is \textbf{sampled KL}. For the $n$-th position in $y$, let
\begin{align}
    y_{n} \sim \pi_{\theta}(\cdot \mid h_{n}) \quad 
    w_{n} \triangleq \pi_{\text{ref}}(y_n \mid h_{n}) / \pi_{\theta}(y_n \mid h_{n}) 
    \label{eq:sampled-kl-sampling-def}
\end{align}
The K1, K2, K3 estimators \cite{approximating-kl} are given by:
\begin{align*}
    {(\text{K1})_{n}} &\triangleq -\log w_{n} \quad 
    {(\text{K2})_{n}} \triangleq (1/2)(\log w_{n})^{2} \\
    {(\text{K3})_{n}} &\triangleq -\log w_{n} + w_{n} -1
\end{align*}
{\textbf{None of the existing sampled KL estimators simultaneously gives \underline{both} unbiased value and unbiased gradient}} on the token level. Ideally, we would like a KL estimator $\widehat{\mathrm{KL}}_{n}$ to satisfy both conditions (with notations from \eqref{eq:token-level-definition}):
\begin{equation}
\begin{aligned}
    \E_{n}[\widehat{\mathrm{KL}}_{n}] = \mathrm{KL}_{n} \quad
    \E_{n}[\nabla_{\theta} \widehat{\mathrm{KL}}_{n}] = \nabla_{\theta} \mathrm{KL}_{n} 
\end{aligned}
\label{eq:kl-estimator-two-conditions}
\end{equation}
K1 and K3 give the correct values but wrong gradients; K2 gives the correct gradients but wrong values (see Tab \ref{tab:kl-estimators}). We show that K1 and K3 can be fixed with a simple trick. Let
\begin{align}
    r_{n} \triangleq \pi_{\theta}(y_n \mid h_{n}) / \mathrm{sg}( \pi_{\theta}(y_n \mid h_{n}) )
    \label{eq:r-trick}
\end{align}
where $\mathrm{sg}$ means stop gradient. We construct the following unbiased sampled KL estimators to satisfy \eqref{eq:kl-estimator-two-conditions}:
\begin{align}
    &{\boxed{({\color{RoyalBlue}\textbf{K3}^{\boldsymbol{++}}})_{n} \triangleq r_{n} (-\log w_{n} + w_{n} -1)}} \label{eq:k3++} \\
    &{\boxed{({\color{RoyalBlue}\textbf{K4}})_{n} \triangleq r_{n} \cdot \mathrm{sg}(-\log w_{n})}}
\end{align}
Eq \eqref{eq:k3++} is named ${\color{RoyalBlue}\textbf{K3}^{\boldsymbol{++}}}$ because of its resemblance to the modified K3 in \cite{deepseek-v3}, but here we disentangle generic off-policy correction (which can be applied to any sampled KL estimator, not just K3; moreover, importance weights should be applied with $\mathrm{sg}(\cdot)$ on themselves) from an actual change in KL direction:
\begin{itemize}
    \item ${\color{RoyalBlue}\textbf{K3}^{\boldsymbol{++}}}$ is obtained by multiplying existing K3 by $r_{n}$ in \eqref{eq:r-trick}. \textbf{This simple change alone converts the gradient from forward KL to reverse KL.}
    \item ${\color{RoyalBlue}\textbf{K4}}$ is obtained by first  doing stop gradient on K1 and then multiplying it by $r_{n}$ in \eqref{eq:r-trick}. It is simpler, and it is unbiased as well.
\end{itemize}
${\color{RoyalBlue}\textbf{K3}^{\boldsymbol{++}}}$ and ${\color{RoyalBlue}\textbf{K4}}$ satisfy the conditions in \eqref{eq:kl-estimator-two-conditions}: they yield \underline{both} unbiased value and unbiased gradient for reverse KL.

\begin{figure*}[t]
\vspace{-5pt}
  \centering
  \includegraphics[width=\textwidth]{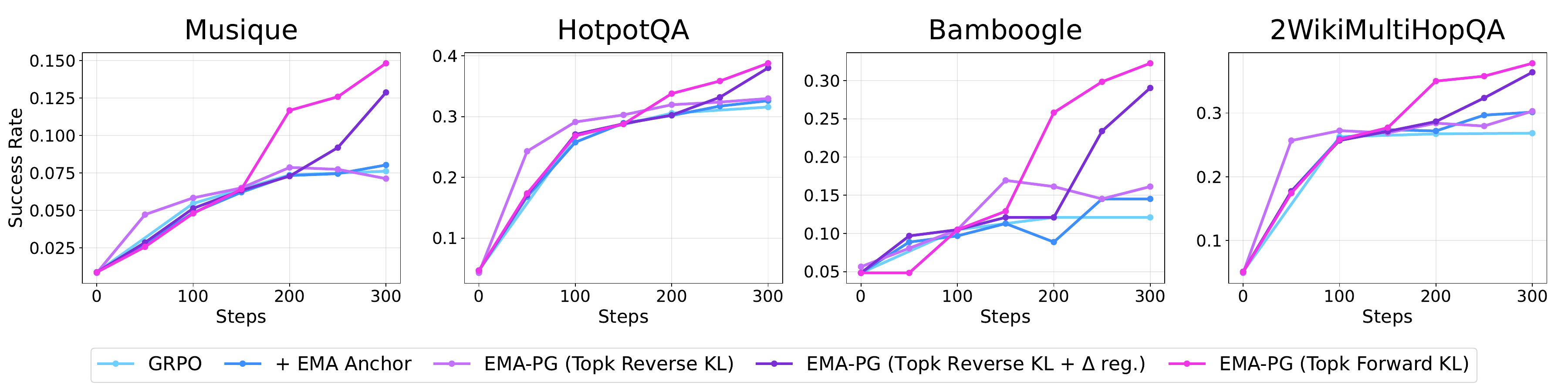}
  \vspace{-18pt}
  \caption{\textbf{EMA-PG not only learns faster but reaches higher asymptotic performance} on agentic tasks of Q\&A with search engine. Both Topk Forward KL and Topk Reverse KL with $\Delta$ reg.\ \eqref{eq:delta-correction} significantly outperform GRPO (and GRPO with EMA Anchor). This shows that the combination of EMA Anchor and Top-k KL is highly effective.}
  \label{fig:ema-pg-search-r1-curves}
  \vspace{-10pt}
\end{figure*}

In addition, the success of K3's forward KL gradients raises the interesting question of whether forward KL is just as effective as reverse KL in regularizing RL training dynamics. 
We propose a \textbf{forward KL estimator} ${\color{RoyalBlue}\textbf{K5}}$, which produces unbiased forward KL values and unbiased forward KL gradients under $\E_{n} [\cdot]$ \eqref{eq:token-level-definition} when sampling from \eqref{eq:sampled-kl-sampling-def}:
\begin{align}
    {\boxed{({\color{RoyalBlue}\textbf{K5}})_{n} \triangleq \mathrm{sg}(w_{n})\log w_{n} + \log r_{n}}}
\end{align}
When sampling is off-policy ($\pi_{\text{sampling}}\neq \pi_{\theta}$), any sampled KL estimator $\widehat{\mathrm{KL}}_{n}$ above simply needs to multiply the (stop-gradient) importance weight $\mathrm{sg}({\textstyle\frac{\pi_{\theta}}{\pi_{\text{sampling}}}})$ to remain unbiased. Concretely, when $y_{n}\sim \pi_{\text{sampling}}$, any previously unbiased $\widehat{\mathrm{KL}}_{n}$ under $y_n\sim \pi_{\theta}$ will require an off-policy correction $\mathrm{KL}_{n}=\E_{y_n}[s_{n} \widehat{\mathrm{KL}}_{n}]$ with importance weight (IW) $s_n$:
\begin{align*}
    \scalebox{0.965}{$\displaystyle y_{n} \sim \pi_{\text{sampling}}(\cdot | h_{n}) \quad s_{n} \triangleq \mathrm{sg}( \pi_{\theta}(y_{n} | h_{n}) / \pi_{\text{sampling}}(y_{n}|h_{n}) )$}
\end{align*}
On the other hand, exact KL is always on-policy (but memory-intensive, with $O(|\mathcal{V}|)$ rather than $O(1)$ memory footprint), and thus does not require off-policy correction.

\vspace{-5pt}
\subsection{Top-k KL Estimator}
\label{sec:top-k-kl-estimator}

We now aim to combine the best of both worlds of exact KL and sampled KL. We start from the observation that for modern LLMs, the probability mass tends to concentrate on a small number of top tokens at each generation step; it was for this reason that top-$k$ sampling \citep{top-k-sampling} and nucleus sampling \citep{top-p-sampling} were effective. The vocabulary size for modern LMs is very large (typically 130k--260k), so most of the logit indices in $\mathcal{V}$ contribute very minimally to the KL value or KL gradients.

\begin{figure*}[t]
\vspace{-5pt}
  \centering
  \includegraphics[width=\textwidth]{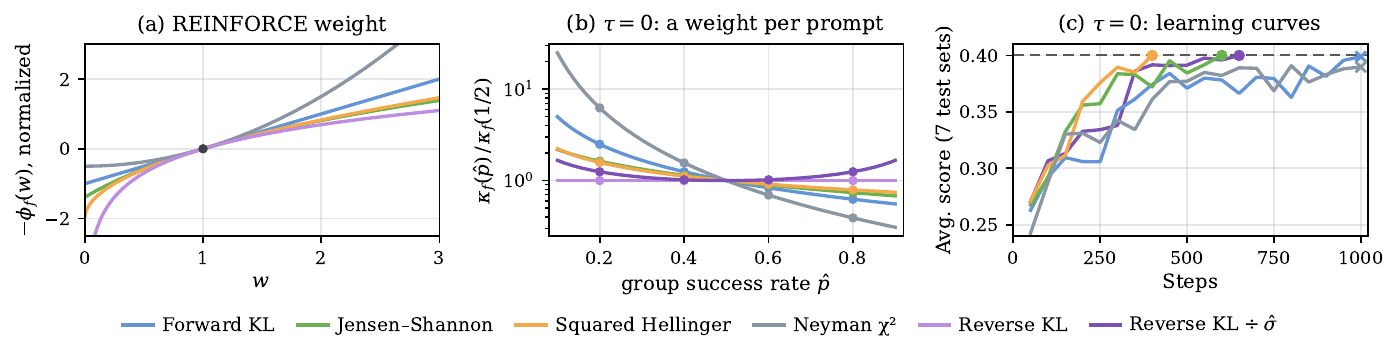}
  \vspace{-16pt}
  \caption{Policy gradient itself can be rewritten as a sampled $f$-divergence estimator. (a) Each $f$-divergence turns $w = \pi^{*}/\pi_{\theta}$ into its own REINFORCE weight. (b) At zero temperature, most $f$-divergences simply reweight prompts by success rate, favoring rarely solved ones; reverse KL weights them alike. (c) With Top-k forward KL as the regularizer, some reweightings learn faster than the standard advantage.}
  \label{fig:pg-fdiv}
  \vspace{-10pt}
\end{figure*}

\paragraph{Key insight.} We only partially compute exact KL on the top-k logit indices of the policy under which the expectation is taken (which is current policy for reverse KL, reference policy for forward KL), and then use a separate sampled KL for the rest of the indices, so that we arrive at an unbiased estimator for KL values and gradients, with lower variance than sampled KL under the same sampling distribution.

First, we consider Reverse KL. Let $q_n$ be the set of top-$k$ logit indices from $\pi_{\theta}(\cdot \mid h_{n})$ at sequence position $n$,
\begin{align*}
    q_n &\leftarrow \mathrm{Topk}(\pi_{\theta}(\cdot \mid h_{n})) \qquad
    q_n \subseteq \mathcal{V} \qquad |q_n|=k
\end{align*}
we can rewrite reverse KL by partitioning the vocabulary:
\begin{align*}
    \scalebox{0.965}{$\displaystyle \underbrace{\sum_{ \color{red}{j\in q_n } } {\pi_{\theta}(j | h_{n} ) \log \dfrac{ \pi_{\theta}(j | h_{n}) }{ \pi_{\text{ref}}(j | h_{n}) } }}_{ \text{`Truncated' KL: } (\mathrm{KL}_{\text{trun}})_{n} } + \underbrace{\sum_{ \color{red}{j\not\in q_n } } {\pi_{\theta}(j | h_{n} ) \log \dfrac{ \pi_{\theta}(j | h_{n}) }{ \pi_{\text{ref}}(j | h_{n}) } }}_{ \text{`Tail' term: } (\mathrm{KL}_{\text{tail}})_{n} }$}
\end{align*}
The first term is computed exactly. The second term is a sum over the non-top-$k$ indices, which we estimate with a separate sampled KL: we draw one \emph{tail token} $z_n$ from the non-top-$k$ part of the distribution, and evaluate K4 on it,
\begin{align}
    z_n &\sim \mu_n \qquad \mu_n(j) \propto \pi_{\theta}(j \mid h_{n})\,\mathbbm{1}(j \not\in q_n) \nonumber \\
    s_n &= \mathrm{sg}\big( \pi_{\theta}(z_n \mid h_{n}) \,/\, \mu_n(z_n) \big) \nonumber \\
    (\widehat{\mathrm{KL}}_{\text{tail}})_{n} &= s_n \, ({\color{RoyalBlue}\textbf{K4}})_{n}\big|_{y_n = z_n}
\label{eq:tail-correction}
\end{align}
where $({\color{RoyalBlue}\textbf{K4}})_{n}|_{y_n = z_n}$ is K4 with $z_n$ in place of the sampled token $y_n$, and $s_n$ is the importance weight of $z_n$ under its proposal $\mu_n$ (\S\ref{sec:sampled-kl}), here simply the probability mass of the non-top-$k$ indices. The tail token is drawn separately from the rollout token $y_n$, which could also estimate the tail term,
\begin{align}
    (\widehat{\mathrm{KL}}_{\text{high-var.\ tail}})_{n} = \mathbbm{1}(y_n \not\in q_n)\,({\color{RoyalBlue}\textbf{K4}})_{n}
    \label{eq:high-var-tail}
\end{align}
but with provably higher variance (\S\ref{app:topk-tail-variance}), since $y_n$ falls outside $q_n$ only with probability $s_n$; the extra tail token adds almost no overhead with an efficient GPU kernel.

\paragraph{The vanilla form has a gradient bias.} The vanilla estimator is $\widehat{\mathrm{KL}}_{n} = (\mathrm{KL}_{\text{trun}})_{n} + (\widehat{\mathrm{KL}}_{\text{tail}})_{n}$. Its \emph{value} is unbiased, but its \emph{gradient} is biased by one term:
\begin{align}
    \E_{z_n}\big[\nabla_{\theta}\widehat{\mathrm{KL}}_{n}\big] = \nabla_{\theta}\mathrm{KL}_{n} + {\sum}_{j\in q_n}\nabla_{\theta}\pi_{\theta}(j \mid h_n)
    \label{eq:topk-vanilla-bias}
\end{align}
where the expectation is over $z_n \sim \mu_n$, and the last term is the gradient of the probability mass of the top-$k$ indices. The reason is that K4's gradient relies on $\sum_{j}\nabla_{\theta}\pi_{\theta}(j \mid h_n) = 0$, which holds only over all of $\mathcal{V}$: the tail term drops its share of this sum, while the truncated term keeps its share $\sum_{j\in q_n}\nabla_{\theta}\pi_{\theta}(j \mid h_n)$, so they no longer cancel.

\paragraph{A gradient correction term.} We remove this bias with:
\begin{align}
    \Delta_n \triangleq {\sum}_{j\in q_n} \big[ \pi_{\theta}(j \mid h_n) - \mathrm{sg}(\pi_{\theta}(j \mid h_n)) \big]
    \label{eq:delta-correction}
\end{align}
\displaywidowpenalty=10000
which is zero in value, has gradient $\sum_{j\in q_n}\nabla_{\theta}\pi_{\theta}(j \mid h_n)$, and costs $O(k)$ since the top-$k$ probabilities are already computed. The Top-k reverse KL estimator is:
\begin{align}
\onelinebox{({\color{RoyalBlue}\mathrm{Topk}\textbf{Reverse}\mathrm{KL}})_{n} = (\mathrm{KL}_{\text{trun}})_{n} - \Delta_n + (\widehat{\mathrm{KL}}_{\text{tail}})_{n}}
\label{eq:topk-reverse-kl}
\end{align}
\displaywidowpenalty=50
Similarly, we split forward KL into a truncated exact KL and a separate sampled tail term, with the top-$k$ indices and the tail token now taken from the reference policy:
\begin{gather*}
    q_{n} \leftarrow \mathrm{Topk}(\pi_{\text{ref}}(\cdot | h_{n})) \hspace{1.4em} \mu_n(j) \propto \pi_{\text{ref}}(j \mid h_{n})\,\mathbbm{1}(j \not\in q_n) \\
    \begin{aligned}
        (\mathrm{FKL}_{\text{trun}})_{n} &= {\sum}_{j\in q_{n} } {\pi_{\text{ref}}(j \mid h_{n} ) \log \dfrac{ \pi_{\text{ref}}(j \mid h_{n}) }{ \pi_{\theta}(j \mid h_{n}) } } \\
        (\widehat{\mathrm{FKL}}_{\text{tail}})_{n} &= s_n \, ({\color{RoyalBlue}\textbf{K5}})_{n}\big|_{y_n = z_n} \qquad z_n \sim \mu_n
    \end{aligned}
\end{gather*}
with $s_n$ as in \eqref{eq:tail-correction}. The vanilla sum $\widehat{\mathrm{FKL}}_{n} = (\mathrm{FKL}_{\text{trun}})_{n} + (\widehat{\mathrm{FKL}}_{\text{tail}})_{n}$ has the same gradient bias with the opposite sign,
\begin{align}
    \scalebox{0.95}{$\displaystyle \E_{z_n}\big[\nabla_{\theta}\widehat{\mathrm{FKL}}_{n}\big] = \nabla_{\theta}\mathrm{FKL}_{n} - {\sum}_{j\in q_n}\nabla_{\theta}\pi_{\theta}(j \mid h_n)$}
    \label{eq:topk-vanilla-bias-forward}
\end{align}
where $\mathrm{FKL}_{n}$ is the forward KL at position $n$; here the cause is the baseline $\log r_n$ of K5, whose gradient has zero mean only over the full vocabulary. Adding the same $\Delta_n$ removes it, and the Top-k forward KL estimator is:
\begin{align}
    \onelinebox{({\color{RoyalBlue}\mathrm{Topk}\textbf{Forward}\mathrm{KL}})_{n} = (\mathrm{FKL}_{\text{trun}})_{n} + \Delta_n + (\widehat{\mathrm{FKL}}_{\text{tail}})_{n}}
    \label{eq:topk-forward-kl}
\end{align}
\begin{proposition} For any index set $q_n$, and any distribution $\mu_n$ of the tail token that is positive exactly on the indices outside $q_n$, Top-k reverse KL \eqref{eq:topk-reverse-kl} and Top-k forward KL \eqref{eq:topk-forward-kl} are unbiased in both values and gradients.
\label{prop:topk-unbiased}
\end{proposition}
We prove this in \S\ref{app:topk-unbiased-proof}. Top-k KL estimators are thus unbiased regardless of whether the chosen logit indices are top-$k$ or not. In practice, reverse KL (Alg \ref{alg:topk-kl-reverse}) takes $q_n$ and $\mu_n$ from the sampling policy before any truncation such as nucleus sampling.
Moreover, Top-k KL renders both exact KL ($k=|\mathcal{V}|$) and sampled KL ($k=0$) as special cases. Top-$k$ KL allows for flexible interpolation between the two, remains unbiased at any $k$, and only requires $O(k)$ memory. {We generalize our estimator to any ${f}$-divergence} in \S\ref{app:topk-fdiv}.

The bias in \eqref{eq:topk-vanilla-bias} is not noise: it is the gradient of the top-$k$ probability mass. The vanilla reverse form thus minimizes $\mathrm{KL}_{n} + \sum_{j\in q_n}\pi_{\theta}(j \mid h_n)$: it moves mass from the top-$k$ indices to the tail. Since $\Delta_n$ isolates this effect, we can also keep it on purpose, as a form of entropy regularization.

\begin{table*}[t]
\centering
\small
\setlength{\tabcolsep}{4.5pt}
\renewcommand{\arraystretch}{1.22}
\begin{tabular}{lcccccccc}
\toprule
\multirow{2}{*}{\textbf{Method}} &
\multicolumn{3}{c}{\textbf{General QA}} &
\multicolumn{5}{c}{\textbf{Multi-Hop QA}} \\
\cmidrule(lr){2-4}\cmidrule(lr){5-9}
& NQ & TriviaQA & PopQA & HotpotQA & 2wiki & Musique & Bamboogle & Avg. \\
\midrule

\multicolumn{9}{l}{Qwen2.5-\textbf{3B-Instruct} (Prompting)} \\
CoT (no search) \cite{cot}    & 0.106 & 0.288 & 0.108 & 0.149 & 0.244 & 0.020 & 0.024 & 0.134 \\
Search-o1 \cite{search-o1}           & 0.238 & 0.472 & 0.262 & 0.221 & 0.218 & 0.054 & 0.320 & 0.255 \\
RAG \cite{rag}           & 0.348 & 0.544 & 0.387 & 0.255 & 0.226 & 0.047 & 0.080 & 0.270 \\
\multicolumn{9}{l}{Qwen2.5-\textbf{7B-Instruct} (Prompting)} \\
RAG & 0.349 & 0.585 & 0.392 & 0.299 & 0.235 & 0.058 & 0.208 & 0.304 \\
\midrule
\multicolumn{9}{l}{Qwen2.5-\textbf{3B-Base} (RL)} \\
\textit{before RL} & 0.060  &  0.122  &  0.076  &  0.047  &  0.051  &  0.009  &  0.048  &  0.059 \\
WARP  &  0.114  &  0.197  &  0.116  &  0.103  &  0.140  &  0.014  &  0.056  &  0.106 \\
PPO    & {0.406} & {0.587} & {0.435} & 0.284 & 0.273 & 0.049 & 0.088 & 0.303 \\
GRPO  &  0.421  &  0.583  &  0.413  &  0.297  & 0.274  &  0.066  &  0.128  &  0.312  \\
\hdashline
\textbf{EMA-PG}: \textit{w/} {EMA Anchor}  &  0.468  &  {0.624}  &  0.445  &  0.345  &  0.277  &  0.084  &  0.194  &  0.348  \\
Topk {Reverse} KL  &  0.484  &  0.627  &  0.463  &  0.360  &  0.325  &  0.085  &  0.202  &  0.364 \\
\rowcolor{rowblue}
Topk {Reverse} KL $+$ $\Delta$ reg.  &  0.486  &  \textbf{0.635}  &  \textbf{0.467}  &  0.456  &  \textbf{0.414}  &  0.162  &  0.355  &  \textbf{0.425} \\
\rowcolor{rowyellow}
Topk {Forward} KL  &  \textbf{0.493}  &  0.627  &  0.464  &  \textbf{0.458}  &  0.405  &  0.171  &  0.347  &  0.423 \\
Topk {Forward} KL $+$ $\Delta$ reg.  &  0.492  &  0.630  &  0.464  &  0.448  &  0.395  &  \textbf{0.177}  &  \textbf{0.363}  &  0.424 \\
\midrule
Improvement over GRPO & {\color{Green} 17.1$\%\uparrow$} & {\color{Green} 7.5$\%\uparrow$} & {\color{Green} 12.3$\%\uparrow$}  & {\color{Green} 54.2$\%\uparrow$}  &  {\color{Green} 47.8$\%\uparrow$}  &  {\color{Green} 159.1$\%\uparrow$}  &  {\color{Green} 171.1$\%\uparrow$}  &  {\color{Green} 35.6$\%\uparrow$} \\
\bottomrule
\end{tabular}
\vspace{4pt}
\caption{\centering
\textbf{EMA-PG excels at agentic RL tasks that require using a search engine inside CoTs.} `$\Delta$ reg.'\ adds $\Delta_n$ \eqref{eq:delta-correction} as a regularizer.}
\label{tab:search-r1-table}
\vspace{-10pt}
\end{table*}

\subsection{Policy Gradient as a Sampled KL}
\label{sec:pg-as-sampled-kl}

Our KL analysis extends beyond the regularizer to the policy gradient itself, which can be written as a sampled divergence \cite{rwr, equivalence-between-pg-and-sq, sac} to a target $\pi^{*} \propto \pi_{\text{sampling}}\, e^{R/\tau}$ with temperature $\tau$, so the introduction's update has two components:
\begin{align}
    \theta \leftarrow \theta - \alpha \big( \nabla_{\theta} \widehat{D}_f(\pi_{\theta} \| \pi^{*}) + \beta\, \nabla_{\theta} \widehat{\mathrm{KL}} \big)
    \label{eq:two-components}
\end{align}
For any $f$-divergence (\S\ref{app:pg-as-sampled-f-divergence}), with the ratio $w = \pi^{*}/\pi_{\theta}$,
\begin{align}
    \nabla_{\theta} D_f(\pi_{\theta}\|\pi^{*}) = \E_{y\sim \pi_{\theta}}\big[\nabla_{\theta}\log \pi_{\theta}(y|x)\, \phi_f(w)\big]
    \label{eq:fdiv-pg-main}
\end{align}
where $\phi_f(w) = f'(1/w)$, so $-\phi_f(w)$ takes the place of the advantage $A(x,y)$ (Fig \ref{fig:pg-fdiv}a). With the rewards $R_i$ of a prompt's $N$ rollouts, ${w_i = e^{R_i/\tau} / \frac{1}{N}\sum_{j} e^{R_j/\tau}}$, and $-\phi_f(w_i)$ is used with a group-mean baseline. A typical policy gradient minimizes the reverse KL, whose weight $\log w_i - 1$, minus its group mean, is ${(R_i - \bar{R})/\tau}$ (\S\ref{app:new-family-of-pg}).

As $\tau \rightarrow 0$, the target keeps only each group's best rollouts, and $w_i$ becomes $N/m$ on the $m$ best and $0$ elsewhere. For finite $f'(\infty)$, rollout $i$'s weight minus its group mean is
\begin{align}
    \kappa_f(\hat{p})\,(b_i - \hat{p}) \qquad \kappa_f(\hat{p}) = f'(\infty) - f'(\hat{p})
    \label{eq:zero-temp-weight}
\end{align}
with $b_i = 1$ on the best rollouts and $\hat{p} = m/N$, the success rate for $0$/$1$ rewards. \textbf{At the zero-temperature limit, most $f$-divergences simply reweight prompts by their success rate} (Fig \ref{fig:pg-fdiv}b), favoring the rarely solved ones since $f$ is convex. Reverse KL weights all prompts alike, and GRPO divides its weight by the group's standard deviation $\hat{\sigma}$. In our experiments (Fig \ref{fig:pg-fdiv}c), some of those reweightings learn faster than the standard advantage; importantly, our Top-k KL regularizer works across a variety of $f$-divergences.

\vspace{-5pt}
\section{Experiments}

We evaluate the efficacy of EMA Policy Gradients in reasoning RL and agentic RL. The base algorithm we use is GRPO \cite{grpo} with clip-higher \cite{dapo}, a popular default for LLMs. To show that the two techniques we have introduced (EMA anchor and Topk KL) can benefit the base RL algorithm, we focus on the following tasks:
\begin{figure}[t]
  \centering
  \includegraphics[width=\linewidth]{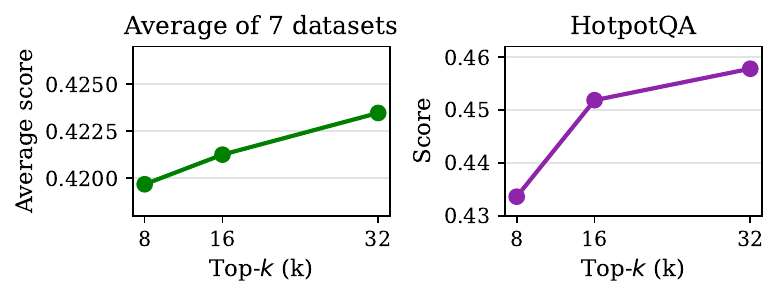}
  \vspace{-15pt}
  \caption{Ablation on $k$ of Top-k (forward) KL in EMA-PG.}
  \label{fig:top-k-ablation-of-k}
  \vspace{-3pt}
\end{figure}

\begin{figure*}[t]
  \centering
  \includegraphics[width=0.995\textwidth]{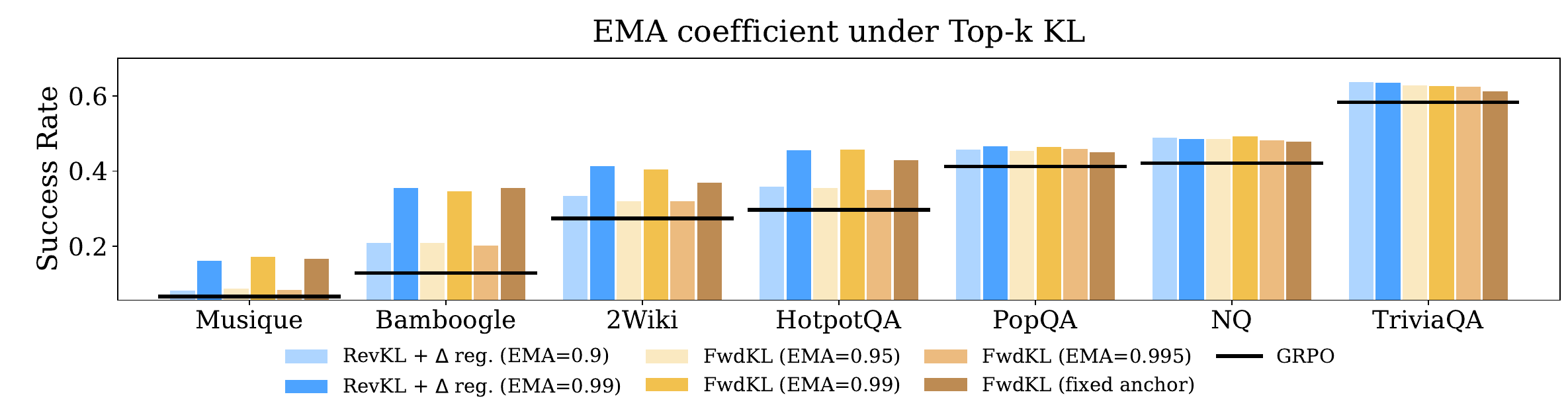}
  \vspace{-4pt}
  \caption{\textbf{Effect of the EMA coefficient and KL direction under Top-k KL} (with $k=32$). Both Topk Forward KL and Topk Reverse KL with $\Delta$ reg.\ \eqref{eq:delta-correction} perform best at EMA $\eta=0.99$ (updated every 10 steps), which is the value we recommend.}
  \label{fig:four_envs}
\end{figure*}

\begin{figure*}[h]
  \centering
  \includegraphics[width=0.995\textwidth]{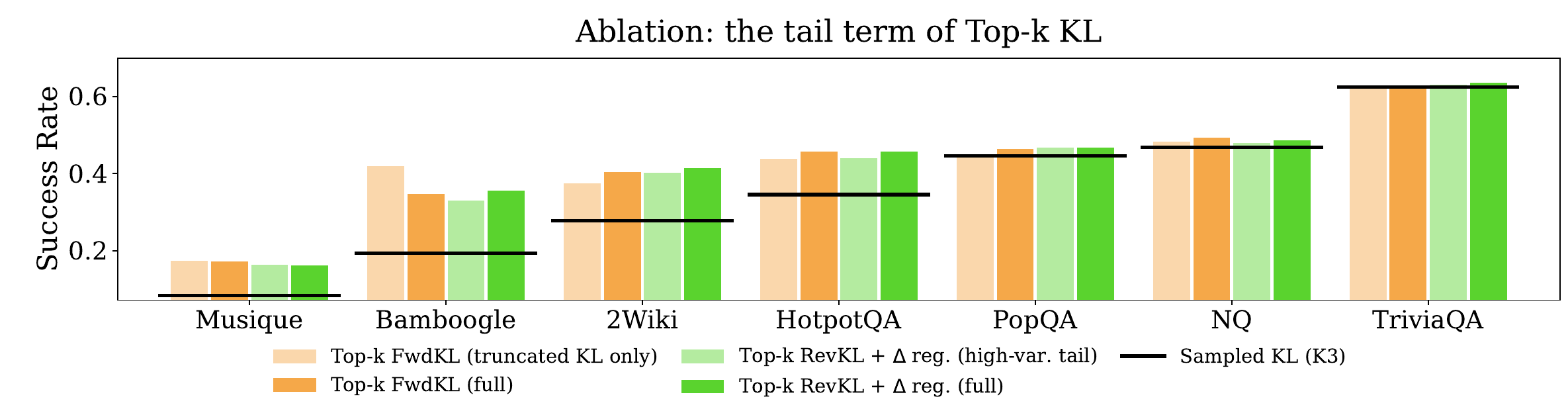}
  \vspace{-4pt}
  \caption{\textbf{Truncated KL alone already performs well}. With $k=32$, Topk Forward KL without the tail term (truncated KL only) performs close to the full estimator on average, and both outperform sampled KL (K3). For Topk Reverse KL with $\Delta$ reg., `high-var.\ tail' is a high-variance estimator of the tail term by re-using the rollout token \eqref{eq:high-var-tail} (\S\ref{app:topk-tail-variance}).}
  \label{fig:tail-correction-ablation}
  \vspace{-7pt}
\end{figure*}
\paragraph{Math Reasoning.} We follow the DeepScaleR setting \cite{deepscaler} and train DeepSeek-R1-Distill-Qwen-1.5B \cite{deepseekr1} on an RL dataset of 40k question-answer pairs that comprise of: AIME problems from 1984-2023, Omni-MATH dataset \cite{omni-math}, AMC problems prior to 2023, Still datasets \cite{Slow_Thinking_with_LLMs_1, Slow_Thinking_with_LLMs_2}. We then evaluate on the datasets of: MATH500 \cite{math500}, Minerva \cite{minerva}, AIME 2024 / 2025, and OlympiadBench \cite{olympiadbench}.

\vspace{-5pt}
\paragraph{Agentic RL for Search.} We follow the Search-R1 setting \cite{search-r1} and train a Qwen-2.5 3B \cite{qwen2.5} to use the search engine with RL. GRPO being used here applies the RL loss only on CoT and action tokens, not on environment observation tokens that interleave with CoTs. RL uses the combined training sets of Natural Questions (NQ) \cite{nq} and HotpotQA \cite{hotpotqa}. After RL, we run the eval on both General QA and Multi-Hop QA. For General QA, we use the test set of NQ, TriviaQA \cite{triviaqa}, PopQA \cite{popqa}. For Multi-Hop QA, we use the test set of HotpotQA, 2WikiMultiHopQA (2wiki) \cite{2wiki}, Musique \cite{musique}, and Bamboogle \cite{bamboogle}.

\vspace{-5pt}
\paragraph{Effect of EMA anchor} We empirically demonstrate the benefits of using an EMA target network instead of a fixed anchor for RL on LLMs. Note that using EMA does not require storing an additional set of weights during training, since the anchor (reference) policy weights already need to be stored for standard KL-regularized policy gradients. 

First of all, we empirically find that keeping the KL with a small $\beta>0$ is necessary for preventing training collapse when doing full-finetuning (no LoRA). This is corroborated by the common practice of still using KL in large-scale RL training runs \cite{deepseek-v3, kimi-k2}.

Secondly, we find that in math reasoning, using an EMA anchor ($\eta<1$) improves GRPO across all evaluated datasets, as shown in Tab \ref{tab:pass1-math-datasets}. For instance, using the EMA anchor gives us a boost of 50.8\% $\rightarrow$ 53.9\% on OlympiadBench. In addition, using an EMA anchor improves not just Pass@1 but Pass@16 as well, as shown in Fig \ref{fig:math-results}; this suggests that target networks can effectively mitigate Pass@N collapse that is common in reasoning RL \cite{rlvr-pass-at-n-collapse}.

Using token-level KL instead of sequence-level KL is crucial for good performance. On agentic search, we compare against WARP \cite{warp}, which uses sequence-level KL \eqref{eq:unbiased-seq-level-kl} and EMA anchor; we find that WARP performs a lot worse than vanilla GRPO (10.6\% vs 31.2\%). When using token-level KL + EMA anchor, we see a gain from 31.2\% $\rightarrow$ 34.8\%. This validates our analysis in \S\ref{sec:seq-kl-vs-token-kl}.

\begin{table}[t]
\centering
\small
\setlength{\tabcolsep}{4pt}
\renewcommand{\arraystretch}{1.27}
\begin{tabular}{lcc >{\columncolor{rowblue}}c c}
\toprule
R1-Distill-Qwen-1.5B & \textbf{Base} & \textbf{GRPO} & \textbf{EMA-PG} & {\color{Green}$\Delta$ (Rel)} \\
\midrule
MATH500        & 0.823 & 0.871 & \textbf{0.884} & {\color{Green}$27\%\uparrow$} \\ 
AMC 2024       & 0.499 & 0.656 & \textbf{0.686} & {\color{Green}$19\%\uparrow$} \\
Minerva        & 0.256 & 0.304 & \textbf{0.317} & {\color{Green}$27\%\uparrow$} \\
AIME 2024      & 0.306 & 0.400 & \textbf{0.446} & {\color{Green}$49\%\uparrow$} \\
AIME 2025      & 0.250 & 0.277 & \textbf{0.298} & {\color{Green}$77\%\uparrow$} \\
OlympiadBench (text)  & 0.446 & 0.508 & \textbf{0.539} & {\color{Green}$50\%\uparrow$} \\
\midrule
Average        & 0.430 & 0.503 & \textbf{0.528} & {\color{Green}$34\%\uparrow$} \\
\bottomrule
\end{tabular}
\vspace{2pt}
\caption{\centering
Pass@1 comparison across six math reasoning datasets. Relative Improvement $=(\text{EMA-PG}- \text{GRPO})/(\text{GRPO}-\text{Base})$.}
\label{tab:pass1-math-datasets}
\vspace{-10pt}
\end{table}

\paragraph{Effect of Top-k KL} We find that the use of Top-k KL can make a substantial difference in RL performance; both Topk Forward KL (Alg \ref{alg:topk-kl-forward}) and Topk Reverse KL (Alg \ref{alg:topk-kl-reverse}) with $\Delta$ reg.\ \eqref{eq:delta-correction} simultaneously improve sample efficiency and asymptotic performance of RL. It is shown in Fig \ref{fig:ema-pg-search-r1-curves} that, compared to sampled KL, Top-k KL enables the learning process to rapidly take off and avoid premature convergence. This is because Top-k KL works like knowledge distillation: it enforces KL on the logit level rather than just on the level of a single token sample, which provides dense supervision and lower gradient variance. In terms of the final performance (Tab \ref{tab:search-r1-table}), when used with EMA anchor, Top-k forward KL obtains a 34.8\% $\rightarrow$ 42.3\% gain compared to K3.

\begin{figure}[t]
  \centering
  \includegraphics[width=\linewidth]{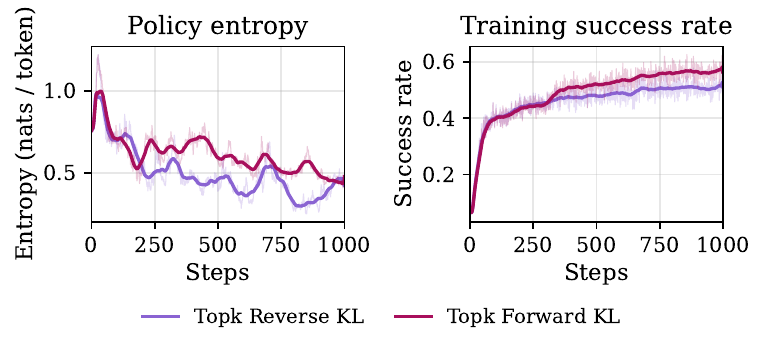}
  \vspace{-17pt}
  \caption{\textbf{Under Top-k KL, forward KL keeps a higher policy entropy than reverse KL for most of training.} Token-level policy entropy and training success rate, for Topk Reverse KL and Topk Forward KL of Tab \ref{tab:search-r1-table} (moving average over 50 steps). Topk Reverse KL plateaus at a lower training success rate.}
  \label{fig:topk-training-dynamics}
  \vspace{-4pt}
\end{figure}

We also find that \textbf{forward KL generally works better than reverse KL} in the absence of additional entropy regularization: on its own, Topk Reverse KL falls well behind Topk Forward KL (36.4\% vs 42.3\%), and it only catches up with the $\Delta$ regularizer of \S\ref{sec:top-k-kl-estimator}, which makes little difference to forward KL. Forward KL is more mode-covering while reverse KL is more mode-seeking: Fig \ref{fig:topk-training-dynamics} shows that forward KL keeps a higher policy entropy than reverse KL for most of training.

The benefits of Top-k KL only appear when $k>0$, but to save memory we need to keep $k$ small. In general, the larger the $k$ the better (with the limit being exact KL), but when $k$ is large, the tail term ($\widehat{\mathrm{KL}}_{\text{tail}}$ and $\widehat{\mathrm{FKL}}_{\text{tail}}$) does not necessarily help, meaning that truncated KL alone might perform better. We illustrate this \emph{bias-variance tradeoff} in a synthetic setting in Fig \ref{fig:topk_kl_sim}. For RL at scale, we empirically validate that increasing $k$ from $8\rightarrow 16 \rightarrow 32$ does lead to slightly improved performance (Fig \ref{fig:top-k-ablation-of-k}). Moreover, Fig \ref{fig:tail-correction-ablation} shows that with $k=32$, truncated KL alone performs close to the full Top-k KL (42.2\% vs 42.3\% on average). This suggests that at this $k$ the tail term can be dropped; we keep it by default, since it makes the estimator unbiased at almost no extra cost. The high-variance tail \eqref{eq:high-var-tail} also works in practice (41.6\%).

We also study what EMA coefficients are effective when using Top-k KL. In Fig \ref{fig:four_envs}, we find that $\eta=0.99$ for EMA is the best choice along with $k=32$ for Top-k KL, for both forward KL and reverse KL with $\Delta$ reg., and it also outperforms a fixed anchor.

\vspace{0pt}
\section{Conclusion}
In this work, we introduced two simple changes to policy gradients on LLMs: EMA anchor, which changes what KL is regularizing against, and Top-k KL, which changes how KL gradients are computed. We showed that the two changes are highly effective. Future work includes further investigation into Top-k $f$-divergence estimators (\S\ref{app:a-general-sampled-f-divergence-estimator}) beyond just KL, the application of Top-k KL to both offline and on-policy knowledge distillation of LLMs, and the application of our sampled KL (and more broadly, $f$-divergence) estimators to the policy gradient operator itself (\S\ref{app:new-family-of-pg}-- \S\ref{app:pg-as-sampled-f-divergence}).

\section*{Acknowledgments}

LZ was supported by the Schwartz Reisman Institute for Technology and Society at the University of Toronto through a Schwartz Reisman Fellowship.

\bibliography{main_paper}
\bibliographystyle{style_file}

\newpage
\appendix
\onecolumn

\clearpage

\section*{Appendix Contents}
\begin{itemize}
    \item \hyperref[sec:algo-box]{A\quad Algorithm Box for EMA Policy Gradient}
    \refpage{sec:algo-box}
    \item \hyperref[app:ema_pg_proofs]{B\quad Derivations for EMA Policy Gradient Dynamics}
        \begin{itemize}
            \item \hyperref[app:kl_quadratic_deriv]{B.1 \quad Policy Gradient Dynamics under Local Quadratic KL Approximation} 
            \apxpage{app:kl_quadratic_deriv}
            
            \item \hyperref[app:quasi_steady]{B.2 \quad Quasi-steady state analysis over a $k$-step window (Proof of Lemma 4.1)} 
            \apxpage{app:quasi_steady}
            
            \item \hyperref[app:eigen_analysis]{B.3 \quad Eigen-decomposition of Fisher and mode-wise dynamics (Proof of Lemma 4.2)}
            \apxpage{app:eigen_analysis}
            
            \item \hyperref[app:stability]{B.4 \quad Stability and oscillations}
            \apxpage{app:stability}
            
            \item \hyperref[app:steady_state]{B.5 \quad Steady-state lag and bounds}
            \apxpage{app:steady_state}
        \end{itemize}
    \item \hyperref[app:seq-vs-token-kl]{C\quad Sequence-level KL vs.\ Token-level KL}
        \begin{itemize}
            \item \hyperref[app:seq-level-kl-gradients]{C.1 \quad Sequence-level KL gradients}
            \apxpage{app:seq-level-kl-gradients}
            \item \hyperref[app:analysis-of-seq-level-kl-gradients]{C.2 \quad Analysis of sequence-level KL gradients}
            \apxpage{app:analysis-of-seq-level-kl-gradients}
            \item \hyperref[app:token-level-kl-gradients]{C.3 \quad Token-level KL gradients}
            \apxpage{app:token-level-kl-gradients}
        \end{itemize}
    \item \hyperref[app:sampled-reverse-kl]{D\quad Derivation of Sampled Reverse KL Estimator}
        \begin{itemize}
            \item \hyperref[app:sampled-estimators-and-their-expected-values]{D.1 \quad Sampled KL estimators (K1 / K2 / K3) and their expected values}
            \apxpage{app:sampled-estimators-and-their-expected-values}
            
            \item \hyperref[app:k1-k2-k3-gradients]{D.2 \quad Gradients of K1 / K2 / K3 under autodiff $\mathbb{E}_p[\nabla_\theta(\cdot)]$}
            \apxpage{app:k1-k2-k3-gradients}
            
            \item \hyperref[app:the-r-trick]{D.3 \quad The $\boldsymbol{r}$-trick restores $\nabla_\theta \mathbb{E}_p[\cdot]$ from $\mathbb{E}_p[\nabla_\theta(\cdot)]$.}
            \apxpage{app:the-r-trick}
            
            \item \hyperref[app:k3-plus-plus-gradients]{D.4 \quad ${\color{RoyalBlue}\textbf{K3}^{\boldsymbol{++}}}$ yields unbiased reverse-KL values and gradients}
            \apxpage{app:k3-plus-plus-gradients}
            
            \item \hyperref[app:k4-gradients]{D.5 \quad ${\color{RoyalBlue}\textbf{K4}}$ yields unbiased reverse-KL values and gradients}
            \apxpage{app:k4-gradients}
        \end{itemize}
    \item \hyperref[app:forward-kl-k5]{E\quad Derivation of Sampled Forward KL Estimator}
        \begin{itemize}
            \item \hyperref[app:forward-kl-and-its-gradient]{E.1 \quad Forward KL and its gradient}
            \apxpage{app:forward-kl-and-its-gradient}
            
            \item \hyperref[app:the-k5-estimator-has-unbiased-value]{E.2 \quad The {\color{RoyalBlue}\textbf{K5}} estimator has unbiased value}
            \apxpage{app:the-k5-estimator-has-unbiased-value}
            
            \item \hyperref[app:k5-estimator-has-unbiased-gradient]{E.3 \quad The {\color{RoyalBlue}\textbf{K5}} estimator has unbiased gradient}
            \apxpage{app:k5-estimator-has-unbiased-gradient}
            
            \item \hyperref[app:baseline-1-minus-w]{E.4 \quad Variance reduction in {\color{RoyalBlue}\textbf{K5}} gradients}
            \apxpage{app:baseline-1-minus-w}
        \end{itemize}
    \item \hyperref[app:topk_kl_sim_exp]{F\quad Bias-Variance Tradeoff in Top-k KL: An Illustrative Synthetic Setup}
    \refpage{app:topk_kl_sim_exp}
    
    \item \hyperref[app:topk-fdiv]{G\quad Top-k $\boldsymbol{f}$-Divergence Estimator}
        \begin{itemize}
            \item \hyperref[app:f-divergences-and-its-gradients]{G.1 \quad $f$-divergence and its gradients}
            \apxpage{app:f-divergences-and-its-gradients}
            
            \item \hyperref[app:a-general-sampled-f-divergence-estimator]{G.2 \quad A generic, unbiased, sampled $f$-divergence estimator}
            \apxpage{app:a-general-sampled-f-divergence-estimator}
            
            \item \hyperref[app:top-k-f-divergence]{G.3 \quad Top-k $f$-divergence estimator: A generalization of Top-k KL estimator}
            \apxpage{app:top-k-f-divergence}
        \end{itemize}
    \item \hyperref[app:token-level-rl]{H\quad Optimal Policies under Forward vs Reverse KL Regularization}
        \begin{itemize}
            \item \hyperref[app:reward-transformation]{H.1 \quad Reward transformation for equivalent optima between forward KL and reverse KL}
            \apxpage{app:reward-transformation}
            
            \item \hyperref[app:reward-transform-for-f-divergence]{H.2 \quad Reward transformation for equivalent optima across $f$-divergences}
            \apxpage{app:reward-transform-for-f-divergence}
            
        \end{itemize}
    \item \hyperref[app:new-family-of-pg]{I\quad Policy Gradient Methods can be Derived from Sampled KL Estimators}
    \begin{itemize}
        \item \hyperref[app:k1-leads-to-ppo-grpo]{I.1 \quad K1 estimator leads to on-policy PPO / GRPO loss}
        \apxpage{app:k1-leads-to-ppo-grpo}

        \item \hyperref[app:k2-leads-to-apa-kimi]{I.2 \quad K2 estimator leads to APA / Kimi loss}
        \apxpage{app:k2-leads-to-apa-kimi}

        \item \hyperref[app:k5-leads-to-rwr-mpo]{I.3 \quad K5 estimator leads to RWR / MPO loss}
        \apxpage{app:k5-leads-to-rwr-mpo}

        \item \hyperref[app:k3-k4-lead-to-combined-loss]{I.4 \quad K3 / K3$^{++}$ / K4 estimators lead to the same losses}
        \apxpage{app:k3-k4-lead-to-combined-loss}
        
    \end{itemize}
    \item \hyperref[app:pg-as-sampled-f-divergence]{J\quad Policy Gradient as $\boldsymbol{f}$-Divergence Minimization}
    \refpage{app:pg-as-sampled-f-divergence}
    \item \hyperref[app:topk-analysis]{K\quad Top-k KL: Unbiasedness and Variance}
        \begin{itemize}
            \item \hyperref[app:topk-unbiased-proof]{K.1 \quad Unbiased values and gradients (Proof of Proposition \ref*{prop:topk-unbiased})}
            \apxpage{app:topk-unbiased-proof}

            \item \hyperref[app:topk-tail-variance]{K.2 \quad Tail token versus rollout token}
            \apxpage{app:topk-tail-variance}

            \item \hyperref[app:topk-vs-sampled-variance]{K.3 \quad Top-k KL versus sampled KL}
            \apxpage{app:topk-vs-sampled-variance}
        \end{itemize}
\end{itemize}

\clearpage

\section{Algorithm Box for EMA Policy Gradient}
\label{sec:algo-box}

\begin{algorithm}[h]
\caption{{\color{magenta}\textbf{EMA}} Policy Gradient with {\color{RoyalBlue}$\mathrm{TopkKL}$}}

\begin{algorithmic}[1]
\label{alg:ema-pg}
\REQUIRE Prompt set $\mathcal{D}$, LLM $\pi_\theta(y|x)$, reward $R(x,y)$, anchor policy $\pi_{\theta_{\text{ema}}}(y|x)$.
\STATE {\color{magenta}$\theta_{\text{ema}} \leftarrow \theta$}
\FOR{iteration $t = 1,2,\dots$}
    \STATE Sample task prompt ${x}\sim \mathcal{D}$
    \STATE $\theta_{\text{old}} \leftarrow \theta$, sample ${y} \sim \pi_{\theta_{\text{old}}}(\cdot \mid {x})$ \hfill {\color{Gray}\{$h_n = (x, y_{0:n-1})$\}}
    \STATE Use $R({x}, {y})$ to compute advantage function ${A}$
    \FOR{$\upsilon = 1,\dots, \Lambda$}
        \STATE $s_{n} = \mathrm{sg}\big(\pi_{\theta}(y_{n} \mid h_n) / \pi_{\theta_{\text{old}}}(y_{n} \mid h_n)\big)$
        \STATE $\mathcal{J}(\theta)={\textstyle\sum}_{n} s_{n}\, A\, \log \pi_{\theta}(y_n\mid h_n)$, clip $\mathcal{J}(\theta)$ in PPO-style
        \STATE {\color{RoyalBlue}$\widehat{\mathrm{KL}} = {\textstyle\sum}_{n} \mathrm{TopkKL}\big(\pi_{\theta}(\cdot \mid h_n), \pi_{\color{magenta}\theta_{\text{ema}}}(\cdot \mid h_n)\big)$} \hfill {\color{Gray}\{Alg \ref{alg:topk-kl-reverse} for reverse KL, Alg \ref{alg:topk-kl-forward} for forward KL\}}
        \STATE $\theta \gets \theta + \alpha\, \nabla_\theta \big(\mathcal{J}(\theta) - \beta \, \widehat{\mathrm{KL}}\big)$
    \ENDFOR
    \IF{$t \text{ mod } T_{\text{ema}} = 0$}
        \STATE {\color{magenta}$\theta_{\text{ema}} \leftarrow \eta\,\theta_{\text{ema}} + (1-\eta)\, \theta$}
    \ENDIF
\ENDFOR
\end{algorithmic}
\end{algorithm}

The changes required to modify any policy gradient method into EMA-PG (with {\color{magenta}\textbf{EMA anchor}} and {\color{RoyalBlue}\textbf{Top-k KL}}) are highlighted above. 

\paragraph{Details of Top-k KL:}
For completeness, we write down the full expression of Top-k KL terms. Using common notations, with the anchor $\pi_{\text{ref}} = \pi_{\theta_{\text{ema}}}$:
\begin{align*}
    h_{n} = (x, y_{0:n-1}) \qquad \Delta_n = {\sum}_{j\in q_n} \big[ \pi_{\theta}(j \mid h_n) - \mathrm{sg}(\pi_{\theta}(j \mid h_n)) \big] \qquad z_n \sim \mu_n
\end{align*}
where $q_n$ is the set of top-$k$ indices, $\Delta_n$ is the correction term \eqref{eq:delta-correction}, and $z_n$ is the tail token, drawn from a proposal $\mu_n$ over the non-top-$k$ indices. The Top-k Reverse KL \eqref{eq:topk-reverse-kl} can be written as:
\begin{align*}
    & q_{n} \leftarrow \mathrm{Topk}(\pi_{\theta_{\text{old}}}(\cdot \mid h_{n})) \qquad \mu_n(j) \propto \pi_{\theta_{\text{old}}}(j \mid h_{n})\,\mathbbm{1}(j \not\in q_n) \\
    & (\widehat{\mathrm{KL}}_{\text{sampled}})_{n} = \dfrac{ \pi_{\theta} (z_n \mid h_n) }{ \mathrm{sg}(\pi_{\theta} (z_n \mid h_n) ) } \, \mathrm{sg} \bigg( \log \dfrac{ \pi_{\theta}(z_{n} \mid h_{n}) }{ \pi_{\text{ref}}(z_{n} \mid h_{n}) } \bigg) \\
    &({\color{RoyalBlue}\mathrm{Topk}\textbf{Reverse}\mathrm{KL}})_{n} = \sum_{ \color{red}{j\in q_{n}} } {\pi_{\theta}(j \mid h_{n} ) \log \dfrac{ \pi_{\theta}(j \mid h_{n} ) }{ \pi_{\text{ref}}(j \mid h_{n}) } } - \Delta_n + \mathrm{sg} \bigg( \dfrac{ \pi_{\theta}(z_{n} \mid h_{n}) }{ \mu_n(z_n) } \bigg) \cdot (\widehat{\mathrm{KL}}_{\text{sampled}})_{n}
\end{align*}
And the Top-k Forward KL \eqref{eq:topk-forward-kl} can be written as:
\begin{align*}
    & q_{n} \leftarrow \mathrm{Topk}(\pi_{\text{ref}}(\cdot \mid h_{n})) \qquad \mu_n(j) \propto \pi_{\text{ref}}(j \mid h_{n})\,\mathbbm{1}(j \not\in q_n) \\
    & (\widehat{\mathrm{FKL}}_{\text{sampled}})_{n} = \mathrm{sg} \bigg( \dfrac{ \pi_{\text{ref}} (z_n \mid h_n) }{ \pi_{\theta} (z_n \mid h_n) } \bigg) \log \dfrac{ \pi_{\text{ref}} (z_n \mid h_n) }{ \pi_{\theta} (z_n \mid h_n) } + \log \dfrac{\pi_{\theta}(z_n \mid h_{n})}{\mathrm{sg}( \pi_{\theta}(z_n \mid h_n ) )} \\
    &({\color{RoyalBlue}\mathrm{Topk}\textbf{Forward}\mathrm{KL}})_{n} = \sum_{ \color{red}{j\in q_{n}} } {\pi_{\text{ref}}(j \mid h_{n} ) \log \dfrac{ \pi_{\text{ref}}(j \mid h_{n} ) }{ \pi_{\theta}(j \mid h_{n}) } } + \Delta_n + \mathrm{sg} \bigg( \dfrac{ \pi_{\theta}(z_{n} \mid h_{n}) }{ \mu_n(z_n) } \bigg) \cdot (\widehat{\mathrm{FKL}}_{\text{sampled}})_{n}
\end{align*}
Then we sum up the token-wise KL along the sequence length:
\begin{align*}
    \widehat{\mathrm{KL}} = {\sum}_{n=0}^{L-1} ({\color{RoyalBlue}\mathrm{Topk}\textbf{Reverse}\mathrm{KL}})_{n} \qquad \widehat{\mathrm{FKL}} = {\sum}_{n=0}^{L-1} ({\color{RoyalBlue}\mathrm{Topk}\textbf{Forward}\mathrm{KL}})_{n}
\end{align*}

\newpage
\paragraph{Details of the Base RL Algorithm}
We discuss some additional implementation details below.
\begin{itemize}
    \item For the policy gradient loss function $\mathcal{J}(\theta)$, we use a standard version of GRPO: PPO-style clipping \cite{ppo}, Monte-Carlo estimation of advantages \cite{grpo, deepseek-v3}, with clip-higher \cite{dapo}:
    \begin{align*}
        &y^{(i)} \sim \pi_{\theta_{\text{old}}}(\cdot \mid x), \quad i=0,\dots, N-1 \\
        &V(x) = {\textstyle\frac{1}{N}} {\sum}_{i=0}^{N-1} R(x, y^{(i)}) \\
        &A(x,y^{(i)}) = R(x,y^{(i)}) - V(x) \qquad \mathrm{std}(x) = \sqrt{ {\textstyle\frac{1}{N-1}} {\sum}_{i=0}^{N-1} \Big( R(x, y^{(i)}) - V(x) \Big)^{2} } \\
        &s_{n}^{(i)} = \mathrm{sg} \bigg({ \frac{ \pi_{\theta}(y^{(i)}_{n} \mid x, y^{(i)}_{0: n-1}) }{ \pi_{\theta_{\text{old}}}(y^{(i)}_{n} \mid x, y^{(i)}_{0:n-1}) }} \bigg) \qquad
        M_{n}^{(i)} = \begin{cases}
            0 & \text{if } A(x,y^{(i)})>0 \text{ and } s_{n}^{(i)} > 1 + \epsilon_{\text{high}}, \\
            0 & \text{if } A(x,y^{(i)})<0 \text{ and } s_{n}^{(i)} < 1 - \epsilon_{\text{low}}, \\
            1 & \text{otherwise}.
        \end{cases} \\
        &\mathcal{J}^{\text{clip}}(x) = {\textstyle\frac{1}{N}} {\sum}_{i=0}^{N-1} {\textstyle\frac{1}{L^{(i)}}} \Big( {\sum}_{n=0}^{L^{(i)}-1} s_{n}^{(i)}\, M_{n}^{(i)} \, \log \pi_{\theta}(y^{(i)}_{n} \mid x, y^{(i)}_{0:n-1}) \Big) \, A(x,y^{(i)}) \\
        &\mathcal{J}(\theta) = \E_{x\sim \mathcal{D}} \Big[ \mathcal{J}^{\text{clip}}(x) \, / \,\mathrm{std}(x) \Big]
    \end{align*}
    where $L^{(i)}$ is the length of $y^{(i)}$, $\epsilon_{\text{high}}=0.28$ and $\epsilon_{\text{low}}=0.2$ (on agentic search, the average is over all the tokens of a batch instead of per rollout). $\mathcal{J}^{\text{clip}}$ applies the PPO-clip objective written in an equivalent form using binary masks \cite{truly-proximal, cispo}. Note that the two techniques of EMA-PG are orthogonal to how we compute $\mathcal{J}(\theta)$; we can in principle also choose to use CISPO \cite{cispo} or GSPO \cite{gspo} as the base algorithm.
    \item For the optimizer, we use Adam \cite{adam} with decoupled weight decay \cite{adamw}, with a learning rate of 1e-6 for \textit{full finetuning} (5e-7 for Topk Forward KL on agentic search). Since the focus of our study is to devise a better algorithm for full-scale post-training, we do not use any low-rank adaptation (LoRA) \cite{lora} in our experiments.
    \item $\beta=0.001$ for KL regularization, with the token-wise KL averaged over tokens in the same way as $\mathcal{J}^{\text{clip}}$. ${T_{\text{ema}}=10}$, following prior empirical practice among practitioners of Deep Q-learning that target networks (even though under EMA soft update) still only get updated every $T_{\text{ema}}$ iterations; \cite{zhang2021world} set the value to $10$ (see \texttt{target\_update\_freq} $=10$ in their implementation).
\end{itemize}

\paragraph{Details of the $f$-Divergence Runs}
In Fig \ref{fig:pg-fdiv}c, Reverse KL $\div\,\hat{\sigma}$ is Topk Forward KL of Tab \ref{tab:search-r1-table}, and the other runs differ from it in one place: in $\mathcal{J}(\theta)$, the weight $A(x,y^{(i)})/\mathrm{std}(x)$ of each rollout becomes
\begin{align*}
    W(x,y^{(i)}) = c\,\Big( {\textstyle\frac{1}{N}} {\sum}_{j=0}^{N-1} \phi_{f}(w_{j}) - \phi_{f}(w_{i}) \Big) \qquad
    w_{i} = \begin{cases} N/m & \text{if } R(x,y^{(i)}) = \max_{j} R(x,y^{(j)}) \\ 0 & \text{otherwise} \end{cases}
\end{align*}
where $w_i$ is the zero-temperature ratio of \S\ref{sec:pg-as-sampled-kl}, with $m$ the number of rollouts of $x$ with the highest reward, and $c$ is one scalar per batch that gives the weights of the batch the same sum of squares as GRPO's $A/\mathrm{std}$. The weight $-\phi_f(w)$ is $w$ for forward KL, $\frac{1}{2}\log\frac{1+w}{2}$ for Jensen--Shannon, $\sqrt{w}-1$ for squared Hellinger and $w^{2}-1$ for Neyman $\chi^{2}$ (\S\ref{app:pg-as-sampled-f-divergence}). The Top-k forward KL regularizer, $\beta$, the EMA anchor, the clipping and the learning rate are those of Topk Forward KL.

\clearpage

\section{Derivations for EMA Policy Gradient Dynamics}
\label{app:ema_pg_proofs}

\paragraph{Fisher matrix}
The Fisher matrix $F_\theta$ of a policy $\pi_\theta$ (defined above \eqref{eq:kl_quadratic}) is positive semidefinite (PSD), so all eigenvalues satisfy $\lambda_i\ge 0$. Consequently,
\begin{align}
\lambda_{\max}(F_\theta) \le \tr(F_\theta)
= \E_{\substack{x \sim \mathcal{D}\\y \sim \pi_\theta(\cdot \mid x)}}\!\big[\|\nabla_{\theta}\log\pi_{\theta}(y|x)\|_{2}^{2}\big]
\end{align}
giving us a practical way of estimating an upper bound on the largest Fisher eigenvalue of a policy under $\mathcal{D}$.

\paragraph{Setup.}
Let $J(\theta)\coloneqq R(\pi_\theta)$ denote the (expected) RL objective.
Define the EMA parameters and the \emph{lag}
\begin{align*}
\theta_{\ema,t+1} &= \eta\,\theta_{\ema,t} + (1-\eta)\,\theta_t,\qquad \eta\in[0,1), \\
\delta_t &\coloneqq \theta_t - \theta_{\ema,t}.
\end{align*}
Let $g_t\coloneqq \nabla_\theta J(\theta)\rvert_{\theta=\theta_t}$ be the policy-gradient (or any ascent direction).
We consider the update
\begin{align}
\theta_{t+1}
= \theta_t + \alpha\Big[g_t - \beta\,\nabla_\theta \KL(\pi_\theta\,\|\,\pi_{\theta_{\ema,t}})\Big]_{\theta=\theta_t},
\label{eq:app_theta_update_def}
\end{align}
with step size $\alpha>0$ and regularization coefficient $\beta\ge 0$.

\subsection{Policy Gradient Dynamics under Local Quadratic KL Approximation}
\label{app:kl_quadratic_deriv}

We assume a local quadratic approximation of the KL around $\theta_t$:
\begin{align}
\KL(\pi_{\theta_t}\,\|\,\pi_{\theta_{\ema,t}})
\;\approx\;
\frac{1}{2}(\theta_t-\theta_{\ema,t})^\top F_t(\theta_t-\theta_{\ema,t})
\;=\; \frac{1}{2}\delta_t^\top F_t\delta_t,
\label{eq:app_kl_quadratic}
\end{align}
where $F_t \coloneqq F_{\theta_t}$ is the (empirical) Fisher information matrix at $\theta_t$. Differentiating the RHS of \eqref{eq:app_kl_quadratic} w.r.t.\ $\theta_t$ yields
\begin{align}
\nabla_{\theta_t}\,\KL(\pi_{\theta_t}\,\|\,\pi_{\theta_{\ema,t}})
\;\approx\;
\nabla_{\theta_t}\Big(\tfrac{1}{2}\delta_t^\top F_t\delta_t\Big)
\;=\; F_t\delta_t,
\label{eq:app_grad_kl}
\end{align}
where we treat $F_t$ as locally constant under this approximation.
Substituting \eqref{eq:app_grad_kl} into \eqref{eq:app_theta_update_def} gives
\begin{align}
\theta_{t+1}
&= \theta_t + \alpha g_t - \alpha\beta F_t\delta_t.
\label{eq:app_theta_update}
\end{align}

\begin{lemma}[Lag dynamics]
\label{lem:app_lag_dynamics}
Under \eqref{eq:app_theta_update} and the EMA recursion, the lag obeys
\begin{align}
\delta_{t+1} = (\eta I - \alpha\beta F_t)\,\delta_t + \alpha g_t.
\label{eq:app_delta_update}
\end{align}
\end{lemma}

\begin{proof}
By the EMA update, $\theta_t-\theta_{\ema,t+1} = \eta(\theta_t-\theta_{\ema,t}) = \eta\,\delta_t$. Therefore, using \eqref{eq:app_theta_update},
\begin{align}
\delta_{t+1}
= (\theta_{t+1}-\theta_t) + (\theta_t-\theta_{\ema,t+1})
= \alpha g_t - \alpha\beta F_t\delta_t + \eta\,\delta_t
= (\eta I - \alpha\beta F_t)\delta_t + \alpha g_t.
\end{align}
\end{proof}

Collecting \eqref{eq:app_theta_update} and \eqref{eq:app_delta_update} gives the affine joint update \eqref{eq:tilde_matrix_update_singleline}.

\subsection{Quasi-steady state analysis over a $k$-step window (Proof of Lemma 4.1)}
\label{app:quasi_steady}

\paragraph{Assumption (frozen gradient and Fisher).}
Fix an integer $k\ge 1$ and assume \eqref{eq:ema-pg-assumption} over the window $\{t,\dots,t+k-1\}$,
for some constant vectors/matrices $g,F$, with the corresponding constant operator $D \coloneqq \eta I - \alpha\beta F$.

\subsubsection{Closed form for $\delta_{t+k}$}
Unrolling \eqref{eq:app_delta_update} with constant $D$ and $g$ yields
\begin{align}
\delta_{t+k} = D^k\delta_t + \alpha\sum_{j=0}^{k-1}D^j g = D^k\delta_t + \alpha S_k\,g,
\qquad S_k \coloneqq \sum_{j=0}^{k-1} D^j.
\label{eq:app_delta_closed_form}
\end{align}

\subsubsection{Invertibility of $I-D$ and closed form for $S_k$}
\label{app:invertibility}
Let $v_i$ be an eigenvector of $F$, and $\lambda_i$ be its corresponding eigenvalue. Using $D=\eta I-\alpha\beta F$,
\begin{align}
Dv_i = \chi_i v_i, \qquad \chi_i \coloneqq \eta - \alpha\beta \lambda_i.
\label{eq:app_chi_def}
\end{align}
Since $\lambda_i\ge 0$ ($F\succeq 0$), $\eta<1$, and $\beta,\alpha\ge 0$, we have $\chi_i \le \eta < 1$.
Hence $1-\chi_i>0$ for all $i$, so $I-D$ is invertible. Consequently,
\begin{align}
S_k
= \sum_{j=0}^{k-1}D^j
= (I-D^k)(I-D)^{-1}.
\label{eq:app_Sk_closed_form}
\end{align}

\subsubsection{Closed form for $\theta_{t+k}$ and the $M_k$ matrix}
From \eqref{eq:app_theta_update} with frozen $F,g$,
\begin{align}
\theta_{t+k} &= \theta_t + \sum_{j=0}^{k-1}\big(\alpha g - \alpha\beta F\delta_{t+j}\big)
= \theta_t + \alpha k\,g - \alpha\beta F\sum_{j=0}^{k-1}\delta_{t+j}.
\label{eq:app_theta_unroll}
\end{align}
Using \eqref{eq:app_delta_closed_form} at times $t+j$, $\delta_{t+j} = D^j\delta_t + \alpha S_j g$, so
\begin{align}
\sum_{j=0}^{k-1}\delta_{t+j}
= S_k\delta_t + \alpha M_k\,g,
\qquad M_k \coloneqq \sum_{j=0}^{k-1} S_j.
\label{eq:app_Mk_def}
\end{align}
Then
\begin{align}
\theta_{t+k}
= \theta_t + \alpha k\,g - \alpha\beta F S_k\delta_t - \alpha^2\beta F M_k\,g.
\label{eq:app_theta_closed_form}
\end{align}

\paragraph{Closed form for $M_k$.}
Using $S_j=(I-D^j)(I-D)^{-1}$ from \eqref{eq:app_Sk_closed_form},
\begin{align}
M_k
= \sum_{j=0}^{k-1}(I-D^j)(I-D)^{-1}
= \Big(kI - \sum_{j=0}^{k-1}D^j\Big)(I-D)^{-1}
= (kI - S_k)(I-D)^{-1}.
\label{eq:app_Mk_closed_form}
\end{align}

Collecting \eqref{eq:app_theta_closed_form} and \eqref{eq:app_delta_closed_form}, one may write the dynamics as the homogeneous linear map in $(\theta,\delta,g)$ of Lemma 4.1.

\subsection{Eigen-decomposition of Fisher and mode-wise dynamics (Proof of Lemma 4.2)}
\label{app:eigen_analysis}

Since $F$ is a Fisher information matrix (or symmetric empirical approximation), we assume $F$ is symmetric PSD.
Therefore, there exists an orthonormal eigenbasis $\{v_i\}_{i=0}^{d-1}$ and eigenvalues $\lambda_i\ge 0$ such that $Fv_i = \lambda_i v_i$.
We refer to $v_i$ as the $i$-th Fisher eigenvector (eigen-direction). ``High curvature'' corresponds to large $\lambda_i$. Previously we have already shown that $\chi_i = \eta - \alpha\beta \lambda_i \le \eta < 1$,
and $\chi_i$ may be negative if $\alpha\beta\lambda_i>\eta$.
Now we define the scalar projections
\begin{align}
\delta_t^{(i)} \coloneqq v_i^\top \delta_t,\qquad g^{(i)}\coloneqq v_i^\top g.
\end{align}
Projecting \eqref{eq:app_delta_closed_form} onto $v_i$ yields
\begin{align}
\delta_{t+k}^{(i)}
&= v_i^\top D^k\delta_t + \alpha\,v_i^\top S_k g
= \chi_i^k \delta_t^{(i)} + \alpha\sum_{j=0}^{k-1}\chi_i^j\,g^{(i)} \\
&= \chi_i^k \delta_t^{(i)} + \alpha\,\frac{1-\chi_i^k}{1-\chi_i}\,g^{(i)}.
\label{eq:app_delta_mode_solution}
\end{align}
This yields the decomposition into a transient term and a particular solution.

\subsection{Stability and oscillations}
\label{app:stability}

\subsubsection{Stability condition}
The transient part of the lag dynamics is $\delta_{t+1}=D\delta_t$.
Stability (transient $\to 0$) requires $\rho(D)<1$, equivalently $\max_i |\chi_i| < 1$.
Since $\chi_i\le \eta<1$, the inequality $|\chi_i|<1$ is only violated if $\chi_i\le -1$ for some $i$.
The most negative eigenvalue occurs at $\lambda_{\max}\coloneqq \max_i \lambda_i$, so stability is equivalent to
\begin{align}
\eta - \alpha\beta \lambda_{\max} > -1
\quad\Longleftrightarrow\quad
\alpha\beta \lambda_{\max} < 1+\eta.
\label{eq:app_stability_condition}
\end{align}
This proves \eqref{eq:ema-pg-stability-condition}.

\subsubsection{Non-oscillatory vs.\ oscillatory transients}
From \eqref{eq:app_delta_mode_solution}, the sign of the transient in mode $i$ is governed by $\chi_i^k$.
If $\chi_i\in(0,1)$, the transient decays monotonically.
If $\chi_i\in(-1,0)$, the transient alternates sign each step (damped overshoot/undershoot), which we call ``oscillatory''.
Combined with stability \eqref{eq:app_stability_condition} yields the three regimes summarized in the main text: 
\begin{itemize}
\item If $0\le \alpha\beta\lambda_{\max}\le \eta$, then $\chi_i\in[0,\eta]\subset[0,1)$ for all $i$, so the dynamics is stable and all transients decay monotonically (non-oscillatory).
\item If $\eta<\alpha\beta\lambda_{\max}<1+\eta$, then stability holds ($|\chi_i|<1$ for all $i$) but there exist high-curvature directions with $\chi_i<0$, producing damped sign-flipping transients (oscillations).
\item If $\alpha\beta\lambda_{\max}\ge 1+\eta$, then $\chi_i\le -1$ for the top-curvature direction, so $\max_i|\chi_i|\ge 1$ and the lag dynamics is unstable.
\end{itemize}

\subsection{Steady-state lag and bounds}
\label{app:steady_state}

Assume stability \eqref{eq:app_stability_condition}, so that $D^k\to 0$ as $k\to\infty$.
Taking $k\to\infty$ in \eqref{eq:app_delta_closed_form} gives the steady lag
\begin{align}
\delta_*
= \lim_{k\to\infty}\delta_{t+k}
= \alpha\Big(\sum_{j=0}^{\infty} D^j\Big) g
= \alpha(I-D)^{-1}g
= \alpha\big((1-\eta)I+\alpha\beta F\big)^{-1}g,
\label{eq:app_delta_star}
\end{align}
using $I-D = (1-\eta)I + \alpha\beta F$. In each eigen-direction,
\begin{align}
\delta_*^{(i)} = \frac{\alpha}{(1-\eta)+\alpha\beta\lambda_i}\,g^{(i)}.
\label{eq:app_delta_star_mode}
\end{align}

\subsubsection{Euclidean norm upper bound}
Since $F\succeq 0$, we have the PSD ordering $(1-\eta)I+\alpha\beta F \succeq (1-\eta)I$.
For positive definite matrices, PSD ordering reverses under inversion, hence $\big((1-\eta)I+\alpha\beta F\big)^{-1} \preceq \frac{1}{1-\eta}I$.
Therefore,
\begin{align}
\|\delta_*\|_2
= \alpha\Big\|\big((1-\eta)I+\alpha\beta F\big)^{-1}g\Big\|_2 \le \alpha\Big\|\big((1-\eta)I+\alpha\beta F\big)^{-1}\Big\|_{2}\,\|g\|_2
\;\le\; \alpha\cdot \frac{1}{1-\eta}\,\|g\|_2.
\label{eq:app_delta_norm_bound}
\end{align}

\subsubsection{Quasi-steady KL magnitude}
Under the quadratic KL approximation \eqref{eq:app_kl_quadratic}, the quasi-steady KL is
\begin{align}
\KL_*
\;\approx\;
\frac{1}{2}\delta_*^\top F\delta_*
= \frac{1}{2}\sum_{i=0}^{d-1}\lambda_i\big(\delta_*^{(i)}\big)^2,
\label{eq:app_kl_star_def}
\end{align}
expanding in the eigenbasis, $F=\sum_i \lambda_i v_i v_i^\top$ and $\delta_*=\sum_i \delta_*^{(i)} v_i$.
Substituting \eqref{eq:app_delta_star_mode} yields
\begin{align}
\KL_*
\;\approx\;
\frac{\alpha^2}{2}\sum_{i=0}^{d-1}
\frac{\lambda_i\big(g^{(i)}\big)^2}{\big((1-\eta)+\alpha\beta\lambda_i\big)^2}.
\label{eq:app_kl_star_eigs}
\end{align}
Thus, within the quasi-steady model, increasing $\eta$ increases $\KL_*$ unless $\beta$ is increased or $\alpha$ is decreased.

\newpage

\section{Sequence-level KL vs.\ Token-level KL}
\label{app:seq-vs-token-kl}

\paragraph{Notation and setup}
For prompt $x$, consider the conditional distribution over output sequence $y=(y_0,\dots,y_{L-1})$.
\begin{align*}
\pi_\theta(y\mid x) = \prod_{n=0}^{L-1}\pi_\theta(y_n \mid x, y_{0:n-1}).
\end{align*}
For each position $n$, define the prefix (history) $h_{n} \triangleq (x, y_{0:n-1})$,
and the per-token log-ratio and score term
\begin{align*}
\rho_n \triangleq \log\frac{\pi_\theta(y_n\mid h_n)}{\pi_{\mathrm{ref}}(y_n \mid h_n)},
\qquad
\omega_n \triangleq \nabla_\theta \log \pi_\theta(y_n \mid h_n).
\end{align*}
Note that $\log\frac{\pi_\theta(y\mid x)}{\pi_{\mathrm{ref}}(y\mid x)} = \sum_{n=0}^{L-1}\rho_n$
and $\nabla_\theta \log \pi_\theta(y\mid x) = \sum_{n=0}^{L-1}\omega_n$.

\subsection{Sequence-level KL gradients}
\label{app:seq-level-kl-gradients}
Define the (sequence-level) conditional KL:
\begin{align*}
\mathrm{KL}\big(\pi_\theta(\cdot\mid x)\,\|\,\pi_{\mathrm{ref}}(\cdot\mid x)\big)
\triangleq
\mathbb{E}_{y\sim\pi_\theta(\cdot\mid x)}
\left[
\log\frac{\pi_\theta(y\mid x)}{\pi_{\mathrm{ref}}(y\mid x)}
\right].
\end{align*}
Differentiate it:
\begin{align*}
\nabla_\theta \mathrm{KL}
&=
\nabla_\theta
\sum_{y} \pi_\theta(y\mid x)\,
\log\frac{\pi_\theta(y\mid x)}{\pi_{\mathrm{ref}}(y\mid x)}
=
\sum_{y} \nabla_\theta \pi_\theta(y\mid x)\,
\log\frac{\pi_\theta(y\mid x)}{\pi_{\mathrm{ref}}(y\mid x)}
+
\sum_{y} \pi_\theta(y\mid x)\,\nabla_\theta \log \pi_\theta(y\mid x).
\end{align*}
Using $\nabla_\theta \pi_\theta = \pi_\theta \nabla_\theta \log \pi_\theta$ and
$\sum_{y} \pi_\theta \nabla_\theta \log \pi_\theta = \mathbb{E}_{\pi_\theta}[\nabla_\theta \log \pi_\theta]=0$,
we obtain the standard score-function form \eqref{eq:unbiased-seq-level-kl}:
\begin{align}
\nabla_\theta \mathrm{KL}\big(\pi_\theta(\cdot\mid x)\,\|\,\pi_{\mathrm{ref}}(\cdot\mid x)\big)
=
\mathbb{E}_{y\sim\pi_\theta(\cdot\mid x)}
\left[
\log\frac{\pi_\theta(y\mid x)}{\pi_{\mathrm{ref}}(y\mid x)}
\;\nabla_\theta \log \pi_\theta(y\mid x)
\right].
\label{eq:seq_kl_score_fn}
\end{align}

\subsection{Analysis of sequence-level KL gradients}
\label{app:analysis-of-seq-level-kl-gradients}
Substitute the autoregressive decompositions:
\begin{align}
\nabla_\theta \mathrm{KL}\big(\pi_\theta(\cdot\mid x)\,\|\,\pi_{\mathrm{ref}}(\cdot\mid x)\big)
&=
\mathbb{E}_{y\sim\pi_\theta(\cdot\mid x)}
\left[
\left(\sum_{m=0}^{L-1}\rho_m\right)\left(\sum_{n=0}^{L-1}\omega_n\right)
\right] =
\mathbb{E}_{y\sim\pi_\theta(\cdot\mid x)}
\left[
\sum_{n=0}^{L-1}\sum_{m=0}^{L-1}\rho_m \omega_n
\right].
\end{align}
Split the inner sum into ``future'' and ``past'' relative to $n$:
\begin{align}
\sum_{m=0}^{L-1}\rho_m \omega_n
=
\left(\sum_{m\ge n}\rho_m\right)\omega_n
+
\left(\sum_{m<n}\rho_m\right)\omega_n.
\label{eq:future_past_split}
\end{align}

\paragraph{Why the ``past'' terms vanish} We use the conditional expectation notation
$\mathbb{E}_n[\cdot] \triangleq \mathbb{E}_{y_n \sim \pi_\theta(\cdot \mid h_n)}[\cdot]$.

\paragraph{Step 1: the score has zero mean.}
Conditioning on a fixed prefix $h_n$, we have
\begin{align*}
\mathbb{E}_n[\omega_n]
&=
\sum_{a} \pi_\theta(a \mid h_n)\,\nabla_\theta \log \pi_\theta(a \mid h_n)
=
\sum_{a} \nabla_\theta \pi_\theta(a \mid h_n)
=
\nabla_\theta \sum_{a} \pi_\theta(a \mid h_n)
=
\nabla_\theta 1
=
0.
\end{align*}
\paragraph{Step 2: past log-ratios are constant under $\mathbb{E}_n[\cdot]$.}
Since $\sum_{m<n}\rho_m$ depends only on the prefix $h_n=(x,y_{0:n-1})$, it does not depend on $y_n$ and can be
pulled out of $\mathbb{E}_n[\cdot]$:
\begin{align*}
\mathbb{E}_n\!\left[\left(\sum_{m<n}\rho_m\right)\omega_n\right]
=
\left(\sum_{m<n}\rho_m\right)\mathbb{E}_n[\omega_n]
=
0.
\end{align*}
\paragraph{Step 3: unconditional past terms also vanish.}
By the law of total expectation:
\begin{align}
\mathbb{E}_{y \sim \pi_\theta(\cdot \mid x)}
\left[\left(\sum_{m<n}\rho_m\right)\omega_n\right]
&=
\mathbb{E}_{y_{0:n-1} \sim \pi_\theta(\cdot \mid x)}
\Big[
\mathbb{E}_n\!\left[\left(\sum_{m<n}\rho_m\right)\omega_n\right]
\Big]
=
0,
\label{eq:past-term-unconditional-zero}
\end{align}
which justifies dropping the ``past tokens'' contribution in the main derivation.

\subsection{Token-level KL gradients}
\label{app:token-level-kl-gradients}
Define the token-level conditional KL at position $n$:
\begin{align*}
\mathrm{KL}_n
\triangleq
\mathrm{KL}\big(\pi_\theta(\cdot\mid h_n)\,\|\,\pi_{\mathrm{ref}}(\cdot\mid h_n)\big)
=
\mathbb{E}_{y_{n} \sim\pi_\theta(\cdot\mid h_n)}
\left[
\log\frac{\pi_\theta(y_{n} \mid h_n)}{\pi_{\mathrm{ref}}(y_{n} \mid h_n)}
\right].
\end{align*}
By the same argument as \eqref{eq:seq_kl_score_fn}, its gradient can be written as
\begin{align}
\nabla_\theta \mathrm{KL}_n
=
\mathbb{E}_{y_{n} \sim\pi_\theta(\cdot \mid h_n)}
\left[
\log\frac{\pi_\theta(y_{n} \mid h_n)}{\pi_{\mathrm{ref}}(y_{n} \mid h_n)}
\;\nabla_\theta \log \pi_\theta(y_{n} \mid h_n)
\right] = \mathbb{E}_{y_{n} \sim\pi_\theta(\cdot \mid h_n)} [ \rho_{n} \, \omega_{n} ].
\end{align}
Hence the token-level KL objective yields the per-token weighting:
\begin{align}
\nabla_\theta L_{\mathrm{KL}}^{\text{token}}
=
\E_{y\sim \pi_{\theta}(\cdot \mid x)} \bigg[ \sum_{n=0}^{L-1} \nabla_{\theta} \mathrm{KL}_{n} \bigg]
=
\mathbb{E}_{y\sim\pi_\theta(\cdot\mid x)}
\left[
\sum_{n=0}^{L-1}\rho_n \omega_n
\right].
\end{align}

\paragraph{Comparison}
Sequence-level KL weights token $n$ by the cumulative \emph{future} log-ratio $\sum_{m\ge n}\rho_m$,
while token-level KL weights token $n$ only by the \emph{local} log-ratio $\rho_n$.
\begin{align*}
    \nabla_{\theta} L_{\mathrm{KL}}^{\text{seq}} = \E_{\pi_{\theta}(y|x)} \Big[\sum_{n}\Big( \sum_{{\color{red}m\geq n}} {\color{red}\rho_{m}} \Big)\, \omega_{n} \Big]
    \qquad \nabla_{\theta} L_{\mathrm{KL}}^{\text{token}} = \E_{\pi_{\theta}(y|x)} \Big[\sum_{n}\big( {\color{red}\rho_{n}} \big) \, \omega_{n} \Big]
\end{align*}
In other words, \textbf{sequence-level KL} \textbf{punishes each current token for the divergence of all its future tokens}; token-level KL is only concerned with the current token.

\newpage

\section{Derivation of Sampled Reverse KL Estimator}
\label{app:sampled-reverse-kl}

We focus on a fixed token position $n$ with prefix $h_n=(x,y_{0:n-1})$.
Let $p \sim \pi_\theta(\cdot\mid h_n)$ be the sampled token, and define
\begin{align*}
w(p) \triangleq \frac{\pi_{\mathrm{ref}}(p\mid h_n)}{\pi_\theta(p\mid h_n)},
\qquad
\ell(p) \triangleq -\log w(p) = \log\frac{\pi_\theta(p\mid h_n)}{\pi_{\mathrm{ref}}(p\mid h_n)}.
\end{align*}
To avoid ambiguity, throughout this section we use the shorthand
$\mathbb{E}_{p}\big[\cdot\big]
\triangleq
\mathbb{E}_{p \sim \pi_\theta(\cdot\mid h_n)}\big[\cdot\big]$.

\paragraph{Reverse KL at a fixed prefix.}
The token-level reverse KL at history $h_n$ is (with $\mathcal{V}$ being the vocabulary):
\begin{align*}
\mathrm{KL}_n(\theta)
\;\triangleq\;
\mathrm{KL}\big(\pi_\theta(\cdot\mid h_n)\,\|\,\pi_{\mathrm{ref}}(\cdot\mid h_n)\big)
=
\sum_{j\in\mathcal V}\pi_\theta(j\mid h_n)\log\frac{\pi_\theta(j\mid h_n)}{\pi_{\mathrm{ref}}(j\mid h_n)}
=
\mathbb{E}_{p}\big[\ell(p)\big].
\end{align*}

\paragraph{Exact reverse-KL gradient.}
\begin{align*}
\nabla_\theta \mathrm{KL}_n(\theta)
=
\sum_{j\in\mathcal V}\pi_\theta(j\mid h_n)\,\ell(j)\,\nabla_\theta \log \pi_\theta(j\mid h_n)
=
\mathbb{E}_{p}\Big[\ell(p)\,\nabla_\theta \log \pi_\theta(p\mid h_n)\Big].
\end{align*}

\paragraph{$\mathbb{E}_p[\nabla_\theta(\cdot)]$ versus \ $\nabla_\theta \mathbb{E}_p[\cdot]$}
Since $\pi_\theta(\cdot\mid h_n)$ depends on $\theta$, $\mathbb{E}_{p\sim \pi_\theta(\cdot\mid h_n)}[\cdot]$ is itself a function of $\theta$; we write it as $\mathbb{E}_{p}[\cdot]$.
For any scalar function $f(p,\theta)$, the gradient $\nabla_\theta \mathbb{E}_p[\cdot]$ is:
\begin{align}
\mathbb{E}_{p}\big[f(p,\theta)\big]
=
\sum_{j\in\mathcal V}\pi_\theta(j\mid h_n)\,f(j,\theta) \qquad
\nabla_\theta \mathbb{E}_{p}\big[f(p,\theta)\big]
=
\mathbb{E}_{p}\Big[
f(p,\theta)\,\nabla_\theta \log \pi_\theta(p\mid h_n)
+
\nabla_\theta f(p,\theta)
\Big].
\label{eq:grad-of-expected-f}
\end{align}
In contrast, $\mathbb{E}_{p}[\nabla_\theta f(p,\theta)]$ alone is in general not equal to
$\nabla_\theta \mathbb{E}_{p}[f]$, as it has only one of the two terms in the full gradient.

\subsection{Sampled KL estimators (K1 / K2 / K3) and their expected values}
\label{app:sampled-estimators-and-their-expected-values}
Define the single-sample estimators
\begin{align}
(\mathrm{K1})_n \triangleq \ell(p),\qquad
(\mathrm{K2})_n \triangleq \frac{1}{2}\,\ell(p)^2,\qquad
(\mathrm{K3})_n \triangleq \ell(p)+w(p)-1.
\label{eq:K123-def}
\end{align}
Then
\begin{align}
\mathbb{E}_{p}\big[(\mathrm{K1})_n\big]
&= \mathbb{E}_{p}\big[\ell(p)\big]
= \mathrm{KL}_n(\theta),
\label{eq:K1-value}
\\
\mathbb{E}_{p}\big[w(p)\big]
&= \sum_{j\in\mathcal V}\pi_\theta(j\mid h_n)\frac{\pi_{\mathrm{ref}}(j\mid h_n)}{\pi_\theta(j\mid h_n)}
= \sum_{j\in\mathcal V}\pi_{\mathrm{ref}}(j\mid h_n)
= 1,
\label{eq:Ew-one}
\\
\mathbb{E}_{p}\big[(\mathrm{K3})_n\big]
&= \mathbb{E}_{p}\big[\ell(p)\big] + \mathbb{E}_{p}\big[w(p)\big] - 1
= \mathrm{KL}_n(\theta),
\label{eq:K3-value}
\end{align}
while in general $\mathbb{E}_{p}\big[(\mathrm{K2})_n\big]=\frac12\mathbb{E}_{p}[\ell(p)^2]\neq \mathbb{E}_{p}[\ell(p)]$,
so K2 has incorrect expected values.

\subsection{Gradients of K1 / K2 / K3 under autodiff $\mathbb{E}_p[\nabla_\theta(\cdot)]$}
\label{app:k1-k2-k3-gradients}
Here $\nabla_\theta(\cdot)$ denotes the autodiff gradient of the sampled scalar while treating the sampled token $p$ as fixed.

\paragraph{K1: correct value but wrong gradient.}
Since $(\mathrm{K1})_n=\ell(p)=\log \pi_\theta(p\mid h_n)-\log \pi_{\mathrm{ref}}(p\mid h_n)$ and $\pi_{\mathrm{ref}}$ is fixed,
$\nabla_\theta (\mathrm{K1})_n = \nabla_\theta \log \pi_\theta(p\mid h_n)$.
Averaging over $p\sim\pi_\theta(\cdot\mid h_n)$ gives
\begin{align*}
\mathbb{E}_{p}\big[\nabla_\theta (\mathrm{K1})_n\big]
=
\sum_{j\in\mathcal V}\pi_\theta(j\mid h_n)\nabla_\theta\log\pi_\theta(j\mid h_n)
=
\sum_{j\in\mathcal V}\nabla_\theta\pi_\theta(j\mid h_n)
=
\nabla_\theta 1
=
0.
\end{align*}

\paragraph{K2: correct reverse-KL gradient but wrong value.}
We have $(\mathrm{K2})_n=\frac12 \ell(p)^2$, so $\nabla_\theta (\mathrm{K2})_n = \ell(p)\,\nabla_\theta \ell(p) = \ell(p)\,\nabla_\theta \log \pi_\theta(p\mid h_n)$. Therefore,
\begin{align}
\mathbb{E}_{p}\big[\nabla_\theta (\mathrm{K2})_n\big]
=
\mathbb{E}_{p}\Big[\ell(p)\,\nabla_\theta \log \pi_\theta(p\mid h_n)\Big]
=
\nabla_\theta \mathrm{KL}_n(\theta).
\label{eq:K2-explicit-grad}
\end{align}

\paragraph{K3: yields the forward-KL gradient under $\mathbb{E}_p[\nabla_\theta(\cdot)]$.}
Since $(\mathrm{K3})_n=\ell(p)+w(p)-1$ and,
\begin{align*}
w(p)=\pi_{\mathrm{ref}}(p\mid h_n)/\pi_\theta(p\mid h_n) \quad
\nabla_\theta w(p)
=
-w(p)\,\nabla_\theta \log \pi_\theta(p\mid h_n), \\
\nabla_\theta (\mathrm{K3})_n
=
\nabla_\theta \ell(p) + \nabla_\theta w(p)
=
\big(1-w(p)\big)\,\nabla_\theta \log \pi_\theta(p\mid h_n).
\end{align*}
Averaging over $p\sim\pi_\theta(\cdot\mid h_n)$ gives
\begin{align}
\mathbb{E}_{p}\big[\nabla_\theta (\mathrm{K3})_n\big]
&=
\mathbb{E}_{p}\big[\nabla_\theta \log \pi_\theta(p\mid h_n)\big]
-
\mathbb{E}_{p}\big[w(p)\,\nabla_\theta \log \pi_\theta(p\mid h_n)\big] \nonumber\\
&=
-\sum_{j\in\mathcal V}\pi_{\mathrm{ref}}(j\mid h_n)\,\nabla_\theta\log\pi_\theta(j\mid h_n)
=
\nabla_\theta \mathrm{KL} \big(\pi_{\mathrm{ref}}(\cdot\mid h_n)\,\|\,\pi_\theta(\cdot\mid h_n)\big),
\label{eq:K3-forwardKL-grad}
\end{align}
which is the forward-KL gradient, not the reverse-KL gradient.

\subsection{The $\boldsymbol{r}$-trick restores $\nabla_\theta \mathbb{E}_p[\cdot]$ from $\mathbb{E}_p[\nabla_\theta(\cdot)]$.}
\label{app:the-r-trick}
Define $r(\cdot)$ as:
\begin{align}
r(p,\theta) \triangleq \frac{\pi_\theta(p\mid h_n)}{\mathrm{sg}\big(\pi_\theta(p\mid h_n)\big)}.
\label{eq:r-def}
\end{align}
In the forward pass $r(p,\theta)=1$, but in backpropagation the denominator is treated as constant, hence
$\nabla_\theta r(p,\theta)
=
r(p,\theta)\,\nabla_\theta \log \pi_\theta(p\mid h_n)
=
\nabla_\theta \log \pi_\theta(p\mid h_n)$.
For any scalar $f(p,\theta)$, differentiating $r(p,\theta)f(p,\theta)$ gives
\begin{align*}
\nabla_\theta\big(r f\big)
=
f\,\nabla_\theta r + r\,\nabla_\theta f
=
f(p,\theta)\,\nabla_\theta \log \pi_\theta(p\mid h_n)
+
\nabla_\theta f(p,\theta),
\end{align*}
where we used $r=1$ in the forward pass so $r\,\nabla_\theta f=\nabla_\theta f$.
Taking $\mathbb{E}_p[\cdot]$ on both sides and comparing with \eqref{eq:grad-of-expected-f} yields
\begin{align}
\mathbb{E}_{p}\big[\nabla_\theta\big(r(p,\theta) f(p,\theta)\big)\big]
=
\nabla_\theta \mathbb{E}_{p}\big[f(p,\theta)\big].
\label{eq:rf-unbiased-gradient}
\end{align}
This identity holds even when $f$ depends on $\theta$ (via the $\nabla_\theta f$ term).

\subsection{${\color{RoyalBlue}\textbf{K3}^{\boldsymbol{++}}}$ yields unbiased reverse-KL values and gradients}
\label{app:k3-plus-plus-gradients}

Define
\begin{align}
(\mathrm{K3}^{++})_n \triangleq r(p,\theta)\,(\mathrm{K3})_n
= r(p,\theta)\big(\ell(p)+w(p)-1\big).
\end{align}
\emph{Value:} since $r=1$ in the forward pass,
$\mathbb{E}_{p}\big[(\mathrm{K3}^{++})_n\big]
=
\mathbb{E}_{p}\big[(\mathrm{K3})_n\big]
=
\mathrm{KL}_n(\theta)$
by \eqref{eq:K3-value}.

\emph{Gradient:} by \eqref{eq:rf-unbiased-gradient} with $f=(\mathrm{K3})_n$,
\begin{align}
\mathbb{E}_{p}\big[\nabla_\theta(\mathrm{K3}^{++})_n\big]
=
\nabla_\theta \mathbb{E}_{p}\big[(\mathrm{K3})_n\big]
=
\nabla_\theta \mathrm{KL}_n(\theta).
\end{align}
Alternatively, we compute the $\mathrm{K3}^{++}$ gradients in a brute force manner by directly differentiating using the product rule, with $\nabla_\theta r=\nabla_\theta \ell=\nabla_\theta \log \pi_\theta(p\mid h_n)$, $\nabla_\theta w=-w\,\nabla_\theta \log \pi_\theta(p\mid h_n)$ and $r=1$ in the forward pass:
\begin{align}
\nabla_\theta (\mathrm{K3}^{++})_n
=
\big(\ell+w-1\big)\,\nabla_\theta r
+
r\big(\nabla_\theta \ell + \nabla_\theta w\big)
=
\ell(p)\,\nabla_\theta \log \pi_\theta(p\mid h_n),
\label{eq:k3pp-grad-simplified}
\end{align}
whose expectation over $p\sim\pi_\theta(\cdot\mid h_n)$ is $\nabla_\theta \mathrm{KL}_n(\theta)$.

\subsection{${\color{RoyalBlue}\textbf{K4}}$ yields unbiased reverse-KL values and gradients}
\label{app:k4-gradients}
Define
\begin{align}
(\mathrm{K4})_n \triangleq r(p,\theta)\cdot \mathrm{sg} \big((\mathrm{K1})_n\big)
=
r(p,\theta)\cdot \mathrm{sg} \big(\ell(p)\big).
\end{align}
\emph{Value:} since $r=1$ in the forward pass,
$\mathbb{E}_{p}\big[(\mathrm{K4})_n\big]
=
\mathbb{E}_{p}\big[\ell(p)\big]
=
\mathrm{KL}_n(\theta)$.

\emph{Gradient:} $\mathrm{sg}(\ell(p))$ is treated as constant in backprop, so differentiating and taking expectation over $p\sim\pi_\theta(\cdot\mid h_n)$ gives
\begin{align}
\mathbb{E}_{p}\big[\nabla_\theta(\mathrm{K4})_n\big]
=
\mathbb{E}_{p}\Big[\mathrm{sg} \big(\ell(p)\big)\,\nabla_\theta r(p,\theta)\Big]
=
\mathbb{E}_{p}\Big[\ell(p)\,\nabla_\theta \log \pi_\theta(p\mid h_n)\Big]
=
\nabla_\theta \mathrm{KL}_n(\theta).
\end{align}

\newpage

\section{Derivation of Sampled Forward KL Estimator}
\label{app:forward-kl-k5}

Similar to before, we focus on a fixed token position $n$ with prefix $h_n=(x,y_{0:n-1})$.
For brevity, write
$\pi_\theta(j) \triangleq \pi_\theta(j\mid h_n)$ and
$\pi_{\mathrm{ref}}(j) \triangleq \pi_{\mathrm{ref}}(j \mid h_n)$.
The sampled token $p \sim \pi_\theta(\cdot\mid h_n)$, the importance ratio $w(p)$, the stop-gradient ratio $r(p,\theta)$ \eqref{eq:r-def} and the expectation $\mathbb{E}_{p}[\cdot]$ over $p$ are also the same as before.

\subsection{Forward KL and its gradient}
\label{app:forward-kl-and-its-gradient}
The token-wise forward KL is
\begin{align*}
\mathrm{FKL}_n(\theta)
\triangleq
\mathrm{KL} \big(\pi_{\mathrm{ref}}(\cdot\mid h_n)\,\|\,\pi_\theta(\cdot\mid h_n)\big)
=
\sum_{j\in\mathcal V} \pi_{\mathrm{ref}}(j)\log\frac{\pi_{\mathrm{ref}}(j)}{\pi_\theta(j)}.
\end{align*}
Differentiating w.r.t.\ $\theta$ (note $\pi_{\mathrm{ref}}$ does not depend on $\theta$),
\begin{align}
\nabla_\theta \mathrm{FKL}_n(\theta)
&=
\sum_{j\in\mathcal V} \pi_{\mathrm{ref}}(j)\,\nabla_\theta\!\left(\log\pi_{\mathrm{ref}}(j)-\log\pi_\theta(j)\right)
=
-\sum_{j\in\mathcal V} \pi_{\mathrm{ref}}(j)\,\nabla_\theta \log\pi_\theta(j).
\label{eq:FKL-grad-exact}
\end{align}
We can rewrite \eqref{eq:FKL-grad-exact} as an expectation over $p\sim \pi_\theta(\cdot\mid h_n)$ using $w(p)$:
\begin{align}
-\sum_{j\in\mathcal V} \pi_{\mathrm{ref}}(j)\,\nabla_\theta \log\pi_\theta(j)
&=
-\sum_{j\in\mathcal V} \pi_{\theta}(j)\frac{\pi_{\mathrm{ref}}(j)}{\pi_{\theta}(j)}\,\nabla_\theta \log\pi_\theta(j)
=
-\mathbb{E}_{p}\!\Big[w(p)\,\nabla_\theta\log\pi_\theta(p)\Big].
\label{eq:FKL-grad-importance}
\end{align}
Finally, we use the score-mean identity $\mathbb{E}_{p}[\nabla_\theta\log\pi_\theta(p)]=0$ to obtain an equivalent form:
\begin{align}
\nabla_\theta \mathrm{FKL}_n(\theta)
=
-\mathbb{E}_{p}\!\Big[w(p)\,\nabla_\theta\log\pi_\theta(p)\Big]
=
\mathbb{E}_{p}\!\Big[(1-w(p))\,\nabla_\theta\log\pi_\theta(p)\Big].
\label{eq:FKL-grad-1-minus-w}
\end{align}

\subsection{The {\color{RoyalBlue}\textbf{K5}} estimator has unbiased value}
\label{app:the-k5-estimator-has-unbiased-value}
We propose the forward-KL estimator
\begin{align}
(\mathbf{K5})_n \triangleq \mathrm{sg}(w(p))\log w(p) + \log r(p,\theta).
\label{eq:K5-def}
\end{align}
For the KL value estimation, note that $r(p,\theta)=1$ in the forward pass, so $\log r(p,\theta)=0$ in value. Therefore
\begin{align}
\mathbb{E}_{p}\big[(\mathbf{K5})_n\big]
&=
\mathbb{E}_{p}\big[w(p)\log w(p)\big]
=
\sum_{j\in\mathcal V}\pi_{\mathrm{ref}}(j)\log\frac{\pi_{\mathrm{ref}}(j)}{\pi_\theta(j)}
=
\mathrm{FKL}_n(\theta).
\label{eq:K5-unbiased-value}
\end{align}

\subsection{The {\color{RoyalBlue}\textbf{K5}} estimator has unbiased gradient}
\label{app:k5-estimator-has-unbiased-gradient}

We next show that $\mathbb{E}_{p}[\nabla_\theta (\mathbf{K5})_n] = \nabla_\theta \mathrm{FKL}_n(\theta)$.

First, compute the gradient of each term in \eqref{eq:K5-def}. Since $\mathrm{sg}(w(p))$ is treated as a constant in backprop, $\log w(p)=\log\pi_{\mathrm{ref}}(p)-\log\pi_\theta(p)$ and $\pi_{\mathrm{ref}}$ is fixed,
\begin{align*}
\nabla_\theta\big(\mathrm{sg}(w(p))\log w(p)\big)
=
\mathrm{sg}(w(p))\,\nabla_\theta \log w(p)
=
-\mathrm{sg}(w(p))\,\nabla_\theta\log\pi_\theta(p) \qquad
\nabla_\theta \log r(p,\theta) = \nabla_\theta\log\pi_\theta(p).
\end{align*}
Combining them yields the per-sample gradient
$\nabla_\theta (\mathbf{K5})_n
=
(1-w(p))\,\nabla_\theta\log\pi_\theta(p)$.
Taking expectation over $p\sim\pi_\theta(\cdot\mid h_n)$ gives
\begin{align}
\mathbb{E}_{p}\big[\nabla_\theta (\mathbf{K5})_n\big]
=
\mathbb{E}_{p}\!\Big[(1-w(p))\,\nabla_\theta\log\pi_\theta(p)\Big]
=
\nabla_\theta \mathrm{FKL}_n(\theta),
\label{eq:K5-unbiased-grad}
\end{align}
where the last equality follows from \eqref{eq:FKL-grad-1-minus-w}.

\subsection{Variance reduction in {\color{RoyalBlue}\textbf{K5}} gradients}
\label{app:baseline-1-minus-w}

We start from the exact forward-KL gradient written in the importance-weighted form (Eq.~\eqref{eq:FKL-grad-importance}), which is the expected gradient of $\mathrm{sg}(w(p))\log w(p)$ alone.
A key observation is that under the same sampling distribution $p\sim\pi_\theta(\cdot\mid h_n)$,
the weight $w(p)$ has mean $1$ \eqref{eq:Ew-one}.
Therefore, we can subtract this constant baseline from $w(p)$ without changing the expectation of the gradient estimator, since $\mathbb{E}_{p}[\nabla_\theta \log \pi_\theta(p\mid h_n)]=0$, which gives \eqref{eq:FKL-grad-1-minus-w}; in $\mathbf{K5}$, this is done by the term $\log r(p,\theta)$.
Thus, replacing $-w(p)$ with $(1-w(p))$ centered around $0$ preserves an unbiased gradient while typically reducing variance, known as the use of a \textit{baseline} in policy gradient methods.

\newpage

\section{Bias-Variance Tradeoff in Top-k KL: An Illustrative Synthetic Setup}
\label{app:topk_kl_sim_exp}

This appendix specifies the synthetic setting used to produce Fig.~\ref{fig:topk_kl_sim}. The goal is to isolate the \emph{bias--variance} behavior of the \textbf{Top-k Reverse KL} estimator in a controlled regime that reflects how KL penalties are applied to LLM policies, while removing confounders such as long-horizon autoregressive dependencies.

\subsection{Token-wise i.i.d. simplification}
\label{app:topk_kl_sim_exp:iid_tokens}

Let $\pi_\theta(\cdot\mid h_n)$ be an LLM policy over the vocabulary $\mathcal V$ at token position $n$ given history $h_n=(x,y_{0:n-1})$, and let $\pi_{\text{ref}}(\cdot\mid h_n)$ be the reference policy.
The token-level reverse KL is:
\begin{align}
\KL \big(\pi_\theta(\cdot\mid h_n)\,\|\,\pi_{\text{ref}}(\cdot\mid h_n)\big)
=\sum_{j\in\mathcal V}\pi_\theta(j\mid h_n)\log\frac{\pi_\theta(j\mid h_n)}{\pi_{\text{ref}}(j\mid h_n)}.
\label{eq:app_revkl_token}
\end{align}
To focus purely on the estimator (rather than the dynamics that produce $h_n$), we adopt the following simplification:

\begin{quote}
\textbf{(i.i.d.\ token-position assumption)} Each token position contributes an i.i.d.\ draw from a common distribution over $(\pi,\pi_{\text{ref}})$, and we analyze a representative position whose ground-truth objective is the reverse KL in \eqref{eq:app_revkl_token}. In other words, we ignore autoregressive coupling across positions when constructing synthetic tasks.
\end{quote}

\subsection{Synthetic task and metric}
\label{app:topk_kl_sim_exp:task-and-metric}

Under this assumption, each token position is treated as a draw of a single-step categorical policy pair $(\pi,\pi_{\text{ref}})$ over a fixed vocabulary $\mathcal V$.
We fix $|\mathcal V|=32000$ and sample
\begin{align}
\pi_{\text{ref}} &= \mathrm{softmax}(\tilde z_{\text{ref}}), \qquad \tilde z_{\text{ref}}\sim\mathcal N(0,I_{|\mathcal V|}),\\
\pi &= \mathrm{softmax}(s\cdot \tilde z), \qquad \tilde z\sim\mathcal N(0,I_{|\mathcal V|}),\; s>0.
\end{align}

\paragraph{Controlling concentration via top-32 mass.}
We tune the logit scale $s$ so that the top-32 probability mass $M_{32}(\pi)\triangleq\sum_{j\in \mathrm{Topk}(\pi,32)} \pi(j)$ matches one of the target values $m\in\{0.2,\,0.5,\,0.8,\,0.9\}$, thereby probing a range of entropy regimes from diffuse to highly concentrated policies.

\paragraph{Estimators and error metric.}
We compare three per-token gradient estimators for the reverse-KL penalty:
(i) \emph{sampled-only} reverse KL (K4 on every batch token; fully Monte Carlo),
(ii) \emph{truncated-only} (the exact top-k head $(\mathrm{KL}_{\text{trun}})_{n} - \Delta_n$ without the tail term; biased), and
(iii) \emph{Top-k} reverse KL \eqref{eq:topk-reverse-kl} (the same head plus one tail token per batch token \eqref{eq:tail-correction}; unbiased).
We sweep $k\in\{4,8,\dots,256\}$ and $B\in\{1,2,4,\dots,32768\}$ over powers of two, where $B$ denotes the number of ``batch tokens'': sampled KL uses $B$ i.i.d.\ samples drawn from $\pi$, and Top-k KL draws one tail token for each of them. For each estimator, we measure the relative root-mean-squared error (RMSE) of the per-token gradient:
\begin{align}
\mathrm{RelRMSE}
\;\triangleq\;
\frac{\sqrt{\mathbb E\big[\|\hat g - g_{\text{true}}\|_2^2\big]}}{\|g_{\text{true}}\|_2},
\label{eq:app_relrmse_def_paper}
\end{align}
where $g_{\text{true}}$ is the exact gradient of \eqref{eq:app_revkl_token} with respect to the token logits at that position, and $\hat g$ is the corresponding estimator. The expectation in \eqref{eq:app_relrmse_def_paper} is approximated by Monte Carlo sampling and averaging over many independently generated synthetic tasks.

\paragraph{Critical sample size.}
For each $k$, the truncated-only estimator induces a bias floor (its error does not vanish with more samples since it omits the tail).
The Top-k estimator is unbiased but has sampling variance from the tail term.
We define the \emph{critical sample size} as the smallest $B$ beyond which Top-k KL achieves lower gradient error than Truncated KL, i.e., the intersection point between the Top-k RMSE curve and the truncated-only bias floor.

\clearpage
\newpage

\section{Top-k $\boldsymbol{f}$-Divergence Estimator}
\label{app:topk-fdiv}

We generalize the Top-$k$ ``truncated exact term + sampled tail'' construction from KL to a broad class of
$f$-divergences, while retaining (i) unbiased value estimates and (ii) unbiased gradients under the convention
$\mathbb{E}[\nabla_\theta(\cdot)]$ (i.e., the expectation of autodiff gradients of sampled scalars).

\paragraph{Setup and notation.}
Fix a token position $n$ and prefix $h_n=(x,y_{0:n-1})$.
Let the current policy at this position be $P_\theta(j) \triangleq \pi_\theta(j\mid h_n)$, $j\in\mathcal V$,
and let the reference policy be $Q(j) \triangleq \pi_{\mathrm{ref}}(j\mid h_n)$,
which is treated as independent of $\theta$.
We use an explicit expectation shorthand:
$\mathbb{E}_{p}\big[\cdot\big]
\triangleq
\mathbb{E}_{p\sim P_\theta(\cdot)}\big[\cdot\big]$.

\subsection{$f$-divergence and its gradients}
\label{app:f-divergences-and-its-gradients}

Let $f:\mathbb{R}_{+}\to\mathbb{R}$ be convex with $f(1)=0$. The (discrete) $f$-divergence is
\begin{align}
D_f(P_\theta\|Q)
\triangleq
\sum_{j\in\mathcal V} Q(j)\,
f\!\left(\frac{P_\theta(j)}{Q(j)}\right).
\label{eq:fdiv-def}
\end{align}
Define the likelihood ratio and its inverse
\begin{align}
t(j) \triangleq \frac{P_\theta(j)}{Q(j)}, \qquad
w(j) \triangleq \frac{Q(j)}{P_\theta(j)} = \frac{1}{t(j)}.
\label{eq:t-w-def}
\end{align}
Eq.~\eqref{eq:fdiv-def} can be written as an expectation under $Q$:
\begin{align}
D_f(P_\theta\|Q)
=
\mathbb{E}_{j\sim Q}\big[f(t(j))\big].
\label{eq:fdiv-Q-exp}
\end{align}
It also admits an equivalent importance-weighted expectation under $P_\theta$.
Define
\begin{align}
g_f(w) \triangleq w\, f(1/w).
\label{eq:gf-def}
\end{align}
Then
\begin{align}
D_f(P_\theta\|Q)
=
\sum_{j\in\mathcal V} P_\theta(j)\,\frac{Q(j)}{P_\theta(j)}\,
f\!\left(\frac{P_\theta(j)}{Q(j)}\right)
=
\sum_{j\in\mathcal V} P_\theta(j)\, g_f\!\left(w(j)\right)
=
\mathbb{E}_{p}\big[g_f(w(p))\big],
\label{eq:fdiv-P-exp}
\end{align}
where $p\sim P_\theta(\cdot)$ and $w(p)=Q(p)/P_\theta(p)$.

\subsubsection{Exact gradient w.r.t.\ $\theta$ (reference policy fixed)}
Since $Q$ is fixed, differentiating \eqref{eq:fdiv-def} yields
\begin{align}
\nabla_\theta D_f(P_\theta\|Q)
&=
\sum_{j\in\mathcal V} Q(j)\, f'(t(j))\,\nabla_\theta\!\left(\frac{P_\theta(j)}{Q(j)}\right)
=
\sum_{j\in\mathcal V} f'(t(j))\,\nabla_\theta P_\theta(j)
\nonumber\\
&=
\sum_{j\in\mathcal V} P_\theta(j)\, f'(t(j))\,\nabla_\theta \log P_\theta(j)
=
\mathbb{E}_{p}\Big[f'(t(p))\,\nabla_\theta \log P_\theta(p)\Big].
\label{eq:fdiv-grad-score}
\end{align}
Eq.~\eqref{eq:fdiv-grad-score} is the policy-gradient-shaped form used throughout this appendix.

\subsection{A generic, unbiased, sampled $f$-divergence estimator}
\label{app:a-general-sampled-f-divergence-estimator}

We define $r(p) \triangleq P_\theta(p) / \mathrm{sg}\big(P_\theta(p)\big)$ as before \eqref{eq:r-trick}.
Then we have our sampled $f$-divergence estimator, $\mathbf{K}f^{++}$:
\begin{align}
(\mathbf{K}f^{++})_n \triangleq r(p)\,g_f(w(p)),
\qquad p\sim P_\theta(\cdot).
\label{eq:Kfpp-def}
\end{align}
Which in effect becomes:
\begin{align}
    \boxed{(\mathbf{K}f^{++})_n = \mathrm{sg} \bigg(\frac{ Q(p) }{ P_{\theta}(p) } \bigg) f \bigg( \frac{ P_{\theta}(p) }{ Q(p) } \bigg) = \mathrm{sg}(w)\, f(1/w) }
    \label{eq:f-div-estimator-canonical}
\end{align}

\paragraph{Unbiased value.}
Since $r=1$ in the forward pass, from \eqref{eq:fdiv-P-exp} we have
\begin{align}
\mathbb{E}_{p}\big[(\mathbf{K}f^{++})_n\big]
=
\mathbb{E}_{p}\big[g_f(w(p))\big]
=
D_f(P_\theta\|Q).
\label{eq:Kfpp-unbiased-value}
\end{align}

\paragraph{Unbiased gradient under $\mathbb{E}[\nabla_\theta(\cdot)]$.}
Differentiate \eqref{eq:Kfpp-def} using product rule:
\begin{align}
\nabla_\theta(\mathbf{K}f^{++})_n
=
g_f(w(p))\,\nabla_\theta r(p) + r(p)\,\nabla_\theta g_f(w(p))
=
g_f(w(p))\,\nabla_\theta \log P_\theta(p)
+
\nabla_\theta g_f(w(p))
\label{eq:Kfpp-grad}
\end{align}
Taking expectation over $p\sim P_\theta(\cdot)$ gives
\begin{align}
\mathbb{E}_{p}\big[\nabla_\theta(\mathbf{K}f^{++})_n\big]
=
\nabla_\theta \mathbb{E}_{p}\big[g_f(w(p))\big]
=
\nabla_\theta D_f(P_\theta\|Q),
\label{eq:Kfpp-unbiased-grad}
\end{align}
where the last equality follows from \eqref{eq:fdiv-P-exp}.
Thus $(\mathbf{K}f^{++})_n$ provides both unbiased values and unbiased gradients for any $f$-divergence.

\begin{table}[t]
\centering
\small
\setlength{\tabcolsep}{8pt}
\renewcommand{\arraystretch}{1.3}
\caption{
\centering
Sampled $f$-divergence estimators $\mathbf{K}f^{++}$ \eqref{eq:f-div-estimator-canonical}, the tail term of the Top-k $f$-divergence estimator. \\
\textbf{Let} $P_\theta(j)=\pi_\theta(j| h_n),\ Q(j)=\pi_{\mathrm{ref}}(j| h_n),\ t_j=P_\theta(j)/Q(j)$. \\ {Sample} $p\sim P_\theta(\cdot)$.\quad  We compute $r=P_{\theta}(p) / \mathrm{sg}( P_{\theta}(p))$, \quad $w=Q(p)/P_\theta(p)$. \quad Note that $\mathrm{sg}(w)/w=r$.
}
\vspace{5pt}
\begin{tabularx}{0.9\linewidth}{lll}
\toprule
& $\boldsymbol{f(t)}$ (so $D_f=\sum_j Q(j)\,f(t_j)$)
& \textbf{Sampled $f$-divergence estimator $=\mathrm{sg}(w)\, f(1/w)$} \\
\midrule

Forward KL
& $-\log t$
& $\mathrm{sg}(w) \cdot \log w$ \\

Reverse KL
& $t\log t$
& $ r\cdot ( -\log w )$ \\

Pearson $\chi^2$
& $(t-1)^2$
& $r \cdot \frac{(1-w)^2}{w}$ \\

Neyman $\chi^2$
& $\dfrac{(t-1)^2}{t}$
& $r \cdot (1-w)^2$ \\

Squared Hellinger
& $(\sqrt{t}-1)^2$
& $r \cdot (1-\sqrt{w})^2$ \\

Jensen--Shannon (JS)
& $\displaystyle {\textstyle\frac{1}{2}} \left(t\log t-(t+1)\log\!\Big(\frac{t+1}{2}\Big) \right)$
& $\displaystyle r \cdot {\textstyle\frac{1}{2}} \Big[-\log w -(1+w)\log\!\Big(\frac{1+w}{2w}\Big)\Big]$ \\

Total variation (TV)
& $\tfrac12|t-1|$
& $r \cdot \tfrac12|1-w|$ \\

$\alpha$-divergence
& $\displaystyle \frac{4}{1-\alpha^{2}}\Big(1-t^{\frac{1+\alpha}{2}}\Big),\ \alpha\neq \pm 1$
& $\displaystyle r\cdot
\frac{4}{1-\alpha^{2}}\Big(w-w^{\frac{1-\alpha}{2}}\Big)$ \\

\bottomrule
\end{tabularx}
\label{tab:fdiv-compact}
\end{table}

\subsection{Top-k $f$-divergence estimator: A generalization of Top-k KL estimator}
\label{app:top-k-f-divergence}

Let $q_n\subseteq\mathcal V$ be any index set of size $k$ (e.g., $q_n=\mathrm{Topk}(P_\theta)$ or $q_n=\mathrm{Topk}(Q)$).
Being an unbiased estimator does not require $q_n$ to be top-$k$; top-$k$ is a variance/compute heuristic.
\paragraph{Truncated exact term.}
Define the truncated $f$-divergence (computed exactly given logits on $q_n$):
\begin{align}
\boxed{(D_{f,\mathrm{trun}})_n
\triangleq
\sum_{j\in q_n} Q(j)\,f\!\left(\frac{P_\theta(j)}{Q(j)}\right)}
\label{eq:fdiv-trun}
\end{align}
\paragraph{Sampled tail term.}
As in \eqref{eq:tail-correction}, draw one tail token $z_n\sim\mu_n$, where $\mu_n$ is any distribution over $\mathcal V$ with $\mu_n(j)>0$ iff $j\notin q_n$, and let $s_n=\mathrm{sg}(P_\theta(z_n)/\mu_n(z_n))$ be its importance weight. Define the tail estimator
\begin{align}
\boxed{(\widehat{D}_{f,\mathrm{tail}})_n
\triangleq
s_n\,(\mathbf{K}f^{++})_n\big|_{p=z_n}}
=
s_n\,r(z_n)\,g_f\!\left(\frac{Q(z_n)}{P_\theta(z_n)}\right)
\label{eq:fdiv-tail}
\end{align}

\paragraph{Top-k $f$-divergence estimator.}
We define the Top-$k$ $f$-divergence estimator
\begin{align}
\boxed{(\mathrm{Topk}\text{-}D_f)_n
\triangleq
(D_{f,\mathrm{trun}})_n
+
(\widehat{D}_{f,\mathrm{tail}})_n}
\label{eq:topk-fdiv-final}
\end{align}

\paragraph{Unbiased value.}
The weight $s_n$ carries no gradient and $r=1$ in the forward pass, so by the definition of $g_f$,
\begin{align}
\mathbb{E}_{z_n}\big[(\widehat{D}_{f,\mathrm{tail}})_n\big]
&=
\sum_{j\notin q_n} \mu_n(j)\,\frac{P_\theta(j)}{\mu_n(j)}\, g_f\!\left(\frac{Q(j)}{P_\theta(j)}\right)
=
\sum_{j\notin q_n} Q(j)\, f\!\left(\frac{P_\theta(j)}{Q(j)}\right),
\end{align}
so
\begin{align}
\mathbb{E}_{z_n}\big[(\mathrm{Topk}\text{-}D_f)_n\big]
=
(D_{f,\mathrm{trun}})_n
+
\mathbb{E}_{z_n}\big[(\widehat{D}_{f,\mathrm{tail}})_n\big]
=
D_f(P_\theta\|Q).
\label{eq:topk-fdiv-unbiased-value}
\end{align}

\paragraph{Unbiased gradient.}
Since $(D_{f,\mathrm{trun}})_n$ is computed exactly on $q_n$, it has the correct gradient contribution from indices in $q_n$. For the tail term, the gradient of \eqref{eq:f-div-estimator-canonical} at a token $j$ is
\begin{align}
\nabla_\theta (\mathbf{K}f^{++})_n\big|_{p=j}
=
\mathrm{sg}(w(j))\,f'(t(j))\,\nabla_\theta t(j)
=
f'(t(j))\,\nabla_\theta \log P_\theta(j),
\label{eq:Kfpp-token-grad}
\end{align}
which is the $j$-th term of the exact gradient \eqref{eq:fdiv-grad-score}, with no additional term. Hence
\begin{align}
\mathbb{E}_{z_n}\big[\nabla_\theta (\widehat{D}_{f,\mathrm{tail}})_n\big]
=
\sum_{j\notin q_n} P_\theta(j)\,f'(t(j))\,\nabla_\theta \log P_\theta(j)
=
\nabla_\theta \sum_{j\notin q_n} Q(j)\, f\!\left(\frac{P_\theta(j)}{Q(j)}\right),
\end{align}
and therefore
\begin{align}
\mathbb{E}_{z_n}\big[\nabla_\theta (\mathrm{Topk}\text{-}D_f)_n\big]
=
\nabla_\theta D_f(P_\theta\|Q).
\label{eq:topk-fdiv-unbiased-grad}
\end{align}
The same holds with the rollout token in place of the tail token, i.e., with the tail term $\mathbbm{1}(p\notin q_n)\,(\mathbf{K}f^{++})_n$ for $p\sim P_\theta(\cdot)$, at a higher variance when $\mu_n(j)\propto P_\theta(j)\,\mathbbm{1}(j\notin q_n)$ (\S\ref{app:topk-tail-variance}).

\paragraph{Adding a baseline.}
For any constant $c$, adding $c\,\log r(z_n)$ to the tail term and $c\,\Delta_n$ \eqref{eq:delta-correction} to the truncated term keeps both the value and the gradient unbiased:
\begin{align}
\boxed{(\mathrm{Topk}\text{-}D_f^{(c)})_n
\triangleq
(D_{f,\mathrm{trun}})_n
+ c\,\Delta_n
+ s_n\,\big[(\mathbf{K}f^{++})_n + c\,\log r\big]_{p=z_n}}
\label{eq:topk-fdiv-baseline}
\end{align}
Both $c$-terms are zero in value, and their expected gradients cancel: $c\sum_{j\in q_n}\nabla_{\theta} P_\theta(j) + c\sum_{j\notin q_n}\nabla_{\theta} P_\theta(j) = 0$. A baseline on the tail term alone, without $c\,\Delta_n$, biases the gradient by $-c\sum_{j\in q_n}\nabla_{\theta} P_\theta(j)$. The estimator \eqref{eq:topk-fdiv-final} is the case $c=0$. Top-k reverse KL \eqref{eq:topk-reverse-kl} is the case $f(t)=t\log t$ with $c=-1$, and Top-k forward KL \eqref{eq:topk-forward-kl} is the case $f(t)=-\log t$ with $c=+1$. In both, $c=-f'(1)$, for which the gradient of the tail term at a token, $(f'(t)-f'(1))\,\nabla_{\theta}\log P_\theta$, vanishes where $P_\theta=Q$.

\paragraph{A note on variance reduction in Forward and Reverse KL gradient estimators.}
\begin{itemize}
    \item The Forward KL estimator derived from the standard $f$-divergence formula in Eq \eqref{eq:f-div-estimator-canonical} is $\mathrm{sg}(w) \log w$, which differs from our {\color{RoyalBlue}\textbf{K5}} estimator $\mathrm{sg}(w)\log w + \log r$. This is explained in \S\ref{app:baseline-1-minus-w}: the additional term $\log r$ is added as a \textit{baseline} to reduce gradient variance, where the baseline is $\E[\mathrm{sg}(w)]=1$.
    \item The Reverse KL estimator derived from Eq \eqref{eq:f-div-estimator-canonical} is $r\cdot (-\log w)$; this also differs from our {\color{RoyalBlue}\textbf{K4}} estimator $r\cdot \mathrm{sg}(-\log w)$. The reason is similar: the stop gradient on $(-\log w)$ removes the term $-\nabla_{\theta} \log w = \nabla_{\theta} \log P_{\theta}(p)$ from the gradient, so that {\color{RoyalBlue}\textbf{K4}} equals $r\cdot(-\log w) - \log r$ in both value and gradient. Since $\E_{p}[ \nabla_{\theta} \log P_{\theta}(p) ]=0$, this keeps the gradients unbiased while typically reducing variance.
\end{itemize}

\clearpage
\newpage

\section{Optimal Policies under Forward vs Reverse KL Regularization}
\label{app:token-level-rl}

We now study the optimal policies under different types of regularizers and see if the optimal policies are interchangeable. 

The RL objectives under reverse and forward KL regularization are:
\begin{align}
    J &= \mathbb{E}_{y \sim \pi_{\theta}(\cdot \mid x)}[R(x, y)] - \beta\, \mathrm{KL}(\pi_{\theta}(\cdot \mid x) \| \pi_{\text{ref}}(\cdot \mid x) ) \label{eq:reverse-kl-obj} \\
    J_{\text{fwd}} &= \mathbb{E}_{y\sim \pi_{\theta}(\cdot \mid x)}[R(x, y)] - \beta \mathrm{KL}(\pi_{\text{ref}}(\cdot \mid x) \| \pi_{\theta}(\cdot \mid x) ) \label{eq:forward-kl-obj}
\end{align}

\paragraph{Optimal Policy under Reverse KL}
It is well known that the optimal policy for objective \eqref{eq:reverse-kl-obj} is \cite{reps, equivalence-between-pg-and-sq, mpo}:
\begin{align}
    \boxed{\pi^{*}(y|x) = \pi_{\text{ref}}(y|x) \, \dfrac{\exp(R(x,y) / \beta)}{Z(x)}} \qquad Z(x) = {\sum}_{y} \pi_{\text{ref}}(y|x) \exp(R(x,y)/\beta)
\end{align}

\paragraph{Optimal Policy under Forward KL}
We derive the optimal policy for objective \eqref{eq:forward-kl-obj}. Expanding the forward KL:
\begin{align}
    J_{\text{fwd}} &= {\sum}_{y} \pi_\theta(y|x) R(x, y) + \beta\; {\sum}_{y} \pi_{\text{ref}}(y|x) \log \pi_\theta(y|x) + \text{const}
\end{align}
where the constant absorbs terms independent of $\pi_\theta$. This objective decomposes into an RL term (first sum) and a behavioral cloning (BC) term (second sum). Forming the Lagrangian:
\begin{align*}
    \mathcal{L} &= {\sum}_{y} \pi_\theta(y|x) R(x, y) + \beta \; {\sum}_{y} \pi_{\text{ref}}(y|x) \log \pi_\theta(y|x) - \lambda \left( {\sum}_{y} \pi_\theta(y|x) - 1 \right) \\
    \frac{\partial \mathcal{L}}{\partial \pi_\theta(y| x)} &= R(x,y) + \beta \; \frac{\pi_{\text{ref}}(y|x)}{\pi_\theta(y|x)} - \lambda = 0 
    \\
    \pi_\theta(y\mid x) &= \frac{\beta \; \pi_{\text{ref}}(y\mid x)}{\lambda - R(x,y)}
\end{align*}
For this to be a valid probability distribution, we need $\lambda > \max_{y} R(x,y)$. Denote $Z^*(x) := \lambda$, we require $Z^*(x)$ to satisfy:
\begin{align}
    {\sum}_{y} \frac{\pi_{\text{ref}}(y \mid x)}{Z^*(x) - R(x, y)} = \frac{1}{\beta} \label{eq:zstar-constraint}
\end{align}
Equation \eqref{eq:zstar-constraint} implicitly defines $Z^*(x)$ as a function of the reward distribution as well. The optimal policy is:
\begin{align}
    \boxed{\pi_{\text{fwd}}^{*}(y\mid x) = \pi_{\text{ref}}(y \mid x) \, \dfrac{\beta}{Z^{*}(x) - R(x,y)}}
\end{align}
A similar result was obtained in \cite{mcts-as-rpo} under a different context.

\subsection{Reward transformation for equivalent optima between forward KL and reverse KL}
\label{app:reward-transformation}

We now derive the reward transformation that makes the two objectives yield the same optimal policy.
\begin{align}
    \pi^*(y\mid x) = \pi_{\text{fwd}}^*(y \mid x) \qquad \frac{\pi_{\text{ref}}(y|x) \exp(\tilde{R}/\beta)}{Z(x)} = \frac{\beta \; \pi_{\text{ref}}(y \mid x)}{Z^*(x) - R(x, y)}
\end{align}
Therefore, the optimal policy under {reward function $R$ and forward KL regularizer} is equivalent to the counterpart under {reward function $\tilde{R}$ and reverse KL regularizer}:
\begin{align}
    \text{Forward KL (with $R$)}\rightarrow\text{Reverse KL (with $\tilde{R}$)} \qquad  \boxed{\tilde{R} = -\beta \log (Z^{*} - R) + \beta \log (\beta Z)}
\end{align}

\begin{proposition}
The function $\tilde{R} = -\beta \log(Z^* - R) + \text{const}$ is strictly {monotonically increasing} and {strictly convex} in $R$.
\end{proposition}
\begin{proof}
Taking the derivative with respect to $R$ (treating $Z^*$ as constant):
$\frac{\partial \tilde{R}}{\partial R} = \frac{\beta}{Z^* - R} > 0$.
The last step was due to $Z^* > \max_{y} R(x,y)$ by construction. Thus $\tilde{R}$ is strictly monotonically increasing in $R$.
Taking the second derivative:
$\frac{\partial^2 \tilde{R}}{\partial R^2} = \frac{\beta}{(Z^* - R)^2} > 0$.
Thus $\tilde{R}$ is strictly convex in $R$ (amplifying high rewards).
\end{proof}

\paragraph{Inverse transformation:} Given a reward $R$ used in the reverse KL objective $J$, what reward $R_{\text{fwd}}$ should be used in the forward KL objective $J_{\text{fwd}}$ to obtain the same optimal policy?

Setting $\pi^*(y|x) = \pi_{\text{fwd}}^*(y|x)$, we have ${\exp(R/\beta)}/{Z} = {\beta}/{(Z^* - R_{\text{fwd}})}$, which simplifies to:
\begin{align}
    \text{Reverse KL (with $R$)}\rightarrow\text{Forward KL (with $R_{\text{fwd}}$)} \qquad \boxed{R_{\text{fwd}} = Z^* - \beta Z \exp(-R/\beta)}
\end{align}

\begin{proposition}
The function $R_{\text{fwd}} = Z^* - \beta Z \exp(-R/\beta)$ is strictly monotonically increasing and strictly concave in $R$.
\end{proposition}

\begin{proof}
Taking first-order derivative
$\frac{\partial R_{\text{fwd}}}{\partial R} = Z \exp(-R/\beta) > 0$; therefore, $R_{\text{fwd}}$ is strictly monotonically increasing (preserving reward ordering).
The second-order derivative is
$\frac{\partial^2 R_{\text{fwd}}}{\partial R^2} = -\frac{Z}{\beta} \exp(-R/\beta) < 0$; therefore, $R_{\text{fwd}}$ is strictly concave (compressing high rewards).
\end{proof}

\subsection{Reward transformation for equivalent optima across $f$-divergences}
\label{app:reward-transform-for-f-divergence}

We now generalize beyond KL divergence to arbitrary $f$-divergence regularization. Consider the objective:
\begin{align}
    J_f = \mathbb{E}_{y \sim \pi_\theta(\cdot|x)}[R(x,y)] - \beta \, D_f(\pi_\theta(\cdot|x) \| \pi_{\text{ref}}(\cdot|x))
\end{align}
where $D_f$ is the $f$-divergence \eqref{eq:fdiv-def}. Standard KKT conditions yield the optimal policy \cite{f-dpo}:
\begin{align}
    \boxed{\pi_f^*(y\mid x) = \pi_{\text{ref}}(y\mid x) \cdot (f')^{-1}\left(\frac{R(x,y) - \lambda(x)}{\beta}\right)}
\end{align}
where $\lambda(x)$ is determined by the normalization constraint ${\sum}_y \pi_f^*(y\mid x) = 1$.

Given two $f$-divergences with $f_1$ and $f_2$, the optimal policies $\pi_{f_1}^*$ and $\pi_{f_2}^*$ coincide when $(f_1')^{-1}\big(\frac{R_1 - \lambda_1}{\beta}\big) = (f_2')^{-1}\big(\frac{R_2 - \lambda_2}{\beta}\big)$.
This yields the general reward transformation:
\begin{align}
    \boxed{R_2 = \beta \cdot f_2'\left( (f_1')^{-1}\left(\frac{R_1 - \lambda_1}{\beta}\right) \right) + \lambda_2}
\end{align}
The curvature of this transformation is determined by the composition $(f_2') \circ (f_1')^{-1}$.

\begin{table}[h]
\vspace{15pt}
\centering
\small
\renewcommand{\arraystretch}{1.2}
\setlength{\tabcolsep}{8pt}
\caption{
\centering
\textbf{Optimal Policy form} under each $f$-divergence regularization.
}
\begin{tabular}{lcccc}
\toprule
\textbf{Divergence} & $f(t)$ & $f'(t)$ & $(f')^{-1}(s)$ & {Optimal Policy Form} \\
\midrule
Reverse KL
& $t\log t - t + 1$
& $\log t$
& $e^{s}$
& $\pi_{\text{ref}}\cdot \exp\!\big((R-\lambda)/\beta\big)$ \\[7pt]

Forward KL
& $-\log t + t - 1$
& $1-\frac{1}{t}$
& $\frac{1}{1-s}$
& $\pi_{\text{ref}}\cdot \frac{1}{1-(R-\lambda)/\beta}$ \\[7pt]

Jensen--Shannon
& $\frac12\!\left[t\log t-(t+1)\log\!\Big(\frac{t+1}{2}\Big)\right]$
& $\frac12\log\!\Big(\frac{2t}{t+1}\Big)$
& $\frac{1}{2e^{-2s}-1}$
& $\pi_{\text{ref}}\cdot \frac{\exp\!\big(2(R-\lambda)/\beta\big)}{2-\exp\!\big(2(R-\lambda)/\beta\big)}$ \\[7pt]

Squared Hellinger
& $(\sqrt{t}-1)^2=t-2\sqrt{t}+1$
& $1-\frac{1}{\sqrt{t}}$
& $\frac{1}{(1-s)^2}$
& $\pi_{\text{ref}}\cdot \frac{1}{\big(1-(R-\lambda)/\beta\big)^2}$ \\[7pt]

Pearson $\chi^2$
& $(t-1)^2$
& $2(t-1)$
& $1+\frac{s}{2}$
& $\pi_{\text{ref}}\cdot \Big(1+\frac{R-\lambda}{2\beta}\Big)$ \\[7pt]

Neyman $\chi^2$
& $\frac{(t-1)^2}{t}=t-2+\frac{1}{t}$
& $1-\frac{1}{t^2}$
& $\frac{1}{\sqrt{1-s}}$
& $\pi_{\text{ref}}\cdot \frac{1}{\sqrt{1-(R-\lambda)/\beta}}$ \\[7pt]

$\alpha$-divergence
& $\frac{4}{1-\alpha^{2}}\Big(1-t^{\frac{1+\alpha}{2}}\Big)$
& $-\frac{2}{1-\alpha}\,t^{\frac{\alpha-1}{2}}$
& $\left(-\frac{1-\alpha}{2}\,s\right)^{\frac{2}{\alpha-1}}$
& $\pi_{\text{ref}}\cdot \left(-\frac{1-\alpha}{2}\cdot\frac{R-\lambda}{\beta}\right)^{\frac{2}{\alpha-1}}$ \\
\bottomrule
\end{tabular}
\end{table}

\begin{table}[h]
\vspace{10pt}
\centering
\small
\renewcommand{\arraystretch}{1.25}
\setlength{\tabcolsep}{7pt}
\caption{
\centering
\textbf{Reward transformations that produce the \emph{same optimal policy} as each $f$-divergence under Reverse KL.}\\
Let $\lambda$ and $\widetilde{\lambda}$ be normalization constants enforcing ${\sum}_y \pi^{*} (y\mid x)=1$ for the $f$-divergence and Reverse KL objectives, respectively.\\
Then the equivalence is achieved by $\widetilde R(x,y)=\beta\log\!\big((f')^{-1}((R-\lambda)/\beta)\big)+\widetilde\lambda$, for each $f(\cdot)$ in $f$-divergence.\\
}
\vspace{3pt}
\begin{tabular}{lccc}
\toprule
\textbf{Divergence $f$ (target)} & $(f')^{-1}(s)$ & \textbf{Reverse-KL reward transform $\widetilde R(x,y)$} & {Ranges of} $\lambda(x)$ \\
\midrule

Forward KL
& $\dfrac{1}{1-s}$
& $-\beta\log\!\Big(1-\dfrac{R-\lambda}{\beta}\Big)+\widetilde\lambda$
& $\lambda > R - \beta$
\\[6pt]

Jensen--Shannon
& $\dfrac{1}{2e^{-2s}-1}$
& $-\beta\log\!\Big(2\exp\!\big(-2\dfrac{R-\lambda}{\beta}\big)-1\Big)+\widetilde\lambda$
& $\lambda > R-\frac{\beta}{2}\log 2$
\\[6pt]

Squared Hellinger
& $\dfrac{1}{(1-s)^2}$
& $-2\beta\log\!\Big(1-\dfrac{R-\lambda}{\beta}\Big)+\widetilde\lambda$
& $\lambda > R-\beta$
\\[6pt]

Pearson $\chi^2$
& $1+\dfrac{s}{2}$
& $\beta\log\!\Big(1+\dfrac{R-\lambda}{2\beta}\Big)+\widetilde\lambda$
& $\lambda<R+2\beta$
\\[6pt]

Neyman $\chi^2$
& $\dfrac{1}{\sqrt{1-s}}$
& $-\dfrac{\beta}{2}\log\!\Big(1-\dfrac{R-\lambda}{\beta}\Big)+\widetilde\lambda$
& $\lambda > R-\beta$ 
\\[6pt]

$\alpha$-divergence
& $\left(-\dfrac{1-\alpha}{2}\,s\right)^{\frac{2}{\alpha-1}}$
& $\dfrac{2\beta}{\alpha-1}\log\!\Big(-\dfrac{1-\alpha}{2}\cdot\dfrac{R-\lambda}{\beta}\Big)+\widetilde\lambda$
& $-\,(1-\alpha)(R-\lambda)>0$
\\

\bottomrule
\end{tabular}
\label{tab:rkl_reward_transforms_from_f}
\end{table}

\clearpage
\newpage

\section{Policy Gradient Methods can be Derived from Sampled KL Estimators}
\label{app:new-family-of-pg}

If RL optimization is simply written as closing the reverse KL between current policy $\pi_{\theta}$ and the reference policy distribution $\pi_{\text{ref}}$ transformed by the return $R$ \cite{rwr, equivalence-between-pg-and-sq, sac}, we have:
\begin{align}
    \min_{\theta} \mathrm{KL}\bigg(\pi_{\theta}(\cdot | x) \Big|\Big| \dfrac{1}{Z(x)} \pi_{\text{ref}}(\cdot | x)\exp \Big( R(x,\cdot) \Big) \bigg)
    \label{eq:f-div-policy-gradient}
\end{align}
to make the connection with sampled KL estimators clear, we define the ratio between the optimal policy $\pi^{*}$ and current policy $\pi_{\theta}$ as $w$, same as before:
\begin{align}
    w = \dfrac{ \pi_{\text{ref}}(y|x) \, \exp R(x,y) }{ \pi_{\theta}(y|x)\, Z(x) }
\end{align}
where the partition function $Z(x)$ is approximated via a group of $N$ trajectories per prompt $y^{(i)} \sim \pi_{\theta}(\cdot \mid x)$ for $i=1,\dots, N$:
\begin{align}
    \log Z(x) \triangleq \log \E_{\pi_{\text{ref}}(y | x)} [\exp R(x,y)] &\approx \log \dfrac{1}{N} \sum_{i=1}^{N} \mathrm{sg}\left( { \frac{ \pi_{\text{ref}}(y^{(i)}\mid x) }{ \pi_{\theta}(y^{(i)} \mid x) }} \right) \exp \Big( R(x,y^{(i)}) \Big) \\
    &\geq \dfrac{1}{N} \sum_{i=1}^{N} \left( R(x,y^{(i)}) - \mathrm{sg}\left(\log { \frac{ \pi_{\theta}(y^{(i)} \mid x) }{ \pi_{\text{ref}}(y^{(i)} \mid x) }} \right) \right)
\end{align}
where the last step uses Jensen's inequality. We split it into two terms, value function ($V(x)$) and aggregate KL ($\widehat{\mathrm{KL}}$):
\begin{align}
    V(x) \triangleq \dfrac{1}{N} \sum_{i=1}^{N} R(x,y^{(i)}) \qquad \widehat{\mathrm{KL}} \triangleq \dfrac{1}{N} \sum_{i=1}^{N} \mathrm{sg}\left(\log { \frac{ \pi_{\theta}(y^{(i)} \mid x) }{ \pi_{\text{ref}}(y^{(i)} \mid x) }} \right)
\end{align}
As a side note, \textit{to sharpen or flatten the target distribution, we simply need to multiply all rewards by $1/\beta$}; as a result, we omit $\beta$ in our derivation and let it be absorbed by $R(x,y)$.

Next, we show that \textbf{every sampled KL estimator in Tab \ref{tab:kl-estimators} can be turned into a policy gradient algorithm}.

\subsection{K1 estimator leads to on-policy PPO / GRPO loss}
\label{app:k1-leads-to-ppo-grpo}
We start with $\mathbf{K1}=-\log w$. Its autodiff gradient has zero mean (\S\ref{app:k1-k2-k3-gradients}), so we use its score-function gradient, $\nabla_{\theta} \log \pi_{\theta}(y|x) \cdot \mathrm{sg}(-\log w)$:
\begin{align*}
    \text{K1:} \quad \min_{\theta}-\log w \qquad  \nabla_{\theta}L_{1} &= -\nabla_{\theta} \log { \pi_{\theta}(y|x)} \left( R(x,y) - \log Z(x) - \mathrm{sg}\left( \log \dfrac{ \pi_{\theta}(y \mid x) }{ \pi_{\text{ref}}(y\mid x) } \right) \right) \\
    &\approx -\nabla_{\theta} \log { \pi_{\theta}(y|x)} \left( R(x,y) - V(x) + \widehat{\mathrm{KL}} - \mathrm{sg}\left( \log \dfrac{ \pi_{\theta}(y \mid x) }{ \pi_{\text{ref}}(y\mid x) } \right) \right)
\end{align*}
Summing over all $N$ rollouts, we get
\begin{align*}
    \underbrace{-\sum_{j=1}^{N} \nabla_{\theta} \log { \pi_{\theta}(y^{(j)}\mid x)} \bigg( R(x,y^{(j)}) - V(x) \bigg)}_{\text{REINFORCE}} + \underbrace{\sum_{j=1}^{N} \nabla_{\theta} \log { \pi_{\theta}(y^{(j)}\mid x)} \left( \mathrm{sg} \left( \log \dfrac{ \pi_{\theta}(y^{(j)} \mid x) }{ \pi_{\text{ref}}(y^{(j)}\mid x) } \right) - \widehat{\mathrm{KL}} \right)}_{\text{The regularizer gradient}}
\end{align*}
\begin{lemma} The regularizer gradient is equivalent to the expected KL gradient: its expected value is \mbox{$(N-1)\, \nabla_{\theta} \mathrm{KL}( \pi_{\theta} \parallel \pi_{\text{ref}} )$}, since
    \begin{equation}
        \begin{aligned}
        \E_{ \substack{y^{(j)} \sim \pi_{\theta}(\cdot \mid x) \\ j=1,\dots, N} } \left[ \sum_{j=1}^{N} \nabla_{\theta} \log { \pi_{\theta}(y^{(j)}\mid x)} \, \widehat{\mathrm{KL}} \right] &= \E_{ \substack{y^{(j)} \sim \pi_{\theta}(\cdot \mid x) \\ j=1,\dots, N} } \left[ \dfrac{1}{N} \sum_{j=1}^{N} \nabla_{\theta} \log { \pi_{\theta}(y^{(j)}\mid x)} \, \mathrm{sg} \left(\log {\textstyle \frac{ \pi_{\theta}(y^{(j)} \mid x) }{ \pi_{\text{ref}}(y^{(j)} \mid x) }} \right) \right] \\
        &= \nabla_{\theta} \mathrm{KL}( \pi_{\theta}(\cdot \mid x) \parallel \pi_{\text{ref}}(\cdot \mid x) )
        \end{aligned}
    \end{equation}
\label{lemma:regularizer-grad}
\end{lemma}

\begin{proof}
Expanding the product of sums gives
\[
\E_{ \substack{y^{(j)} \sim \pi_{\theta}(\cdot \mid x) \\ j=1,\dots, N} } \left[ \sum_{j=1}^{N} \nabla_{\theta} \log { \pi_{\theta}(y^{(j)}\mid x)} \, \widehat{\mathrm{KL}} \right] = 
\E_{ \substack{y^{(j)} \sim \pi_{\theta}(\cdot \mid x) \\ j=1,\dots, N} } \left[ \frac{1}{N}\sum_{i=1}^{N}\sum_{j=1}^{N}
\nabla_{\theta} \log \pi_{\theta}(y^{(j)}\mid x)\,
\mathrm{sg}\!\left(
\log \frac{ \pi_{\theta}(y^{(i)} \mid x) }{ \pi_{\text{ref}}(y^{(i)}\mid x) }
\right) \right]
\]
which can be written as two components:
\begin{align*}
    &\frac{1}{N}\sum_{i=1}^{N}\sum_{j=1}^{N} \nabla_{\theta} \log \pi_{\theta}(y^{(j)}\mid x)\, \mathrm{sg}\!\left( \log \frac{ \pi_{\theta}(y^{(i)} \mid x) }{ \pi_{\text{ref}}(y^{(i)}\mid x) } \right) \\
    = & \underbrace{\frac{1}{N}\sum_{j=1}^{N} \nabla_{\theta} \log \pi_{\theta}(y^{(j)}\mid x)\, \mathrm{sg} \Big( \log \frac{ \pi_{\theta}(y^{(j)} \mid x) }{ \pi_{\text{ref}}(y^{(j)}\mid x) } \Big)}_{\text{= $\mathbb{E}_{y\sim\pi_\theta}\!\left[\nabla_{\theta} \log \pi_{\theta}(y\mid x)\,\mathrm{sg}\!\left(\log \frac{ \pi_{\theta}(y \mid x) }{ \pi_{\text{ref}}(y\mid x) }\right)\right]$}} + \underbrace{ \frac{1}{N} \sum_{j=1}^{N} \sum_{i\neq j} \nabla_{\theta} \log \pi_{\theta}(y^{(j)}\mid x) \, \mathrm{sg}\Big( \log \frac{ \pi_{\theta}(y^{(i)} \mid x) }{ \pi_{\text{ref}}(y^{(i)}\mid x) } \Big) }_{\text{has an expected value $=0$, since \(\mathbb{E}_{y\sim\pi_\theta}[\nabla_\theta \log \pi_\theta(y\mid x)]=0\)}}
\end{align*}
For \(i\neq j\), independence of samples and the score-function identity
\(\mathbb{E}_{y\sim\pi_\theta}[\nabla_\theta \log \pi_\theta(y\mid x)]=0\)
imply that these off-diagonal terms have zero expectation. The only
non-vanishing contribution comes from the diagonal \(i=j\), yielding
\[
\mathbb{E}_{y\sim\pi_\theta}\!\left[
\nabla_{\theta} \log \pi_{\theta}(y\mid x)\,
\mathrm{sg}\!\left(
\log \frac{ \pi_{\theta}(y \mid x) }{ \pi_{\text{ref}}(y\mid x) }
\right)
\right] = \nabla_{\theta} \mathrm{KL}( \pi_{\theta}(\cdot \mid x) \parallel \pi_{\text{ref}}(\cdot \mid x) ).
\]
The first part of the regularizer gradient is a sum of $N$ such terms, with an expected value of $N\, \nabla_{\theta} \mathrm{KL}( \pi_{\theta} \parallel \pi_{\text{ref}} )$, and its second part, with $\widehat{\mathrm{KL}}$, has an expected value of $\nabla_{\theta} \mathrm{KL}( \pi_{\theta} \parallel \pi_{\text{ref}} )$. Thus, the regularizer gradient is equivalent to the common KL gradient.
\end{proof}

By the same argument, the REINFORCE term with the group baseline $V(x)$ has an expected value of \mbox{$-(N-1)\, \nabla_{\theta} \E_{\pi_{\theta}(y|x)} [R(x,y)]$}. Therefore, $L_{1}$ is simply a typical KL-regularized policy gradient objective:
\begin{align}
    \text{K1:} \quad \min_{\theta}-\log w \quad 
    \boxed{\nabla_{\theta}L_{1} = -\nabla_{\theta} \log { \pi_{\theta}(y\mid x)} \Big( R(x,y) - V(x) \Big) + \nabla_{\theta} \mathrm{KL}( \pi_{\theta}(\cdot \mid x) \parallel \pi_{\text{ref}}(\cdot \mid x) )}
\end{align}
Which is what typical policy gradient algorithms like PPO \cite{ppo} and GRPO \cite{grpo} optimize in the on-policy case.

\subsection{K2 estimator leads to APA / Kimi loss}
\label{app:k2-leads-to-apa-kimi}
Using the K2 estimator: 
\begin{align*}
    \text{K2:} \qquad & \min_{\theta} \dfrac{1}{2} (\log w)^{2} \quad &&  \nabla_{\theta}L_{2} = \dfrac{1}{2} \nabla_{\theta} \bigg( \log \pi_{\theta}(y|x) - \log \pi_{\text{ref}}(y|x) - R(x,y) + \log Z(x) \bigg)^{2}
\end{align*}
Expanding the partition function:
\begin{align*}
    &\dfrac{1}{2} \nabla_{\theta} \bigg( \log \pi_{\theta}(y|x) - \log \pi_{\text{ref}}(y|x) - (R(x,y) - V(x)) - \widehat{\mathrm{KL}} \bigg)^{2} \\
    =& \underbrace{\dfrac{1}{2} \nabla_{\theta} \bigg( \log \pi_{\theta}(y|x) - \log \pi_{\text{ref}}(y|x) - (R(x,y) - V(x)) \bigg)^{2}}_{\text{APA Loss}} - \underbrace{\nabla_{\theta}\log \pi_{\theta}(y\mid x) \, \widehat{\mathrm{KL}}}_{\substack{\text{Becomes $\nabla_{\theta}\mathrm{KL}( \pi_{\theta}(\cdot \mid x) \parallel \pi_{\text{ref}}(\cdot \mid x) )$} \\ \text{when summed over $N$ trajectories (Lemma \ref{lemma:regularizer-grad})}}}
\end{align*}
The loss $L_2$ is the Advantage-induced Policy Alignment (\textbf{APA}) \cite{apa}, which was also used by Kimi models \cite{kimi-thinking}. The second term was often ignored, since prior works often approximate $\log Z(x) \approx V(x)$ rather than $\log Z(x) \approx V(x) - \widehat{\mathrm{KL}}$. The APA loss already incorporates the same KL term in expectation:
\begin{align*}
    &\dfrac{1}{2} \nabla_{\theta} \Big( \log \pi_{\theta}(y|x) - \log \pi_{\text{ref}}(y|x) - (R(x,y) - V(x)) \Big)^{2} \\
    =& - \nabla_{\theta} \log \pi_{\theta}(y|x) \Big( R(x,y) - V(x) \Big) + \nabla_{\theta} \mathrm{KL}(\pi_{\theta}(\cdot\mid x) \| \pi_{\text{ref}}(\cdot\mid x))
\end{align*}
Which means that the additional $-\nabla_{\theta}\log \pi_{\theta}(y\mid x) \, \widehat{\mathrm{KL}}$ term just downweights the KL. Both APA and Kimi showed that this loss is effective in the off-policy regime without necessarily adding off-policy correction:
\begin{align}
    \text{K2:} \qquad & \min_{\theta} \dfrac{1}{2} (\log w)^{2} \qquad  \boxed{\nabla_{\theta}L_{2} = \dfrac{1}{2} \nabla_{\theta} \Big( \log \pi_{\theta}(y|x) - \log \pi_{\text{ref}}(y|x) - (R(x,y) - V(x)) \Big)^{2}}
\end{align}

\subsection{K5 estimator leads to RWR / MPO loss}
\label{app:k5-leads-to-rwr-mpo}
Under the K5 estimator, the divergence minimization in policy optimization becomes:
\begin{align*}
    \text{K5:} \quad \min_{\theta} \mathrm{sg}(w) \log w + \log r \qquad 
    \nabla_{\theta} L_{5} = \nabla_{\theta} \log \pi_{\theta}(y|x) \bigg( 1 - \dfrac{ \pi_{\text{ref}}(y|x) \, \exp R(x,y) }{ \pi_{\theta}(y|x)\, Z(x) } \bigg)
\end{align*}
The constant $1$ comes from $\log r$ (\S\ref{app:baseline-1-minus-w}) and has zero mean, since \(\mathbb{E}_{y\sim\pi_\theta}[\nabla_\theta \log \pi_\theta(y\mid x)]=0\). Dropping it and averaging over all $N$ rollouts per $x$ to simplify the partition function:
\begin{equation*}
\begin{aligned}
    &- \dfrac{1}{N} \sum_{i=1}^{N} \nabla_{\theta} \log \pi_{\theta}(y^{(i)} \mid x) \left( \dfrac{ \pi_{\text{ref}}(y^{(i)} \mid x) \, \exp R(x, y^{(i)}) }{ \pi_{\theta}(y^{(i)} \mid x)\, Z(x) } \right) \\
    \approx & - \dfrac{1}{N} \sum_{i=1}^{N} \nabla_{\theta} \log \pi_{\theta}(y^{(i)} \mid x) \left( \dfrac{ { \mathrm{sg} \left(\frac{\pi_{\text{ref}}(y^{(i)} \mid x)}{\pi_{\theta}(y^{(i)} \mid x)} \right) \, \exp R(x, y^{(i)}) } }{ \dfrac{1}{N} \sum_{j=1}^{N} \mathrm{sg}\Big( {\textstyle \frac{ \pi_{\text{ref}}(y^{(j)}\mid x) }{ \pi_{\theta}(y^{(j)} \mid x) }} \Big) \exp R(x,y^{(j)} ) } \right) \\
    = & - \sum_{i=1}^{N} \nabla_{\theta} \log \pi_{\theta}(y^{(i)} \mid x) \cdot \mathrm{sg} \underbrace{\left( \dfrac{ { \exp \left( R(x, y^{(i)}) - \log \frac{\pi_{\theta}(y^{(i)} \mid x)}{\pi_{\text{ref}}(y^{(i)} \mid x)} \right) } }{ \sum_{j=1}^{N} { \exp \left( R(x, y^{(j)}) - \log \frac{\pi_{\theta}(y^{(j)} \mid x)}{\pi_{\text{ref}}(y^{(j)} \mid x)} \right) } } \right)}_{\text{$\mathrm{softmax}$ over $[y^{(1)}, \dots, y^{(N)}]$}}
\end{aligned}
\end{equation*}
Which can be written in a batched manner over a group of $N$ rollouts per prompt $\boldsymbol{y}=[y^{(1)}, \dots, y^{(N)}]$, where $\mathrm{sum}$ and $\mathrm{softmax}$ are taken along the axis of the group dimension $i=1,\dots, N$ (which we denote as $\mathrm{sum}_{y}$ and $\mathrm{softmax}_{y}$):
\begin{align*}
    - \mathrm{sum}_{y} \left(\nabla_{\theta} \log \pi_{\theta}(\boldsymbol{y} \mid {x}) \cdot \mathrm{softmax}_{y} \left( {R}({x}, \boldsymbol{y}) - \log \frac{ \pi_{\theta}(\boldsymbol{y} \mid {x}) }{ \pi_{\text{ref}}(\boldsymbol{y} \mid {x}) } \right) \right)
\end{align*}
This loss closely resembles Reward-Weighted Regression (\textbf{RWR}) \cite{rwr} and Maximum a Posteriori Policy Optimisation (\textbf{MPO}) \cite{mpo}.
\begin{align}
    \text{K5:} \quad \min_{\theta} \mathrm{sg}(w) \log w + \log r \quad 
    \boxed{\nabla_{\theta} L_{5} = - \nabla_{\theta} \log \pi_{\theta}(y|x) \cdot \mathrm{softmax}_{y} \Big( R(x,y) - \log (\pi_{\theta}(y|x)/\pi_{\text{ref}}(y|x)) \Big)}
\end{align}

\subsection{K3 / K3$^{++}$ / K4 estimators lead to the same losses}
\label{app:k3-k4-lead-to-combined-loss}
The K3 estimator leads to $L_5$ (same as the RWR / MPO loss):
\begin{align}
    \text{K3:} \quad &\min_{\theta}(-\log w + w -1) \nonumber \\
    \nabla_{\theta}L_{3} &= \underbrace{\nabla_{\theta} \log { \pi_{\theta}(y|x)}}_{\text{From } -\log w} - \underbrace{ \Big( \nabla_{\theta} \log \pi_{\theta}(y|x) \Big) \bigg( \dfrac{ \pi_{\text{ref}}(y|x) \, \exp R(x,y) }{ \pi_{\theta}(y|x)\, Z(x) } \bigg)}_{\text{From } w - 1} = \nabla_{\theta} L_{5}
\end{align}
The $\text{K3}^{++}$ estimator leads to $L_2$ (same as the APA loss):
\begin{align}
    \text{K3}^{++}\text{:} \quad & \min_{\theta} r \cdot (- \log w + w - 1) \qquad r=\dfrac{\pi_{\theta}(y|x)}{ \mathrm{sg}( \pi_{\theta}(y|x) ) } \nonumber \\
    \nabla_{\theta} L_{3++} &= \Big( \nabla_{\theta} \log \pi_{\theta} (y|x) \Big) \bigg( \log \pi_{\theta}(y|x) - \log \pi_{\text{ref}}(y|x) - R(x,y) + \log Z(x) \bigg) = \nabla_{\theta}L_{2}
\end{align}

The K4 estimator leads to $L_2$ (same as the APA loss):
\begin{align}
    &\text{K4:} \quad \min_{\theta} r \cdot \mathrm{sg}(- \log w) \nonumber \\
    &\nabla_{\theta} L_{4} = \Big( \nabla_{\theta} \log \pi_{\theta}(y|x) \Big) \bigg( \log \pi_{\theta}(y|x) - \log \pi_{\text{ref}}(y|x) - R(x,y) + \log Z(x) \bigg) = \nabla_{\theta} L_{2}
\end{align}

\paragraph{A Note on Off-policy Correction}
Just like the off-policy correction for any sampled KL, any policy loss derived above can simply multiply a stop gradient version of the importance weight to remain unbiased when sampling is not on-policy:
\begin{align*}
    y \sim \pi_{\text{sampling}}(\cdot\mid x) \qquad s = \mathrm{sg} \left(\dfrac{ \pi_{\theta} (y \mid x) }{ \pi_{\text{sampling}}(y\mid x) } \right)
\end{align*}
With the option of using additional PPO-like binary masks on the loss
\begin{align*}
    M(x,y)=\begin{cases}
        0 &\text{ if } R(x,y) - V(x) > 0 \text{ and } s>1+\epsilon_{\text{high}} \\
        0 &\text{ if } R(x,y) - V(x) < 0 \text{ and } s<1-\epsilon_{\text{low}} \\
        1 &\text{ otherwise}
    \end{cases}
\end{align*}
Leading us to conclude that the off-policy gradient under any loss function $L$ above is
\begin{align*}
    \E_{y\sim \pi_{\text{sampling}}(\cdot \mid x)} [s(x,y) \, M(x,y)\; \nabla_{\theta} L(x,y)]
\end{align*}
Though it is worth noting that many prior works find those losses to be effective in the absence of off-policy correction.

\paragraph{Summary of the Equivalence between Policy Gradient and Divergence Minimization}
We have shown that many existing policy gradient methods can be derived from sampled KL estimators. In other words, better techniques in sampled KL estimation may directly lead to better policy gradient algorithms under such an equivalence.

\begin{table}[h]
\vspace{25pt}
\centering
\small
\renewcommand{\arraystretch}{1.18}
\setlength{\tabcolsep}{4pt}
\begin{tabularx}{1.0\linewidth}{llll}
\toprule
Estimator
& Expression
& Corresponding policy gradient & Prior literature? \\
\midrule
K1
& $-\log w$
& $\nabla_{\theta}L_{1}=-\,\nabla_{\theta} \log { \pi_{\theta}(y|x)} \Big( R(x,y) - V(x) \Big) + \nabla_{\theta} \mathrm{KL}( \pi_{\theta}(\cdot| x) \| \pi_{\text{ref}}(\cdot| x) )$ & \textbf{PPO}, \textbf{GRPO}, etc \\
\midrule
K2 
& $\dfrac{1}{2} (\log w)^{2}$
& $\nabla_{\theta}L_{2} = \,\dfrac{1}{2} \nabla_{\theta} \Big( \log \pi_{\theta}(y|x) - \log \pi_{\text{ref}}(y|x) - R(x,y) + V(x) \Big)^{2}$ & \textbf{APA}, \textbf{Kimi} \\
\midrule
K4 
& $r \cdot \mathrm{sg}(-\log w)$
& $\nabla_{\theta}L_{4} = \nabla_{\theta}L_{2}$
& APA, Kimi \\
\midrule
K5
& $\mathrm{sg}(w) \log w + \log r$
& $\nabla_{\theta} L_{5} = - \, \nabla_{\theta} \log \pi_{\theta}(y|x) \; \mathrm{softmax}_{y} \Big( R(x,y) - \log \big( \pi_{\theta}(y|x) / \pi_{\text{ref}}(y|x) \big) \Big) $
& \textbf{RWR}, \textbf{MPO}
\\
\midrule
K3
& $-\log w + w -1$
& $\nabla_{\theta} L_{3} = \nabla_{\theta}L_{5}$ 
& RWR, MPO
\\
\midrule
$\text{K3}^{++}$
& $r (- \log w + w - 1)$
& $\nabla_{\theta} L_{3++} = \nabla_{\theta}L_{2} $ 
& APA, Kimi
\\
\bottomrule
\end{tabularx}
\vspace{5pt}
\caption{\textbf{Each KL estimator directly corresponds to a policy gradient loss} under the divergence minimization view of RL.}
\label{tab:family-of-pg}
\end{table}

\clearpage
\newpage

\section{Policy Gradient as $\boldsymbol{f}$-Divergence Minimization}
\label{app:pg-as-sampled-f-divergence}

Once a problem is framed as minimizing a KL, it can typically be recast into the general framework of $f$-divergence minimization, leading to a family of loss functions. We have seen this in a variety of other optimization problems: the generalization of GAN \cite{gan} to $f$-GAN \cite{f-gan}, GAIL \cite{gail} to $f$-MAX \cite{f-max} and $f$-VIM \cite{ke2020imitation}, DPO \cite{dpo} to $f$-DPO \cite{f-dpo}. In this section, we still assume that the optimal policy takes the form of a Boltzmann policy \cite{max-ent-irl, equivalence-between-pg-and-sq} due to its partition function having an intuitive interpretation (which says that the log partition function is equivalent to the soft value function), but explore using a sampled $f$-divergence estimator for policy optimization.

We now extend the results in \S\ref{app:new-family-of-pg} to generic $f$-divergence estimators. Let the optimal policy $\pi^{*}$ be:
\begin{align*}
    \pi^{*}(y|x) = \dfrac{1}{Z(x)} \pi_{\text{ref}}(y | x)\exp (R(x, y)) \qquad w = \dfrac{\pi^{*}(y|x)}{\pi_{\theta}(y|x)} = \dfrac{ \pi_{\text{ref}}(y|x) \, \exp R(x,y) }{ \pi_{\theta}(y|x)\, Z(x) }
\end{align*}
Re-iterating the definition of $f$-divergence, with $f(\cdot)$ being convex and $f(1)=0$:
\begin{align*}
    D_{f}(p \| q) = \E_{y\sim q}\bigg[ f\bigg( \dfrac{p(y)}{q(y)} \bigg) \bigg]
\end{align*}
We now study two directions of $f$-divergence minimization: $D_{f}(\pi_{\theta}(\cdot | x) \| \pi^{*}(\cdot | x) )$ and $D_{f}(\pi^{*}(\cdot | x) \| \pi_{\theta}(\cdot | x) )$.

\subsection{Minimizing $D_{f}(\pi_{\theta} \| \pi^{*})$}
\label{app:minimize-d_f-theta-star}

If we define function $g_{f}(\cdot)$ as:
\begin{align}
    g_{f}(w) \triangleq w\, f\!\left(\frac{1}{w}\right)
\end{align}
Then, using the definition of $f$-divergence:
\begin{align}
    D_{f}(\pi_{\theta}(\cdot | x) \| \pi^{*}(\cdot | x) ) = \E_{y\sim \pi^{*}(\cdot | x)} \bigg[ f\bigg( \dfrac{1}{w} \bigg) \bigg] = \E_{y\sim \pi_{\theta}(\cdot | x)} \bigg[ w \, f\bigg( \dfrac{1}{w} \bigg) \bigg] = \E_{y\sim \pi_{\theta}(\cdot|x)}\big[ g_{f}(w(x,y)) \big]
\end{align}
Note that $w(x,y)$ has the following properties:
\begin{align}
    \nabla_{\theta} \log w(x,y) = - \nabla_{\theta}\log \pi_{\theta}(y|x),
    \qquad
    \nabla_{\theta} w(x,y) = - w(x,y)\,\nabla_{\theta}\log \pi_{\theta}(y|x).
\end{align}
Therefore, using the score-function identity, we arrive at:
\begin{align}
    \nabla_{\theta} D_{f}(\pi_{\theta}\|\pi^{*})
    = \nabla_{\theta}\E_{y\sim \pi_{\theta}(\cdot|x)}\big[g_f(w)\big] &= \E_{y\sim \pi_{\theta}(\cdot|x)}\Big[\nabla_{\theta}\log \pi_{\theta}(y|x)\, g_f(w) + \nabla_{\theta} g_f(w)\Big] \nonumber\\
    &= \E_{y\sim \pi_{\theta}(\cdot|x)}\Big[\nabla_{\theta}\log \pi_{\theta}(y|x)\,\big(g_f(w) - w\, g_f'(w)\big)\Big].
\end{align}
Hence, minimizing $D_{f}(\pi_{\theta}\|\pi^{*})$ is equivalent to a policy gradient update with scalar weight $-\phi_{f}(w)$ where $\phi_{f}(w)$ is
\begin{align}
    \phi_f(w) \triangleq g_f(w) - w\, g_f'(w) = f'(1/w).
    \label{eq:phi_f_def}
\end{align}

\subsection{Minimizing $D_{f}(\pi^{*} \| \pi_{\theta})$}
\label{app:minimize-d_f-star-theta}

We now consider the reverse direction:
\begin{align}
    D_{f}(\pi^{*}(\cdot|x)\,\|\,\pi_{\theta}(\cdot|x))
    = \E_{y\sim \pi_{\theta}(\cdot|x)}\bigg[ f\bigg( \dfrac{\pi^{*}(y|x)}{\pi_{\theta}(y|x)} \bigg) \bigg]
    = \E_{y\sim \pi_{\theta}(\cdot|x)}\big[ f(w(x,y)) \big].
\end{align}
By the same derivation as above,
\begin{align}
    \nabla_{\theta} D_{f}(\pi^{*}\|\pi_{\theta})
    &= \E_{y\sim \pi_{\theta}(\cdot|x)}\Big[\nabla_{\theta}\log \pi_{\theta}(y|x)\,\big(f(w) - w\, f'(w)\big)\Big].
\end{align}
Thus the induced policy-gradient weight for the reverse direction is
\begin{align}
    \psi_f(w) \triangleq f(w) - w\, f'(w).
    \label{eq:psi_f_def}
\end{align}

\paragraph{Summary of $f$-divergence loss.}
The two directions of $f$-divergence minimization are reduced to the following:
\begin{align}
    \nabla_{\theta} D_f(\pi_{\theta}\|\pi^{*})
    &= \E_{y\sim \pi_{\theta}}\Big[\nabla_{\theta}\log \pi_{\theta}(y|x)\, \phi_f(w)\Big],
    \qquad \phi_f(w)=f'(1/w), \\
    \nabla_{\theta} D_f(\pi^{*}\|\pi_{\theta})
    &= \E_{y\sim \pi_{\theta}}\Big[\nabla_{\theta}\log \pi_{\theta}(y|x)\, \psi_f(w)\Big],
    \qquad \psi_f(w)=f(w)-w f'(w).
\end{align}
\paragraph{The EM Algorithm} The expectation-maximization (EM) process can be summarized as the following. For each prompt or task $x$, we sample $N$ rollouts $[y^{(1)}, \dots, y^{(N)}]$ from $\pi_{\theta}(\cdot\mid x)$ to estimate the partition function $Z(x)$ and the ratio $w$ as:
\begin{align*}
    \textbf{E-Step} \qquad Z(x) \approx \dfrac{1}{N} \sum_{i=1}^{N} \mathrm{sg}\left( { \frac{ \pi_{\text{ref}}(y^{(i)}\mid x) }{ \pi_{\theta}(y^{(i)} \mid x) }} \right) \exp \Big( R(x,y^{(i)}) \Big) \qquad w = \frac{\pi_{\text{ref}}(y\mid x)}{\pi_{\theta}(y\mid x)} \cdot \frac{ \exp R(x,y) }{ Z(x) }
\end{align*}
where the returns $R(x,y^{(i)})$ are (pre-) scaled by $1/\beta$ to control the regularization effect of the anchor $\pi_{\text{ref}}$. After the E-Step, M-Step performs gradient updates based on $f$-divergence minimization; the following two tables, Table \ref{tab:f-div-objectives-weight-reverse} and Table \ref{tab:f-div-objectives-weight-forward}, summarize what to minimize under each $f$-divergence (\textbf{M-Step}).

\begin{figure*}[t]
  \centering
  \includegraphics[width=1.0\textwidth]{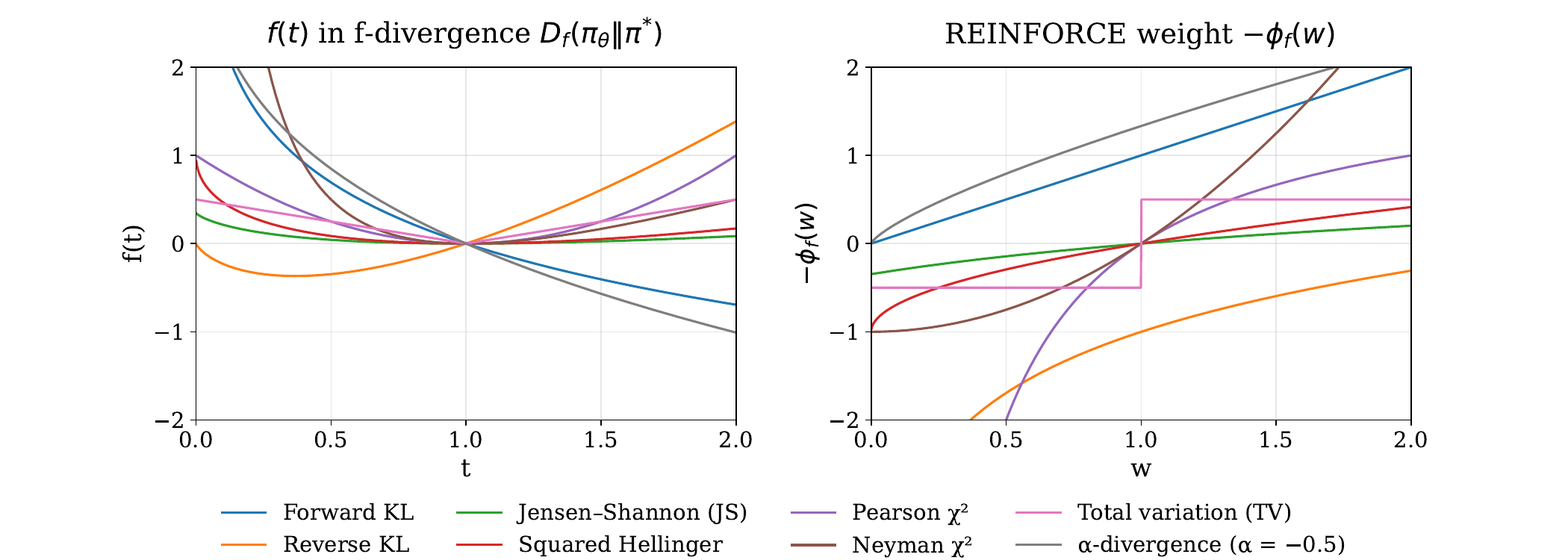}
  \caption{Policy gradient based on $\nabla_{\theta} D_f(\pi_{\theta}\|\pi^{*})$, where the REINFORCE weight for $\nabla_{\theta}\log \pi_{\theta}(y|x)$ is $-\phi_{f}(w) = -f'(1/w)$.}
  \label{fig:pg-as-f-div-minimization}
  \vspace{10pt}
\end{figure*}

\begin{table}[h]
\centering
\small
\setlength{\tabcolsep}{5pt}
\renewcommand{\arraystretch}{1.25}
\vspace{20pt}
\caption{\centering \textbf{M-Step:} Minimizing $D_f(\pi_{\theta}\|\pi^{*})$ results in the gradient update $\E_{y\sim \pi_{\theta}}[\nabla_{\theta}\log \pi_{\theta}(y|x)\,\phi_f(w)]$}
\begin{tabular}{l l l}
\toprule
\textbf{$f$-divergence} & \textbf{$f(t)$} & $\phi_f(w)=f'(1/w)$ \\
\midrule
Forward KL
& $-\log t$
& $-w$ \\

Reverse KL
& $t\log t$
& $-\log w + 1$ \\

Jensen--Shannon (JS)
& $\frac{1}{2}\left(t\log t - (t+1)\log\!\Big(\frac{t+1}{2}\Big)\right)$
& $\frac{1}{2}\log\!\Big(\frac{2}{1+w}\Big)$ \\

Squared Hellinger
& $(\sqrt{t}-1)^2$
& $1-\sqrt{w}$ \\

Pearson $\chi^2$
& $(t-1)^2$
& $-2 + 2/w$ \\

Neyman $\chi^2$
& $\frac{(t-1)^2}{t}$
& $1-w^{2}$ \\

Total variation (TV)
& $\frac{1}{2}\,\lvert t-1\rvert$
& $-\frac{1}{2}\,\mathrm{sgn}(w-1)$ \\

$\alpha$-divergence ($\alpha\neq \pm 1$)
& $\frac{4}{1-\alpha^{2}}\Big(1-t^{\frac{1+\alpha}{2}}\Big)$
& $\frac{2}{\alpha-1}w^{\frac{1-\alpha}{2}}$ \\
\bottomrule
\end{tabular}
\label{tab:f-div-objectives-weight-reverse}
\end{table}

\begin{figure*}[t]
  \centering
  \includegraphics[width=1.0\textwidth]{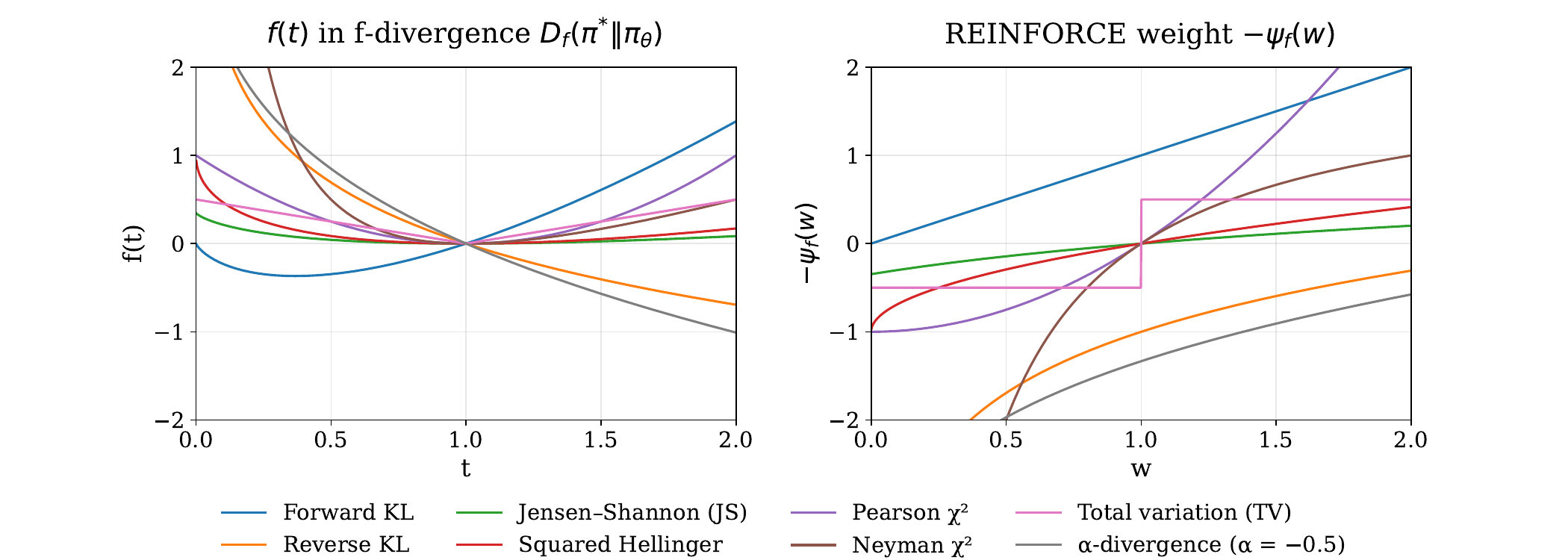}
  \caption{Policy gradient based on $\nabla_{\theta} D_f(\pi^{*}\|\pi_{\theta})$, where the REINFORCE weight for $\nabla_{\theta}\log \pi_{\theta}(y|x)$ is $-\psi_f(w)=-f(w)+w f'(w)$.}
  \label{fig:pg-as-f-div-reverse-minimization}
  \vspace{10pt}
\end{figure*}

\begin{table}[h]
\centering
\small
\setlength{\tabcolsep}{5pt}
\renewcommand{\arraystretch}{1.25}
\vspace{-20pt}
\caption{\centering \textbf{M-Step:} Minimizing $D_f(\pi^{*}\|\pi_{\theta})$ results in the gradient update
$\E_{y\sim \pi_{\theta}}[\nabla_{\theta}\log \pi_{\theta}(y|x)\,\psi_f(w)]$}
\begin{tabular}{l l l}
\toprule
\textbf{$f$-divergence} & \textbf{$f(t)$} & $\psi_f(w)=f(w)-w f'(w)$ \\
\midrule
Forward KL
& $t\log t$
& $-w$ \\

Reverse KL
& $-\log t$
& $-\log w + 1$ \\

Jensen--Shannon (JS)
& $\frac{1}{2} \left(t\log t - (t+1)\log\!\Big(\frac{t+1}{2}\Big) \right)$
& $\frac{1}{2} \log\!\Big(\frac{2}{1+w}\Big)$ \\

Squared Hellinger
& $(\sqrt{t}-1)^2$
& $1 - \sqrt{w}$ \\

Pearson $\chi^2$
& $(t-1)^2$
& $1 - w^{2}$ \\

Neyman $\chi^2$
& $\frac{(t-1)^2}{t}$
& $-2 + 2/w$ \\

Total variation (TV)
& $\frac{1}{2}\,\lvert t-1\rvert$
& $-\frac{1}{2}\,\mathrm{sgn}(w-1)$ \\

$\alpha$-divergence ($\alpha\neq \pm 1$)
& $\frac{4}{1-\alpha^{2}}\Big(1-t^{\frac{1+\alpha}{2}}\Big)$
& $-\frac{2}{\alpha+1} w^{\frac{\alpha+1}{2}} - \frac{4}{(\alpha-1)(\alpha+1)}$ \\
\bottomrule
\end{tabular}
\label{tab:f-div-objectives-weight-forward}
\end{table}

\clearpage

\section{Top-k KL: Unbiasedness and Variance}
\label{app:topk-analysis}

This appendix proves that the Top-k KL estimators of \S\ref{sec:top-k-kl-estimator} have unbiased values and gradients (Proposition~\ref{prop:topk-unbiased}), compares their tail token with the rollout token, and compares their variance with that of sampled KL. We fix a position $n$ with prefix $h_n$, and omit $h_n$ from the notation.

\subsection{Unbiased values and gradients (Proof of Proposition \texorpdfstring{\ref{prop:topk-unbiased}}{5.1})}
\label{app:topk-unbiased-proof}

Let $\ell_n(j) \triangleq \log \pi_{\theta}(j) - \log \pi_{\text{ref}}(j)$ and $w_n(j) \triangleq \pi_{\text{ref}}(j) / \pi_{\theta}(j)$, so that $\mathrm{KL}_{n} = \sum_{j\in\mathcal{V}}\pi_{\theta}(j)\,\ell_n(j)$ and $\mathrm{FKL}_{n} = \sum_{j\in\mathcal{V}}\pi_{\theta}(j)\,w_n(j)\log w_n(j)$. At a tail token $z_n = j$, the two sampled estimators of \S\ref{sec:sampled-kl} are
\begin{align*}
    ({\color{RoyalBlue}\textbf{K4}})_{n}\big|_{y_n = j} = r_n(j)\,\mathrm{sg}(\ell_n(j)) \qquad
    ({\color{RoyalBlue}\textbf{K5}})_{n}\big|_{y_n = j} = \mathrm{sg}(w_n(j)) \log w_n(j) + \log r_n(j)
\end{align*}
where $r_n(j) = \pi_{\theta}(j) / \mathrm{sg}(\pi_{\theta}(j))$ equals $1$ in value and has $\nabla_{\theta} r_n(j) = \nabla_{\theta}\log r_n(j) = \nabla_{\theta}\log\pi_{\theta}(j)$. Their gradients are therefore
\begin{align}
    \nabla_{\theta}({\color{RoyalBlue}\textbf{K4}})_{n}\big|_{y_n = j} = \ell_n(j)\,\nabla_{\theta}\log\pi_{\theta}(j) \qquad
    \nabla_{\theta}({\color{RoyalBlue}\textbf{K5}})_{n}\big|_{y_n = j} = (1 - w_n(j))\,\nabla_{\theta}\log\pi_{\theta}(j)
    \label{eq:k4-k5-token-gradients}
\end{align}
The index set $q_n$, the proposal $\mu_n$ and the importance weight $s_n = \mathrm{sg}(\pi_{\theta}(z_n) / \mu_n(z_n))$ carry no gradient, and $\mu_n(j) > 0$ for every $j \not\in q_n$. Hence, for any function $\phi$ of the tail token,
\begin{align}
    \E_{z_n}\big[s_n\,\phi(z_n)\big] = {\sum}_{j\not\in q_n}\mu_n(j)\,\frac{\pi_{\theta}(j)}{\mu_n(j)}\,\phi(j) = {\sum}_{j\not\in q_n}\pi_{\theta}(j)\,\phi(j)
    \label{eq:tail-is-sum}
\end{align}
so the importance weight turns the tail draw into the sum over the non-top-$k$ indices, whatever the proposal.

\begin{proof}[Proof of Proposition \ref{prop:topk-unbiased}]
\emph{Values.} $\Delta_n$ and $\log r_n$ are zero in value, so by \eqref{eq:tail-is-sum},
\begin{align*}
    \E_{z_n}\big[(\widehat{\mathrm{KL}}_{\text{tail}})_{n}\big] &= {\sum}_{j\not\in q_n}\pi_{\theta}(j)\,\ell_n(j) \\
    \E_{z_n}\big[(\widehat{\mathrm{FKL}}_{\text{tail}})_{n}\big] &= {\sum}_{j\not\in q_n}\pi_{\theta}(j)\,w_n(j)\log w_n(j) = {\sum}_{j\not\in q_n}\pi_{\text{ref}}(j)\log\frac{\pi_{\text{ref}}(j)}{\pi_{\theta}(j)}
\end{align*}
These are the terms of $\mathrm{KL}_{n}$ and $\mathrm{FKL}_{n}$ that the truncated sums over $q_n$ leave out, so both estimators have unbiased values.

\emph{Gradients, reverse KL.} Since $\nabla_{\theta}[\pi_{\theta}(j)\,\ell_n(j)] = (\ell_n(j) + 1)\,\nabla_{\theta}\pi_{\theta}(j)$ and $\nabla_{\theta}\Delta_n = \sum_{j\in q_n}\nabla_{\theta}\pi_{\theta}(j)$, and by \eqref{eq:k4-k5-token-gradients} and \eqref{eq:tail-is-sum},
\begin{align*}
    \nabla_{\theta}\big[(\mathrm{KL}_{\text{trun}})_{n} - \Delta_n\big] &= {\sum}_{j\in q_n}\ell_n(j)\,\nabla_{\theta}\pi_{\theta}(j) \\
    \E_{z_n}\big[\nabla_{\theta}(\widehat{\mathrm{KL}}_{\text{tail}})_{n}\big] &= {\sum}_{j\not\in q_n}\pi_{\theta}(j)\,\ell_n(j)\,\nabla_{\theta}\log\pi_{\theta}(j) = {\sum}_{j\not\in q_n}\ell_n(j)\,\nabla_{\theta}\pi_{\theta}(j)
\end{align*}
The two add up to $\sum_{j\in\mathcal{V}}\ell_n(j)\,\nabla_{\theta}\pi_{\theta}(j)$. This is $\nabla_{\theta}\mathrm{KL}_{n} = \sum_{j\in\mathcal{V}}(\ell_n(j) + 1)\,\nabla_{\theta}\pi_{\theta}(j)$, because $\sum_{j\in\mathcal{V}}\nabla_{\theta}\pi_{\theta}(j) = 0$.

\emph{Gradients, forward KL.} Since $\nabla_{\theta}(\mathrm{FKL}_{\text{trun}})_{n} = -\sum_{j\in q_n}\pi_{\text{ref}}(j)\,\nabla_{\theta}\log\pi_{\theta}(j)$ and $\pi_{\theta}(j)\,w_n(j) = \pi_{\text{ref}}(j)$,
\begin{align*}
    \nabla_{\theta}\big[(\mathrm{FKL}_{\text{trun}})_{n} + \Delta_n\big] &= -{\sum}_{j\in q_n}\pi_{\text{ref}}(j)\,\nabla_{\theta}\log\pi_{\theta}(j) + {\sum}_{j\in q_n}\nabla_{\theta}\pi_{\theta}(j) \\
    \E_{z_n}\big[\nabla_{\theta}(\widehat{\mathrm{FKL}}_{\text{tail}})_{n}\big] &= -{\sum}_{j\not\in q_n}\pi_{\text{ref}}(j)\,\nabla_{\theta}\log\pi_{\theta}(j) + {\sum}_{j\not\in q_n}\nabla_{\theta}\pi_{\theta}(j)
\end{align*}
The two add up to $-\sum_{j\in\mathcal{V}}\pi_{\text{ref}}(j)\,\nabla_{\theta}\log\pi_{\theta}(j) = \nabla_{\theta}\mathrm{FKL}_{n}$, again because $\sum_{j\in\mathcal{V}}\nabla_{\theta}\pi_{\theta}(j) = 0$.
\end{proof}

\paragraph{The vanilla forms.} Without $\Delta_n$, the same sums give $\nabla_{\theta}\mathrm{KL}_{n} + \sum_{j\in q_n}\nabla_{\theta}\pi_{\theta}(j)$ in the reverse case, which is \eqref{eq:topk-vanilla-bias}, and $\nabla_{\theta}\mathrm{FKL}_{n} - \sum_{j\in q_n}\nabla_{\theta}\pi_{\theta}(j)$ in the forward case, which is \eqref{eq:topk-vanilla-bias-forward}.

\subsection{Tail token versus rollout token}
\label{app:topk-tail-variance}

The tail term of Top-k KL can be estimated with a separate tail token $z_n \sim \mu_n$, as in \eqref{eq:topk-reverse-kl} and \eqref{eq:topk-forward-kl}, or with the rollout token itself. This section shows that the two estimates have the same expectation, and compares their variance.

\paragraph{Setup.} Let $y_n \sim \pi_{\text{sampling}}$ be the rollout token, with $\pi_{\text{sampling}}(j) > 0$ for all $j \not\in q_n$, and let $\mathbf{K}_n(j)$ denote K4 (for reverse KL) or K5 (for forward KL) evaluated at token $j$. The rollout-token estimators are \eqref{eq:topk-reverse-kl} and \eqref{eq:topk-forward-kl} with the tail term replaced by
\begin{align}
    \mathbbm{1}(y_n \not\in q_n)\;\mathrm{sg}\Big( \frac{\pi_{\theta}(y_n)}{\pi_{\text{sampling}}(y_n)} \Big)\;\mathbf{K}_n(y_n)
    \label{eq:rollout-token-tail}
\end{align}
This is \eqref{eq:high-var-tail} with the importance weight of \S\ref{sec:sampled-kl} for an off-policy rollout. The exact top-$k$ term and $\Delta_n$ are the same in both kinds of estimators and are not random, so only the tail terms need to be compared.

\begin{proposition}[Lower variance than the rollout token]
\label{prop:topk-tail-variance}
The rollout-token estimators have the same expected values and gradients as \eqref{eq:topk-reverse-kl} and \eqref{eq:topk-forward-kl}. If $\mu_n(j) \geq \pi_{\text{sampling}}(j)$ for all $j \not\in q_n$, the covariance of their gradients is larger or equal in the positive semi-definite order, and the variance of their values is larger or equal. In particular, this holds for $\mu_n(j) \propto \pi_{\text{sampling}}(j)\,\mathbbm{1}(j \not\in q_n)$, where the tail token shrinks the second moment of the tail term by the factor $m_n \triangleq \sum_{j\not\in q_n}\pi_{\text{sampling}}(j)$, the probability mass of the non-top-$k$ indices.
\end{proposition}
\begin{proof}
Let $c_n(j) \triangleq \nabla_{\theta}\mathbf{K}_n(j)$. Consider a token $j \not\in q_n$ that is drawn with probability $p(j)$ and contributes the tail term $\mathrm{sg}(\pi_{\theta}(j) / p(j))\,\mathbf{K}_n(j)$, with no contribution otherwise. The gradient $G$ of this tail term satisfies
\begin{align*}
    \E[G] &= {\sum}_{j\not\in q_n} p(j)\,\frac{\pi_{\theta}(j)}{p(j)}\,c_n(j) = {\sum}_{j\not\in q_n}\pi_{\theta}(j)\,c_n(j) \\
    \E[G\,G^{\top}] &= {\sum}_{j\not\in q_n}\frac{\pi_{\theta}(j)^2}{p(j)}\,c_n(j)\,c_n(j)^{\top}
\end{align*}
The rollout token \eqref{eq:rollout-token-tail} is the case $p = \pi_{\text{sampling}}$, and the tail token is the case $p = \mu_n$. The mean does not depend on $p$, so the two have the same expected gradient, and
\begin{align}
    \mathrm{Cov}(G_{\text{rollout}}) - \mathrm{Cov}(G_{\text{tail}}) = {\sum}_{j\not\in q_n}\pi_{\theta}(j)^2 \Big( \frac{1}{\pi_{\text{sampling}}(j)} - \frac{1}{\mu_n(j)} \Big)\,c_n(j)\,c_n(j)^{\top}
    \label{eq:tail-covariance-gap}
\end{align}
which is positive semi-definite when $\mu_n(j) \geq \pi_{\text{sampling}}(j)$ for all $j \not\in q_n$. The same computation with the value of $\mathbf{K}_n(j)$ in place of $c_n(j)$ gives the statement for the values. For $\mu_n(j) = \pi_{\text{sampling}}(j) / m_n$ on the non-top-$k$ indices, $\E[G_{\text{tail}}\,G_{\text{tail}}^{\top}] = m_n\,\E[G_{\text{rollout}}\,G_{\text{rollout}}^{\top}]$, and \eqref{eq:tail-covariance-gap} equals $(1 - m_n)\,\E[G_{\text{rollout}}\,G_{\text{rollout}}^{\top}]$.
\end{proof}

\paragraph{Remarks.}
\begin{itemize}
    \item \emph{Same policy.} When the tail token is drawn from the sampling policy (Alg \ref{alg:topk-kl-reverse}), the condition of Proposition~\ref{prop:topk-tail-variance} is automatic. The rollout token falls outside $q_n$ only with probability $m_n$, whereas the tail token always does and carries the factor $m_n$ in its weight instead: the tail token replaces the indicator $\mathbbm{1}(y_n \not\in q_n)$ by its expectation.
    \item \emph{Reference policy.} When the tail token is drawn from the reference policy (Alg \ref{alg:topk-kl-forward}), the condition requires that no non-top-$k$ token is more likely under the sampling policy than under $\mu_n$. If it fails for some tokens, \eqref{eq:tail-covariance-gap} is no longer positive semi-definite in general, and which estimator has the lower variance depends on the policies.
    \item \emph{Per token.} Like the conditions in \eqref{eq:kl-estimator-two-conditions}, the statement is for one position given its prefix $h_n$.
\end{itemize}

\subsection{Top-k KL versus sampled KL}
\label{app:topk-vs-sampled-variance}

Sampled KL (\S\ref{sec:sampled-kl}) draws one token $y_n \sim p$ from a sampling distribution $p$ over $\mathcal{V}$ and uses $\mathrm{sg}(\pi_{\theta}(y_n) / p(y_n))\,\mathbf{K}_n(y_n)$, with $\mathbf{K}_n$ as in \S\ref{app:topk-tail-variance}. Top-k KL under the same sampling distribution computes the indices in $q_n$ exactly, and draws its tail token from $p$ outside of them.

\begin{proposition}[Lower variance than sampled KL]
\label{prop:topk-vs-sampled-variance}
Let $p$ be any distribution with $p(j) > 0$ for all $j \in \mathcal{V}$, and let $q_n$ be any index set. Top-k reverse KL \eqref{eq:topk-reverse-kl} and Top-k forward KL \eqref{eq:topk-forward-kl} with $\mu_n(j) \propto p(j)\,\mathbbm{1}(j \not\in q_n)$ have the same expected values and gradients as sampled K4 and K5 with $y_n \sim p$. The covariance of their gradients is smaller or equal in the positive semi-definite order, and the variance of their values is smaller or equal.
\end{proposition}
\begin{proof}
The expectations are equal by Proposition~\ref{prop:topk-unbiased} and \S\ref{sec:sampled-kl}. Let $\phi(j)$ be the value of $\mathbf{K}_n(j)$, or its gradient $c_n(j)$ projected on a fixed direction. The sampled estimator is $X = (\pi_{\theta}(y_n) / p(y_n))\,\phi(y_n)$, with
\begin{align*}
    \E[X] = \underbrace{{\sum}_{j\in q_n}\pi_{\theta}(j)\,\phi(j)}_{a} + \underbrace{{\sum}_{j\not\in q_n}\pi_{\theta}(j)\,\phi(j)}_{b} \qquad
    \E[X^2] = \underbrace{{\sum}_{j\in q_n}\frac{\pi_{\theta}(j)^2}{p(j)}\,\phi(j)^2}_{A} + \underbrace{{\sum}_{j\not\in q_n}\frac{\pi_{\theta}(j)^2}{p(j)}\,\phi(j)^2}_{B}
\end{align*}
In Top-k KL, the exact terms equal $a$ and are not random (\S\ref{app:topk-unbiased-proof}), and the tail term $Y = s_n\,\phi(z_n)$ has $\E[Y] = b$ and $\E[Y^2] = P_t\,B$, where $P_t \triangleq \sum_{j\not\in q_n} p(j)$ is the mass of $p$ on the non-top-$k$ indices; let $P_h \triangleq 1 - P_t$. By the Cauchy--Schwarz inequality, $a^2 \leq P_h\,A$ and $b^2 \leq P_t\,B$. Hence, for $0 < P_h < 1$,
\begin{align*}
    \mathrm{Var}(X) - \mathrm{Var}(a + Y) = A + P_h\,B - a^2 - 2ab \;\geq\; \frac{P_t}{P_h}\,a^2 + \frac{P_h}{P_t}\,b^2 - 2ab = \bigg( \sqrt{\frac{P_t}{P_h}}\,a - \sqrt{\frac{P_h}{P_t}}\,b \bigg)^{2} \geq 0
\end{align*}
The cases $P_h = 0$ and $P_h = 1$ are sampled KL and exact KL themselves. Since the direction of the projection is arbitrary, the gradient covariances are ordered.
\end{proof}

\paragraph{Remarks.}
\begin{itemize}
    \item \emph{Reverse KL.} In Alg \ref{alg:topk-kl-reverse}, $p$ is the sampling policy, so the comparison is with sampled K4 on the rollout token (\S\ref{sec:sampled-kl}).
    \item \emph{Forward KL.} In Alg \ref{alg:topk-kl-forward}, $p$ is the reference policy, so the comparison is with sampled K5 on a token drawn from the reference policy. Against sampled K5 on the rollout token, which is drawn from a different distribution, which estimator has the lower variance depends on the policies, as in \S\ref{app:topk-tail-variance}.
\end{itemize}

\end{document}